\documentclass[twoside,11pt]{article}

\usepackage{blindtext}

\usepackage{jmlr2e}
\usepackage{thmtools, thm-restate}
\usepackage{enumerate}
\usepackage{soul}
\usepackage[dvipsnames]{xcolor}
\usepackage{comment}
\usepackage{multirow}
\usepackage[T1]{fontenc}
\usepackage[ruled,vlined]{algorithm2e}
\usepackage{hyperref}
\usepackage{wrapfig}
\usepackage{caption}
\usepackage{subcaption}
\newenvironment{metaalgorithm}[1][htb]
  {
   \begin{algorithm}[#1]%
  }{\end{algorithm}}

\hypersetup{ hidelinks }

\usepackage{amsmath,amsfonts,bm}

\def\eqref#1{equation~\ref{#1}}

\def\1{\bm{1}}

\def\ra{{\textnormal{a}}}

\def\ro{{\textnormal{o}}}

\def\rr{{\textnormal{r}}}
\def\rs{{\textnormal{s}}}

\def\vpi{{\boldsymbol{\pi}}}

\def\rva{{\mathbf{a}}}

\def\rvg{{\mathbf{g}}}

\def\rvo{{\mathbf{o}}}

\def\vmu{{\bm{\mu}}}
\def\vtheta{{\bm{\theta}}}
\def\va{{\bm{a}}}

\def\vg{{\bm{g}}}

\def\vo{{\bm{o}}}

\def\vx{{\bm{x}}}

\def\vpi{{\boldsymbol{\pi}}}
\def\vtheta{{\boldsymbol{\theta}}}
\def\vbarpi{{\boldsymbol{\bar{\pi}}}}

\def\mH{{\bm{H}}}

\DeclareMathAlphabet{\mathsfit}{\encodingdefault}{\sfdefault}{m}{sl}
\SetMathAlphabet{\mathsfit}{bold}{\encodingdefault}{\sfdefault}{bx}{n}

\newcommand{\E}{\mathbb{E}}

\DeclareMathOperator*{\argmax}{arg\,max}
\DeclareMathOperator*{\argmin}{arg\,min}

\usepackage{lastpage}
\jmlrheading{25}{2024}{1-\pageref{LastPage}}{4/23; Revised
10/23}{1/24}{23-0488}{Yifan Zhong, Jakub Grudzien Kuba, Xidong Feng, Siyi Hu, Jiaming Ji, and Yaodong Yang}

\ShortHeadings{Heterogeneous-Agent Reinforcement Learning}{Zhong, Kuba, Feng, Hu, Ji, and Yang}
\firstpageno{1}

\begin{document}

\title{Heterogeneous-Agent Reinforcement Learning}

\author{\name Yifan Zhong$^{1,2,*}$ \email zhongyifan@stu.pku.edu.cn \\
       \name Jakub Grudzien Kuba$^{3,*}$ \email jakub.grudzien@new.ox.ac.uk \\
       \name Xidong Feng$^{4,*}$ \email xidong.feng.20@ucl.ac.uk \\
       \name Siyi Hu$^{5}$ \email siyi.hu@student.uts.edu.au \\
       \name Jiaming Ji$^{1}$ \email jiamg.ji@stu.pku.edu.cn \\
       \name Yaodong Yang$^{1,\dagger}$ \email yaodong.yang@pku.edu.cn \\
       \addr 
       $^{1}$ Institute for Artificial Intelligence, Peking University \\
       $^{2}$ Beijing Institute for General Artificial Intelligence \\
       $^{3}$ University of Oxford \\
       $^{4}$ University College London \\
       $^{5}$ ReLER, AAII, University of Technology Sydney \\
        $^*$ Equal contribution\quad $\dagger$ Corresponding author
       }

\editor{George Konidaris}

\maketitle

\begin{abstract}

The necessity for cooperation among intelligent machines has popularised cooperative \textsl{multi-agent reinforcement learning} (MARL) in AI research. However, many research endeavours heavily rely on parameter sharing among agents, which confines them to only \emph{homogeneous}-agent setting and leads to training instability and lack of convergence guarantees. To achieve effective cooperation in the general \emph{heterogeneous}-agent setting, we propose \textsl{Heterogeneous-Agent Reinforcement Learning} (HARL) algorithms that resolve the aforementioned issues. Central to our findings are the \textsl{multi-agent advantage decomposition lemma} and the \textsl{sequential update scheme}. Based on these, we develop the provably correct \textsl{Heterogeneous-Agent Trust Region Learning} (HATRL), and derive HATRPO and HAPPO by tractable approximations. Furthermore, we discover a novel framework named \textsl{Heterogeneous-Agent Mirror Learning} (HAML), which strengthens theoretical guarantees for HATRPO and HAPPO and provides a general template for cooperative MARL algorithmic designs. We prove that all algorithms derived from HAML inherently enjoy monotonic improvement of joint return and convergence to Nash Equilibrium. As its natural outcome, HAML validates more novel algorithms in addition to HATRPO and HAPPO, including HAA2C, HADDPG, and HATD3, which generally outperform their existing MA-counterparts. We comprehensively test HARL algorithms on six challenging benchmarks and demonstrate their superior effectiveness and stability for coordinating heterogeneous agents compared to strong baselines such as MAPPO and QMIX.\footnote{Our code is available at \url{https://github.com/PKU-MARL/HARL}.}

\end{abstract}


\section{Introduction}

Cooperative Multi-Agent Reinforcement Learning (MARL) is a natural model of learning in multi-agent systems, such as robot swarms \citep{huttenrauch2017guided, huttenrauch2019deep}, autonomous cars \citep{cao2012overview}, and traffic signal control \citep{calvo2018heterogeneous}. To solve cooperative MARL problems, one naive approach is to directly apply single-agent reinforcement learning algorithm to each agent and consider other agents as a part of the environment, a paradigm commonly referred to as \textsl{Independent Learning} \citep{tan1993multi,de2020independent}. Though effective in certain tasks, independent learning fails in the face of more complex scenarios \citep{hu2022marllib, foerster2018counterfactual}, which is intuitively clear: once a learning agent updates its policy, so do its teammates, which causes changes in the effective environment of each agent which single-agent algorithms are not prepared for \citep{claus1998dynamics}. To address this, a learning paradigm named \textsl{Centralised Training with Decentralised Execution}  (CTDE) \citep{maddpg,foerster2018counterfactual,JMLR:v24:22-0169} was developed. The CTDE framework learns a joint value function which, during training, has access to the global state and teammates' actions. With the help of the centralised value function that accounts for the non-stationarity caused by others, each agent adapts its policy parameters accordingly. Thus, it effectively leverages global information while still preserving decentralised agents for execution. As such, the CTDE paradigm allows a straightforward extension of single-agent policy gradient theorems \citep{sutton:nips12, silver2014deterministic} to multi-agent scenarios \citep{maddpg, kuba2021settling,spg-david}.  Consequently, numerous multi-agent policy gradient algorithms have been developed  \citep{foerster2018counterfactual, peng1703multiagent,zhang2020bi,wen2018probabilistic,gr2,yang2018mean, ackermann2019reducing}.

Though existing methods have achieved reasonable performance on common benchmarks, several limitations remain. Firstly, some algorithms \citep{mappo, de2020independent} rely on parameter sharing and require agents to be  \emph{homogeneous}  (\emph{i.e.}, share the same observation space and action space, and play similar roles in a cooperation task),  which largely limits their applicability to \emph{heterogeneous}-agent settings (\emph{i.e.}, no constraint on the observation spaces, action spaces, and the roles of agents) and potentially harms the performance \citep{christianos2021scaling}. While there has been work extending parameter sharing for heterogeneous agents \citep{terry2020parameter}, their methods rely on padding, which is neither elegant nor general. 
Secondly, existing algorithms update the agents simultaneously. As we show in Section \ref{sec:hh} later, the agents are unaware of partners' update directions under this update scheme, which could lead to potentially conflicting updates, resulting in training instability and failure of convergence. 
Lastly, some algorithms, such as IPPO and MAPPO, are developed based on intuition and empirical results. The lack of theory compromises their trustworthiness for important usage.

To resolve these challenges, in this work we propose \textsl{Heterogeneous-Agent Reinforcement Learning} (HARL) algorithm series, that is meant for the general \emph{heterogeneous}-agent settings, achieves effective coordination through a novel \textsl{sequential update scheme}, and is grounded theoretically.

In particular, we capitalize on the \textsl{multi-agent advantage decomposition lemma} \citep{kuba2021settling} and derive the theoretically underpinned multi-agent extension of trust region learning, which is proved to enjoy monotonic improvement property and convergence to the Nash Equilibrium (NE) guarantee. Based on this, we propose Heterogeneous-Agent Trust Region Policy Optimisation (HATRPO) and Heterogeneous-Agent Proximal Policy Optimisation (HAPPO) as tractable approximations to theoretical procedures. 

Furthermore, inspired by Mirror Learning \citep{kuba2022mirror} that provides a theoretical explanation for the effectiveness of TRPO and PPO
, we discover a novel framework named \textsl{Heterogeneous-Agent Mirror Learning} (HAML), which strengthens theoretical guarantees for HATRPO and HAPPO and provides a general template for cooperative MARL algorithmic designs.
We prove that all algorithms derived from HAML inherently satisfy the desired property of the monotonic improvement of joint return and the convergence to Nash equilibrium.
Thus, HAML dramatically expands the theoretically sound algorithm space and, potentially, provides cooperative MARL solutions to more practical settings. 
We explore the HAML class and derive more theoretically underpinned and practical heterogeneous-agent algorithms, including HAA2C, HADDPG, and HATD3. 

To facilitate the usage of HARL algorithms, we open-source our PyTorch-based integrated implementation. Based on this, we test HARL algorithms comprehensively on Multi-Agent Particle Environment (MPE) \citep{maddpg, mordatch2018emergence}, Multi-Agent MuJoCo (MAMuJoCo) \citep{peng2021facmac}, StarCraft Multi-Agent Challenge (SMAC) \citep{samvelyanstarcraft}, SMACv2 \citep{ellis2022smacv2}, Google Research Football Environment (GRF) \citep{kurach2020google}, and Bi-DexterousHands \citep{chen2022towards}. The empirical results confirm the algorithms' effectiveness in practice. 
On all benchmarks with heterogeneous agents including MPE, MAMuJoCo, GRF, and Bi-Dexteroushands, HARL algorithms generally outperform their existing MA-counterparts, and their performance gaps become larger as the heterogeneity of agents increases, showing that HARL algorithms are more robust and better suited for the general heterogeneous-agent settings. While all HARL algorithms show competitive performance, they culminate in HAPPO and HATD3 in particular, which establish the new state-of-the-art results. As an off-policy algorithm, HATD3 also improves sample efficiency, leading to more efficient learning and faster convergence. On tasks where agents are mostly homogeneous such as SMAC and SMACv2, HAPPO and HATRPO attain comparable or superior win rates at convergence while not relying on the parameter-sharing trick, demonstrating their general applicability.
Through ablation analysis, we empirically show the novelties introduced by HARL theory and algorithms are crucial for learning the optimal cooperation strategy, thus signifying their importance. Finally, we systematically analyse the computational overhead of sequential update and conclude that it does not need to be a concern.

\section{Preliminaries}

\label{sec:prelim}

In this section, we first introduce problem formulation and notations for cooperative MARL, and then review existing work and analyse their limitations.
\subsection{Cooperative MARL Problem Formulation and Notations}

We consider a fully cooperative multi-agent task that can be described as a Markov game (MG) \citep{littman1994markov}, also known as a stochastic game \citep{shapley1953stochastic}.

\begin{restatable}{definition}{mdp}
A cooperative Markov game is defined by a tuple $\langle \mathcal{N}, \mathcal{S}, \boldsymbol{\mathcal{A}}, r, P, \gamma, d\rangle$. Here, $\mathcal{N}=\{1, \dots, n\}$ is a set of $n$ agents, $\mathcal{S}$ is the state space, $\boldsymbol{\mathcal{A}}=\times_{i=1}^{n}\mathcal{A}^i$ is the products of all agents' action spaces, known as the joint action space. Further, $r:\mathcal{S}\times\boldsymbol{\mathcal{A}}\rightarrow \mathbb{R}$ is the joint reward function, $P:\mathcal{S}\times\boldsymbol{\mathcal{A}}\times\mathcal{S}\rightarrow [0,1]$ is the transition probability kernel, $\gamma\in[0, 1)$ is the discount factor, and $d\in\mathcal{P}(\mathcal{S})$ (where $\mathcal{P}(X)$ denotes the set of probability distributions over a set $X$) is the positive initial state distribution. 
\end{restatable}

Although our results hold for general compact state and action spaces, in this paper we assume that they are finite, for simplicity. In this work, we will also use the notation $\mathbb{P}(X)$ to denote the power set of a set $X$. At time step $t\in\mathbb{N}$, the agents are at state $\rs_t$; they take independent actions $\ra^i_t, \forall i\in\mathcal{N}$ drawn from their policies $\pi^i(\cdot^i|\rs_t)\in\mathcal{P}(\mathcal{A}^i)$, and equivalently, they take a joint action $\rva_t=(\ra^1_t, \dots, \ra^n_t)$ drawn from their joint policy $\vpi(\cdot|\rs_t)=\prod_{i=1}^{n}\pi^i(\cdot^i|\rs_t)\in\mathcal{P}(\boldsymbol{\mathcal{A}})$. We write $\Pi^i\triangleq \{ \times_{s\in\mathcal{S}}\pi^i(\cdot^i|s) \ | \forall s\in\mathcal{S}, \pi^{i}(\cdot^i|s)\in\mathcal{P}(\mathcal{A}^i) \}$ to denote the policy space of agent $i$, and $\boldsymbol{\Pi}\triangleq (\Pi^1, \dots, \Pi^n)$ to denote the joint policy space. It is important to note that when $\pi^i(\cdot^i|s)$ is a Dirac delta distribution, $\forall s \in \mathcal{S}$, the policy is referred to as \textit{deterministic} \citep{silver2014deterministic} and we write $\mu^i(s)$ to refer to its centre. Then, the environment emits the joint reward $\rr_{t} = r(\rs_t, \rva_t)$ and moves to the next state $\rs_{t+1}\sim P(\cdot|\rs_t, \rva_t)\in\mathcal{P}(\mathcal{S})$. The joint policy $\boldsymbol{\pi}$, the transition probabililty kernel $P$, and the initial state distribution $d$, induce a marginal state distribution at time $t$, denoted by $\rho^{t}_{\boldsymbol{\pi}}$. We define an (improper) marginal state distribution $\rho_{\boldsymbol{\pi}} \triangleq \sum_{t=0}^{\infty}\gamma^{t}\rho^{t}_{\boldsymbol{\pi}}$. The state value function and the state-action value function are defined as: 
\begin{align*}
V_{\boldsymbol{\pi}}(s) \triangleq         \E_{\rva_{0:\infty}\sim\boldsymbol{\pi}, \rs_{1:\infty}\sim P}\big[ \sum_{t=0}^{\infty}\gamma^{t}\rr_{t}
\big| \ \rs_{0} = s \big]
\end{align*}
and\footnote{We write $a^i$, $\va$, and $s$ when we refer to the action, joint action, and state as to values, and $\ra^i$, $\rva$, and $\rs$ as to random variables.}
\begin{align*}
Q_{\boldsymbol{\pi}}(s, \va)  \triangleq \E_{\rs_{1:\infty}\sim P, \rva_{1:\infty}\sim\boldsymbol{\pi}}\big[ \sum_{t=0}^{\infty}\gamma^{t}\rr_{t} 
    \big| \ \rs_{0} = s,  \ \rva_{0} = \va \big].
\end{align*}
The advantage function is defined to be
\begin{align*}
A_{\boldsymbol{\pi}}(s, \va) \triangleq Q_{\boldsymbol{\pi}}(s, \va) - V_{\boldsymbol{\pi}}(s). 
\end{align*}
In this paper, we consider the fully-cooperative setting where the agents aim to maximise the expected joint return, defined as 
\begin{align*}
    J(\boldsymbol{\pi}) \triangleq \E_{\rs_{0:\infty}\sim \rho^{0:\infty}_{\boldsymbol{\pi}}, \rva_{0:\infty}\sim\boldsymbol{\pi}}\left[ \sum_{t=0}^{\infty}\gamma^{t}\rr_{t} \right].
\end{align*} 

We adopt the most common solution concept for multi-agent problems which is that of Nash equilibrium (NE) \citep{nash1951non,yang2020overview,filar2012competitive,bacsar1998dynamic},  defined as follows.

\begin{restatable}{definition}{NE}
\label{definition:ne}
In a fully-cooperative game, a joint policy $\boldsymbol{\pi}_* = (\pi_*^1, \dots, \pi_*^n)$ is a Nash equilibrium (NE) if for every $i\in\mathcal{N}$, $\pi^i\in\Pi^i$ implies $J\left(\boldsymbol{\pi}_*\right) \geq J\left(\pi^i,  \boldsymbol{\pi}_*^{-i}\right)$.
\end{restatable}

NE is a well-established game-theoretic solution concept. Definition \ref{definition:ne} characterises the equilibrium point at convergence for cooperative MARL tasks. To study the problem of finding a NE, we pay close attention to the contribution to performance from different subsets of agents. To this end, we introduce the following novel definitions.   
\begin{restatable}{definition}{multiagentfunctions}
    Let $i_{1:m}$  denote an ordered subset $\{i_{1}, \dots, i_{m}\}$ of $\mathcal{N}$. We write $-i_{1:m}$ to refer to its complement, and $i$ and $-i$, respectively, when $m=1$. We write $i_k$ when we refer to the $k^{\text{th}}$ agent in the ordered subset. Correspondingly, the multi-agent state-action value function is defined as
    \begin{align}
        Q_{\boldsymbol{\pi}}^{i_{1:m}}\left(s, \va^{i_{1:m}}\right) \triangleq \E_{\rva^{-i_{1:m}}\sim\boldsymbol{\pi}^{-i_{1:m}}}\left[ Q_{\boldsymbol{\pi}}\left(s, \va^{i_{1:m}}, \rva^{-i_{1:m}}\right) \right],\nonumber
    \end{align}
    In particular, when $m=n$ (the joint action of all agents is considered), then $i_{1:n}\in\text{Sym}(n)$, where $\text{Sym}(n)$ denotes the set of permutations of integers $1,\dots, n$, known as the \textbf{symmetric group}. In that case, $Q^{i_{1:n}}_{\vpi}(s, \va^{i_{1:n}})$ is equivalent to $Q_{\vpi}(s, \va)$. On the other hand, when $m=0$, i.e., $i_{1:m}=\emptyset$, the function takes the form of $V_{\vpi}(s)$. Moreover, consider two disjoint subsets of agents, $j_{1:k}$ and $i_{1:m}$. Then, the multi-agent advantage function of $i_{1:m}$ with respect to $j_{1:k}$ is defined as 
    \begin{align}
            \label{eq:single-agent-advantage}
         A_{\boldsymbol{\pi}}^{i_{1:m}}\left(s, \va^{j_{1:k}}, \va^{i_{1:m}} \right) \triangleq  Q_{\boldsymbol{\pi}}^{j_{1:k}, i_{1:m}}\left( s, \va^{j_{1:k}}, \va^{i_{1:m}}\right) 
        - Q_{\boldsymbol{\pi}}^{j_{1:k}}\left( s, \va^{j_{1:k}}\right).
    \end{align}
    
\end{restatable}

In words, $Q_{\boldsymbol{\pi}}^{i_{1:m}}\left(s, \va^{i_{1:m}}\right)$ evaluates the value of agents $i_{1:m}$ taking actions $\va^{i_{1:m}}$ in state $s$ while marginalizing out $\rva^{-i_{1:m}}$, and $A_{\boldsymbol{\pi}}^{i_{1:m}}\left(s, \va^{j_{1:k}}, \va^{i_{1:m}} \right)$ evaluates the advantage of agents $i_{1:m}$ taking actions $\va^{i_{1:m}}$ in state $s$ given that the actions taken by agents $j_{1:k}$ are $\va^{j_{1:k}}$, with the rest of agents' actions marginalized out by expectation. As we show later in Section \ref{sec:HARL}, these functions allow to decompose the joint advantage function, thus shedding new light on the credit assignment problem.

\subsection{Dealing With Partial Observability}
\label{sec:partial-obs}

Notably, in some cooperative multi-agent tasks, the global state $s$ may be only partially observable to the agents. That is, instead of the omniscient global state, each agent can only perceive a local observation of the environment, which does not satisfy the \emph{Markov property}. The model that accounts for partial observability is Decentralized Partially Observable Markov Decision Process (Dec-POMDP) \citep{oliehoek2016concise}. However, Dec-POMDP is proved to be NEXP-complete \citep{bernstein2002complexity} and requires super-exponential time to solve in the worst case \citep{zhang2021multi}. To obtain tractable results, we assume full observability in theoretical derivations and let each agent take actions conditioning on the global state, \emph{i.e.}, $a_t^i \sim \pi^i(\cdot^i|s)$, thereby arriving at practical algorithms. In literature \citep{yang2018mean, kuba2021settling, wang2023more}, this is a common modeling choice for rigor, consistency, and simplicity of the proofs.

In our implementation, we either compensate for partial observability by employing RNN so that agent actions are conditioned on the action-observation history, or directly use the MLP network so that agent actions are conditioned on the partial observations. Both of them are common approaches adopted by existing work, including MAPPO \citep{mappo}, QMIX \citep{rashid2018qmix}, COMA \citep{foerster2018counterfactual}, OB \citep{kuba2021settling}, MACPF \citep{wang2023more} etc.. From our experiments (Section \ref{sec:experiments}), we show that both approaches are capable of solving partially observable tasks.

\subsection{The State of Affairs in Cooperative MARL}
\label{sec:affairs}

Before we review existing SOTA algorithms for cooperative MARL, we introduce two settings in which the algorithms can be implemented. Both of them can be considered appealing depending on the application, but their benefits also come with limitations which, if not taken care of, may deteriorate an algorithm's performance and applicability.

\subsubsection{Homogeneity \emph{vs.} Heterogeneity}
\label{sec:hh}

The first setting is that of \textit{homogeneous} policies, \emph{i.e.}, those where all agents share one set of policy parameters: $\pi^i=\pi, \forall i \in \mathcal{N}$, so that $\vpi=(\pi, \dots, \pi)$ \citep{de2020independent, mappo}, commonly referred to as \emph{Full Parameter Sharing} (FuPS) \citep{christianos2021scaling}. This approach enables a straightforward adoption of an RL algorithm to MARL, and it does not introduce much computational and sample complexity burden with the increasing number of agents. As such, it has been a common practice in the MARL community to improve sample efficiency and boost algorithm performance \citep{sunehag2018value, foerster2018counterfactual, rashid2018qmix}. However, FuPS could lead to an exponentially-suboptimal outcome in the extreme case (see Example \ref{eg:suboptimal} in Appendix \ref{appendix:proof-of-examples}). While agent identity information could be added to observation to alleviate this difficulty, FuPS+id still suffers from interference during agents' learning process in scenarios where they have different abilities and goals, resulting in poor performance, as analysed by \cite{christianos2021scaling} and shown by our experiments (Figure \ref{fig:humanoid}). One remedy is the \emph{Selective Parameter Sharing} (SePS) \citep{christianos2021scaling}, which only shares parameters among similar agents. Nevertheless, this approach has been shown to be suboptimal and highly scenario-dependent, emphasizing the need for prior understanding of task and agent attributes to effectively utilize the SePS strategy \citep{hu2022policy}. More severely, both FuPS and SePS require the observation and action spaces of agents in a sharing group to be the same, restricting their applicability to the general heterogeneous-agent setting. Existing work that extends parameter sharing to heterogeneous agents relies on \emph{padding} \citep{terry2020parameter}, which also cannot be generally applied. To summarize, algorithms relying on parameter sharing potentially suffer from compromised performance and applicability.

A more ambitious approach to MARL is to allow for \textit{heterogeneity} of policies among agents, \emph{i.e.}, to let $\pi^i$ and $\pi^j$ be different functions when $i\neq j\in\mathcal{N}$. This setting has greater applicability as heterogeneous agents can operate in different action spaces. Furthermore, thanks to this model's flexibility they may learn more sophisticated joint behaviors. Lastly, they can recover homogeneous policies as a result of training, if that is indeed optimal.

Nevertheless, training heterogeneous agents is highly non-trivial. Given a joint reward, an individual agent may not be able to distill its own contribution to it --- a problem known as \textit{credit assignment} \citep{foerster2018counterfactual, kuba2021settling}. Furthermore, even if an agent identifies its improvement direction, it may conflict with those of other agents when not optimised properly. We provide two examples to illustrate this phenomenon.

The first one is shown in Figure \ref{fig:simultaneous-vs-sequential}. We design a single-state differentiable game where two agents play continuous actions $a^1, a^2 \in \mathbb{R}$ respectively, and the reward function is $r(a^1, a^2) = a^1a^2$. When we initialise agent policies in the second or fourth quadrants and set a large learning rate, the simultaneous update approach could result in a decrease in joint reward. In contrast, the sequential update proposed in this paper enables agent 2 to fully adapt to agent 1's updated policy and improves the joint reward.

\begin{figure}[tbp]
  \centering
  \includegraphics[width=0.8\linewidth]{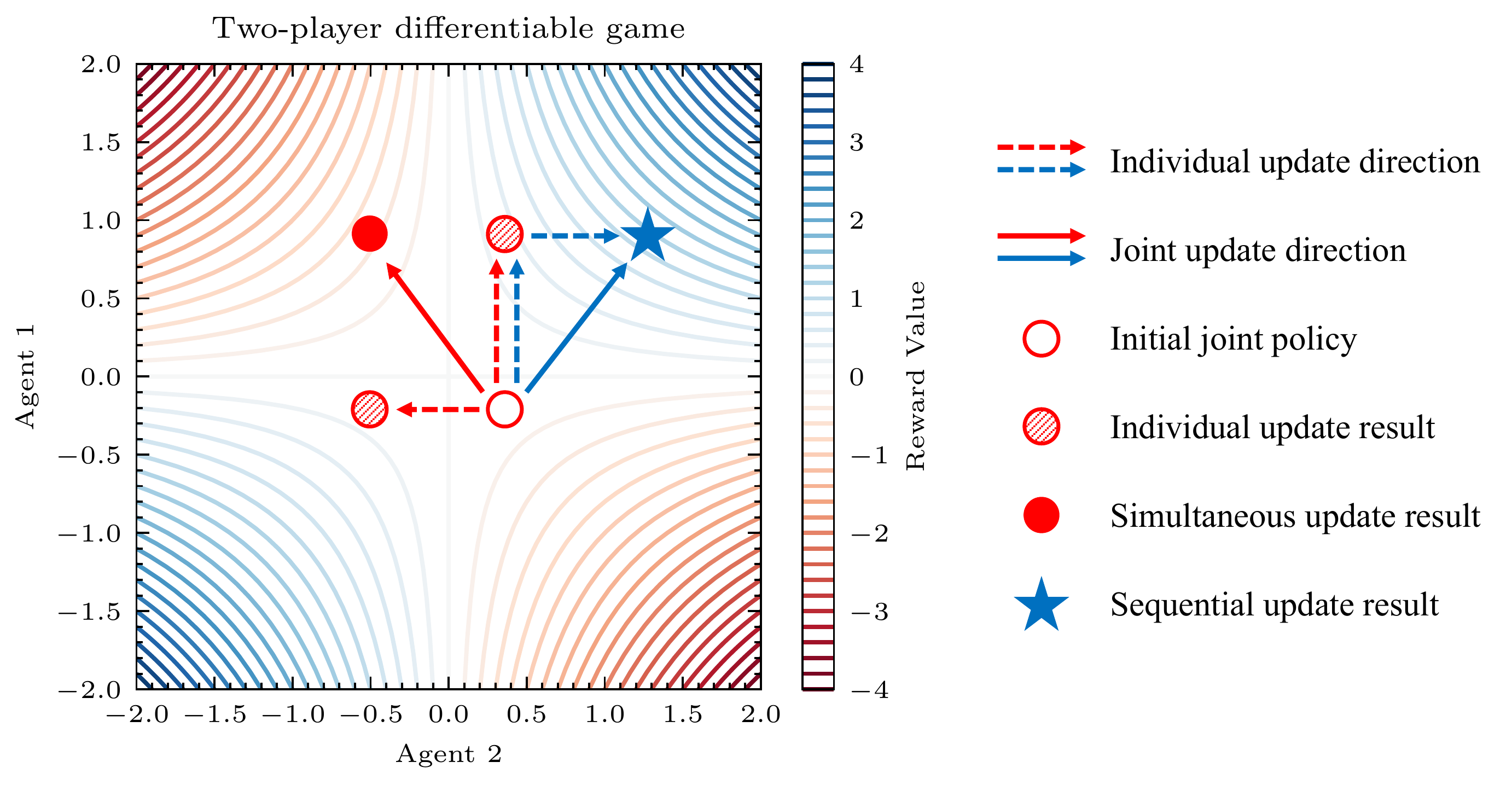}
    \caption{Example of a two-agent differentiable game with $r(a^1, a^2)=a^1a^2$. We initialise the two policies in the fourth quadrant. Under the straightforward simultaneous update scheme (red), agent 1 takes a positive update to improve the joint reward, meanwhile agent 2 moves towards the negative axis for the same purpose. However, their update directions conflict with each other and lead to a decrease in the joint return. By contrast, under our proposed sequential update scheme (blue), agent 1 updates first, and agent 2 adapts to agent 1' updated policy, jointly leading to improvement.}
  
  \label{fig:simultaneous-vs-sequential}
\end{figure}

We consider a matrix game with discrete action space as the second example. Our matrix game is illustrated as follows:
\begin{restatable}{example}{needscare}
\label{eg:needscare}
Let's consider a fully-cooperative game with $2$ agents, one state, and the joint action space $\{0, 1\}^2$, where the reward is given by $r(0, 0)=0, r(0,1)=r(1,0)=2,$ and $r(1, 1)=-1$. Suppose that $\pi_{\text{old}}^i(0) > 0.6$ for $i=1,2$. Then, if agents $i$ update their policies by 
\begin{align}
\pi_{\text{new}}^i = \argmax_{\color{red}\pi^i\color{black}}\E_{\ra^i\sim\color{red}\pi^i\color{black}, \ra^{-i}\sim\pi^{-i}_{\text{old}}}\big[ A_{\vpi_{\text{old}}}(\ra^i, \ra^{-i})\big], \forall i\in\mathcal{N}, \nonumber
\end{align}
then the resulting policy will yield a lower return, 
\begin{align}
J(\vpi_{\text{old}}) > J(\vpi_{\text{new}}) = \min_{\vpi}J(\vpi). \nonumber
\end{align}
\end{restatable}

This example helpfully illustrates the miscoordination problem when agents conduct independent reward maximisation simultaneously. A similar miscoordination problem when heterogeneous agents update at the same time is also shown in Example 2 of \cite{alos2010logit}.

Therefore, our discussion in this section not only implies that homogeneous algorithms could have restricted performance and applicability, but also highlight that heterogeneous algorithms should be developed with extra care when not optimised properly (large learning rate in Figure \ref{fig:simultaneous-vs-sequential} and independent reward maximisation in Example \ref{eg:needscare}), which could be common in complex high-dimensional problems. In the next subsection, we describe existing SOTA actor-critic algorithms which, while often very effective, are still not impeccable, as they suffer from one of the above two limitations.

\subsubsection{Analysis of Existing Work}
\label{sec:sota}
MAA2C  \citep{papoudakis2021benchmarking} extends the A2C \citep{mnih2016asynchronous} to MARL by replacing the RL optimisation (single-agent policy) objective with the MARL one (joint policy),
\begin{align}
    \label{eq:maa2c}
    \mathcal{L}^{\text{MAA2C}}(\vpi) \triangleq \E_{\rs\sim\vpi, \rva\sim\vpi}\big[ A_{\vpi_{\text{old}}}(\rs, \rva) \big],
\end{align}
which computes the gradient with respect to every agent $i$'s policy parameters, and performs a gradient-ascent update for each agent. This algorithm is straightforward to implement and is capable of solving simple multi-agent problems \citep{papoudakis2021benchmarking}. We point out, however, that by simply following their own MAPG, the agents could perform uncoordinated updates, as illustrated in Figure \ref{fig:simultaneous-vs-sequential}. Furthermore, MAPG estimates have been proved to suffer from large variance which grows linearly with the number of agents \citep{kuba2021settling}, thus making the algorithm unstable. To assure greater stability, the following MARL methods, inspired by stable RL approaches, have been developed.

MADDPG \citep{maddpg} is a MARL extension of the popular DDPG algorithm \citep{lillicrap2015continuous}. At every iteration, every agent $i$ updates its deterministic policy by maximising the following objective 
\begin{align}
    \label{eq:maddpg}
    \mathcal{L}^{\text{MADDPG}}_{i}(\mu^{i}) \triangleq \E_{\rs\sim\beta_{\vmu_{\text{old}}}}\Big[ Q^{i}_{\vmu_{\text{old}}}\big(\rs, \mu^{i}(\rs)\big)\Big]
    = 
    \E_{\rs\sim\beta_{\vmu_{\text{old}}}}\Big[ Q_{\vmu_{\text{old}}}\big(\rs, \mu^{i}(\rs), \vmu^{-i}_{\text{old}}(\rs)\big)\Big],
\end{align}
where $\beta_{\vmu_{\text{old}}}$ is a state distribution that is not necessarily equivalent to $\rho_{\vmu_{\text{old}}}$, thus allowing for off-policy training. In practice, MADDPG maximises Equation (\ref{eq:maddpg}) by a few steps of gradient ascent. 
The main advantages of MADDPG include a small variance of its MAPG estimates---a property granted by deterministic policies \citep{silver2014deterministic}, as well as low sample complexity due to learning from off-policy data. Such a combination makes the algorithm competitive on certain continuous-action tasks \citep{maddpg}. However, MADDPG does not address the multi-agent credit assignment problem \citep{foerster2018counterfactual}. Plus, when training the decentralised actors, MADDPG does not take into account the updates agents have made and naively uses the off-policy data from the replay buffer which, much like in Section \ref{sec:hh}, leads to uncoordinated updates and suboptimal performance in the face of harder tasks \citep{peng2021facmac, ray}. MATD3 \citep{ackermann2019reducing} proposes to reduce overestimation bias in MADDPG using double centralized critics, which improves its performance and stability but does not help with getting rid of the aforementioned limitations.

MAPPO \citep{mappo} is a relatively straightforward extension of PPO \citep{ppo} to MARL. In its default formulation, the agents employ the trick of \textit{parameter sharing} described in the previous subsection. As such, the policy is updated to maximise 
\begin{align}
\scriptsize
    \label{eq:mappo}
    \mathcal{L}^{\text{MAPPO}}(\pi) \triangleq \E_{\rs\sim\rho_{\vpi_{\text{old}}}, \rva\sim\vpi_{\text{old}}}\Bigg[ \sum_{i=1}^{n} \min\Big( \frac{\pi(\ra^i|\rs)}{\pi_{\text{old}}(\ra^i|\rs)}A_{\vpi_{\text{old}}}(\rs, \rva), \text{clip}\big(\frac{\pi(\ra^i|\rs)}{\pi_{\text{old}}(\ra^i|\rs)}, 1\pm\epsilon\big)A_{\vpi_{\text{old}}}(\rs, \rva)\Big)\Bigg],
\end{align}
where the $\text{clip}(\cdot, 1\pm\epsilon)$ operator clips the input to $1-\epsilon$/$1+\epsilon$ if it is below/above this value. Such an operation removes the incentive for agents to make large policy updates, thus stabilising the training effectively. Indeed, the algorithm's performance on the StarCraftII benchmark is remarkable, and it is accomplished by using only on-policy data. Nevertheless, the parameter-sharing strategy limits the algorithm's applicability and could lead to its suboptimality when agents have different roles. In trying to avoid this issue, one can implement the algorithm without parameter sharing, thus making the agents simply take simultaneous PPO updates meanwhile employing a joint advantage estimator. In this case, the updates could be uncoordinated, as we discussed in Section \ref{sec:hh}.

In summary, all these algorithms do not possess performance guarantees. 
Altering their implementation settings to avoid one of the limitations from Section \ref{sec:hh} makes them, at best, fall into another. This shows that the MARL problem introduces additional complexity into the single-agent RL setting, and needs additional care to be rigorously solved. With this motivation, in the next section, we propose novel heterogeneous-agent methods based on \emph{sequential update} with correctness guarantees.

\section{Our Methods}
\label{sec:HARL} 

The purpose of this section is to introduce Heterogeneous-Agent Reinforcement Learning (HARL) algorithm series which we prove to solve cooperative problems theoretically. HARL algorithms are designed for the general and expressive setting of heterogeneous agents, and their essence is to coordinate agents' updates, thus resolving the challenges in Section \ref{sec:hh}. We start by developing a theoretically justified Heterogeneous-Agent Trust Region Learning (HATRL) procedure in Section \ref{subsec:HATRL} and deriving practical algorithms, namely HATRPO and HAPPO, as its tractable approximations in Section \ref{subsec:pracalgo}.
We further introduce the novel Heterogeneous-Agent Mirror Learning (HAML) framework in Section \ref{subsec:haml}, which strengthens performance guarantees of HATRPO and HAPPO (Section \ref{subsec:hatrpo-happo-haml-instances}) and provides a general template for cooperative MARL algorithmsic design, leading to more HARL algorithms (Section \ref{subsec:new-haml}).

\subsection{Heterogeneous-Agent Trust Region Learning (HATRL)}
\label{subsec:HATRL}

Intuitively, if we parameterise all agents separately and let them learn one by one, then we will break the homogeneity constraint and allow the agents to coordinate their updates, thereby avoiding the two limitations from Section \ref{sec:affairs}. 
Such coordination can be achieved, for example, by accounting for previous agents' updates in the optimization objective of the current one along the aforementioned sequence.
Fortunately, this idea is embodied in the multi-agent advantage function $A_{\boldsymbol{\pi}}^{i_m}\left(s, \va^{i_{1:m-1}}, a^{i_m} \right) $
which allows agent $i_m$ to evaluate the utility of its action $a^{i_m}$ given actions of previous agents $\va^{i_{1:m-1}}$.
Intriguingly, multi-agent advantage functions allow for rigorous decomposition of the joint advantage function, as described by the following pivotal lemma.

\begin{restatable}[Multi-Agent Advantage Decomposition]{lemma}{maadlemma}
    \label{lemma:maadlemma}
In any cooperative Markov games, given a joint policy $\boldsymbol{\pi}$, for any state $s$, and any agent subset $i_{1:m}$, the below equation holds. 
    \begin{align}
        A^{i_{1:m}}_{\boldsymbol{\pi}}\left(s, \va^{i_{1:m}}\right)
        = \sum_{j=1}^{m}A^{i_{j}}_{\boldsymbol{\pi}}\left(s, \va^{i_{1:j-1}}, a^{i_{j}}\right). \nonumber
    \end{align}
\end{restatable}
For proof see Appendix \ref{appendix:theoretical-matrpo}.  
Notably, Lemma \ref{lemma:maadlemma} holds in general for  cooperative Markov games, with no need for  any assumptions on the decomposability of the joint value function such as  those in VDN \citep{sunehag2018value}, QMIX \citep{rashid2018qmix} or Q-DPP \citep{yang2020multi}.

\begin{figure}[tbp]
  \centering
  \includegraphics[width=\linewidth]{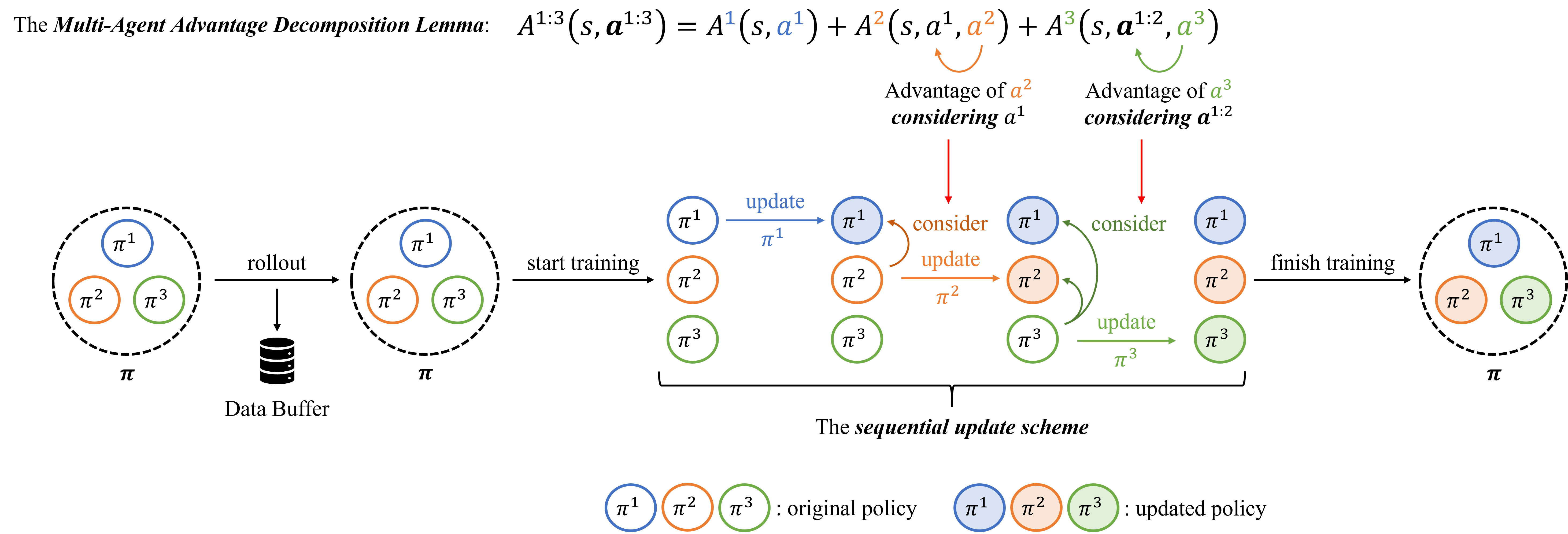}
    \caption{The \textsl{multi-agent advantage decomposition lemma} and the \textsl{sequential update scheme} are naturally consistent. The former (upper in the figure) decomposes joint advantage into sequential advantage evaluations, each of which takes into consideration previous agents' actions. Based on this, the latter (lower in the figure) allows each policy to be updated considering previous updates during the training stage. The rigor of their connection is embodied in Lemma \ref{lemma:trpotosadtrpo} and Lemma \ref{lemma:hamo}, where multi-agent advantage decomposition lemma is crucial for the proofs and leads to algorithms that employ sequential update scheme.}
  
  \label{fig:maad-sus}
\end{figure}

Lemma \ref{lemma:maadlemma} confirms that a sequential update is an effective approach to search for the direction of performance improvement (i.e., joint actions with positive advantage values) in multi-agent learning.   
That is, imagine that agents take actions sequentially by following an arbitrary order  $i_{1:n}$. Let agent $i_{1}$ take action $\bar{a}^{i_{1}}$ such that { $A_{\boldsymbol{\pi}}^{i_{1}}(s, \bar{a}^{i_{1}}) > 0$}, and then, for the remaining  $m=2, \dots, n$, each agent $i_{m}$ takes an action $\bar{a}^{i_{m}}$ such that { $A_{\boldsymbol{\pi}}^{i_{m}}(s, \bar{\va}^{i_{1:m-1}}, \bar{a}^{i_{m}}) > 0$}. 
For the induced joint action $\bar{\va}$, Lemma \ref{lemma:maadlemma} assures that { $A_{\boldsymbol{\pi}}(s, \bar{\va})$} is positive, thus the performance is guaranteed to improve.     
To formally extend the above process into a policy iteration procedure with monotonic improvement guarantee, we begin by introducing the following definitions.

\begin{restatable}{definition}{localsurrogate}
    \label{definition:localsurrogate}
     Let $\boldsymbol{\pi}$ be a joint policy, $\boldsymbol{\bar{\pi}}^{i_{1:m-1}}=\prod_{j=1}^{m-1}\bar{\pi}^{i_j}$ be some \textbf{other} joint policy of agents $i_{1:m-1}$, and $\hat{\pi}^{i_{m}}$ be some \textbf{other} policy of agent $i_{m}$. Then
     \begin{align}
         L^{i_{1:m}}_{\boldsymbol{\pi}}\left( \boldsymbol{\bar{\pi}}^{i_{1:m-1}}, \hat{\pi}^{i_{m}} \right) 
         \triangleq
         \E_{\rs \sim \rho_{\boldsymbol{\pi}}, \rva^{i_{1:m-1}}\sim\boldsymbol{\bar{\pi}}^{i_{1:m-1}}, \ra^{i_m}\sim\hat{\pi}^{i_{m}}}
         \left[  A_{\boldsymbol{\pi}}^{i_{m}}\left(\rs, \rva^{i_{1:m-1}}, \ra^{i_{m}}\right) \right].
         \nonumber
     \end{align}
\end{restatable}
Note that, for any $\boldsymbol{\bar{\pi}}^{i_{1:m-1}}$, we have
 \begin{align}
    \label{eq:nice-maad-property}
     L^{i_{1:m}}_{\boldsymbol{\pi}}\left( \boldsymbol{\bar{\pi}}^{i_{1:m-1}}, \pi^{i_{m}} \right) 
     &=
     \E_{\rs \sim \rho_{\boldsymbol{\pi}}, \rva^{i_{1:m-1}}\sim\boldsymbol{\bar{\pi}}^{i_{1:m-1}}, \ra^{i_m}\sim\pi^{i_{m}}}
     \left[  A_{\boldsymbol{\pi}}^{i_{m}}\left(\rs, \rva^{i_{1:m-1}}, \ra^{i_{m}}\right) \right]
     \nonumber\\
     &= \E_{\rs \sim \rho_{\boldsymbol{\pi}}, \rva^{i_{1:m-1}}\sim\boldsymbol{\bar{\pi}}^{i_{1:m-1}} }
     \left[  \E_{\ra^{i_m}\sim\pi^{i_m}}
     \left[A_{\boldsymbol{\pi}}^{i_{m}}\left(\rs, \rva^{i_{1:m-1}}, \ra^{i_{m}}\right) \right] \right] = 0.
 \end{align}
 
Building on Lemma \ref{lemma:maadlemma} and Definition \ref{definition:localsurrogate}, we derive the bound for joint policy update.
\begin{restatable}{lemma}{trpotosadtrpo}
\label{lemma:trpotosadtrpo}
    Let $\boldsymbol{\pi}$ be a joint policy. Then, for any joint policy $\boldsymbol{\bar{\pi}}$, we have
    \begin{align}
    &J(\boldsymbol{\bar{\pi}}) \geq
        J(\boldsymbol{\pi}) +
        \sum_{m=1}^{n}\left[L^{i_{1:m}}_{\boldsymbol{\pi}}\left(
        \boldsymbol{\bar{\pi}}^{i_{1:m-1}}, \bar{\pi}^{i_{m}}\right)
        - C\text{{\normalfont D}}_{\text{KL}}^{\text{max}}(\pi^{i_{m}}, \bar{\pi}^{i_{m}})\right], \nonumber\\
    &\qquad \text{ where } C = \frac{4\gamma\max_{s, \va}|A_{\boldsymbol{\pi}}(s, \va)|}{(1-\gamma)^{2}}. 
    \end{align}
\end{restatable}
For proof see Appendix \ref{appendix:analysis-training-the-matrpo}.
This lemma provides an idea about how a joint policy can be improved. Namely, by Equation (\ref{eq:nice-maad-property}), we know that if any agents were to set the values of the above summands { $L_{\boldsymbol{\pi}}^{i_{1:m}}(\boldsymbol{\bar{\pi}}^{i_{1:m-1}}, \bar{\pi}^{i_m}) - C\text{{\normalfont D}}_{\text{KL}}^{\text{max}}(\pi^{i_m}, \bar{\pi}^{i_m})$} by sequentially updating their policies, each of them can always make its summand be zero by making no policy update (i.e., $\bar{\pi}^{i_m} = \pi^{i_m})$. This implies that any positive update will lead to an increment in summation.  Moreover, as there are $n$ agents making policy updates, the compound increment can be large,  leading to a substantial improvement. Lastly, note that this property holds with no requirement on the specific order by which agents make their updates; this allows for flexible scheduling on the update order at each iteration. 
To summarise, we propose the following Algorithm \ref{algorithm:theoretical-matrpo}. 

\begin{algorithm}
\caption{Multi-Agent Policy Iteration with Monotonic Improvement Guarantee}
\label{algorithm:theoretical-matrpo}
Initialise the joint policy $\boldsymbol{\pi}_{0} = (\pi^{1}_{0}, \dots, \pi^{n}_{0})$.

\For{$k=0, 1, \dots$}{
    Compute the advantage function $A_{\boldsymbol{\pi}_{k}}(s, \va)$ for all state-(joint)action pairs $(s, \va)$.
    
    Compute $\epsilon = \max_{s, \va}|A_{\boldsymbol{\pi}_k}(s, \va)|$
    and $C = \frac{4\gamma\epsilon}{(1-\gamma)^{2}}$.
    
    Draw a permutation $i_{1:n}$ of agents at random.
    
    \For{$m=1:n$}{
        Make an update $\pi^{i_{m}}_{k+1} = \argmax_{\pi^{i_{m}}}\left[ 
        L_{\boldsymbol{\pi}_k}^{i_{1:m}}\left(\boldsymbol{\pi}^{i_{1:m-1}}_{k+1}, \pi^{i_{m}}\right) - C\text{{\normalfont D}}_{\text{KL}}^{\text{max}}(\pi^{i_{m}}_{k}, \pi^{i_{m}})\right]$.
    }
}

\end{algorithm}

We want to highlight that the algorithm is markedly different from naively applying the  TRPO update on the joint policy of all agents.  
Firstly, our Algorithm \ref{algorithm:theoretical-matrpo} does not update the entire joint policy at once, but rather updates each agent's individual policy sequentially. 
Secondly, during the sequential update, each agent has a unique optimisation objective that takes into account all previous agents' updates, which is also the key for the monotonic improvement property to hold. We justify by the following theorem that Algorithm \ref{algorithm:theoretical-matrpo} enjoys monotonic improvement property.
\begin{restatable}{theorem}{matrpomonotonic} 
\label{theorem:monotonic-matrpo}
A sequence $\left(\boldsymbol{\pi}_{k}\right)_{k=0}^{\infty}$ of joint policies updated by Algorithm \ref{algorithm:theoretical-matrpo} has the monotonic improvement property, i.e., $J(\boldsymbol{\pi}_{k+1})\geq J(\boldsymbol{\pi}_{k})$ for all $k\in\mathbb{N}$.
\end{restatable}
For proof see Appendix \ref{appendix:analysis-training-the-matrpo}. 
With the above theorem, we claim a successful development of Heterogeneous-Agent Trust Region Learning (HATRL), as it retains the monotonic improvement property of trust region learning.  
Moreover, we take a step further to prove Algorithm \ref{algorithm:theoretical-matrpo}'s asymptotic convergence behavior towards NE.  

\begin{restatable}{theorem}{matrpoconvergence}
\label{proposition:convergence-matrpo}
Supposing in Algorithm \ref{algorithm:theoretical-matrpo} any permutation of agents has a fixed non-zero probability to begin the update, a sequence $\left(\boldsymbol{\pi}_{k}\right)_{k=0}^{\infty}$ of joint policies generated by the algorithm, in a cooperative Markov game, has a non-empty set of limit points, each of which is a Nash equilibrium. 
\end{restatable}
For proof see Appendix \ref{appendix:analysis-convergence-thematrpo}. In deriving this result, the novel details introduced by Algorithm \ref{algorithm:theoretical-matrpo} played an important role. The monotonic improvement property (Theorem \ref{theorem:monotonic-matrpo}), achieved through the multi-agent advantage decomposition lemma and the sequential update scheme, provided us with a guarantee of the convergence of the return. Furthermore, randomisation of the update order ensured that, at convergence, none of the agents is incentified to make an update. The proof is finalised by excluding the possibility that the algorithm converges at non-equilibrium points.

\subsection{Practical Algorithms}
\label{subsec:pracalgo}

When implementing Algorithm \ref{algorithm:theoretical-matrpo} in practice,  large state and action spaces could prevent agents from designating policies $\pi^i(\cdot|s)$ for each state $s$ separately. To handle this, we parameterise each agent's policy $\pi^i_{\theta^i}$ by  $\theta^{i}$, which, together with other agents' policies,  forms a joint policy $\boldsymbol{\pi}_{\vtheta}$  parametrised by $\vtheta = (\theta^1, \dots, \theta^n)$. 
In this subsection, we develop two deep MARL algorithms to optimise the $\vtheta$. 
\subsubsection{HATRPO} 
Computing $\text{{\normalfont D}}_{\text{KL}}^{\text{max}}\big(\pi^{i_m}_{\theta^{i_m}_k}, \pi^{i_m}_{\theta^{i_m}}\big)$  in Algorithm \ref{algorithm:theoretical-matrpo} is challenging; it requires evaluating the KL-divergence for all states at each iteration.  
Similar to TRPO, one can ease this maximal KL-divergence penalty {$\text{{\normalfont D}}_{\text{KL}}^{\text{max}}\big(\pi^{i_m}_{\theta^{i_m}_k}, \pi^{i_m}_{\theta^{i_m}}\big)$} by replacing it with the expected KL-divergence constraint  $\E_{\rs\sim\rho_{\boldsymbol{\pi}_{\vtheta_k}}}\Big[ \text{{\normalfont D}}_{\text{KL}}\big(\pi^{i_m}_{\theta^{i_m}_k}(\cdot|\rs), \pi^{i_m}_{\theta^{i_m}}(\cdot|\rs) \big) \Big] \leq \delta$ where $\delta$ is a threshold hyperparameter and  
the expectation can be easily approximated by stochastic sampling. 
With the above amendment, we propose practical HATRPO algorithm in which,
at every iteration $k+1$, given a permutation of agents   $i_{1:n}$, agent $i_{m \in \{1,...,n\}}$ sequentially optimises its policy parameter $\theta^{i_m}_{k+1}$ by maximising a  constrained objective: 
\begin{align}
    \label{eq:matrpo-objective}
    &\theta^{i_m}_{k+1}
    = \argmax_{\theta^{i_m}}\E_{\rs\sim\rho_{\boldsymbol{\pi}_{{\vtheta}_{k}}}, \rva^{i_{1:m-1}}\sim\boldsymbol{\pi}^{i_{1:m-1}}_{{\vtheta}^{i_{1:m-1}}_{k+1}}, \ra^{i_m}\sim\pi^{i_m}_{\theta^{i_m}}}\big[ 
    A^{i_m}_{\boldsymbol{\pi}_{\vtheta_k}}(\rs, \rva^{i_{1:m-1}}, \ra^{i_m})
    \big],\nonumber\\
    &\quad \quad \quad \quad \quad \quad \text{subject to   }\E_{\rs\sim\rho_{\boldsymbol{\pi}_{\vtheta_{k}}}}\big[ 
    \text{{\normalfont D}}_{\text{KL}}\big(\pi^{i_m}_{\theta^{i_m}_{k}}(\cdot|\rs), \pi^{i_m}_{\theta^{i_m}}(\cdot|\rs)\big)\big] \leq \delta.
\end{align}
To compute the above equation, similar to TRPO, one can apply  a linear approximation to the objective function and a quadratic approximation to the KL constraint;  the optimisation problem in Equation (\ref{eq:matrpo-objective}) can be solved by a closed-form  update rule as 
\begin{align}
    \theta^{i_m}_{k+1} = \theta^{i_m}_k + \alpha^j \sqrt{\frac{2\delta}{\vg^{i_m}_k(\mH^{i_m}_k)^{-1}\vg^{i_m}_k}} (\mH^{i_m}_k)^{-1}\vg^{i_m}_k \ , 
        \label{eq:matrpo-update}
\end{align}
where {\small $\mH^{i_m}_{k} = \nabla^{2}_{\theta^{i_m}}\E_{\rs\sim\rho_{\boldsymbol{\pi}_{\vtheta_k}}}\big[ \text{{\normalfont D}}_{\text{KL}}\big( \pi^{i_m}_{\theta^{i_m}_k}(\cdot|\rs),  \pi^{i_m}_{\theta^{i_m}}(\cdot|\rs) \big)\big] \big|_{\theta^{i_m}=\theta^{i_m}_k}$}  is the Hessian of the expected KL-divergence,  $\vg^{i_m}_k$ is the gradient of the objective in Equation (\ref{eq:matrpo-objective}), 
$\alpha^j<1$ is a positive coefficient that is found via backtracking line search, and the product of  {\small $(\mH^{i_m}_k)^{-1}\vg^{i_m}_k$} can be efficiently computed with conjugate gradient algorithm.   

Estimating 
{$\E_{\rva^{i_{1:m-1}}\sim\boldsymbol{\pi}^{i_{1:m-1}}_{\vtheta_{k+1}},
\ra^{i_{m}}\sim\pi^{i_{m}}_{\theta^{i_m}}}\Big[A^{i_{m}}_{\boldsymbol{\pi}_{\vtheta_{k}}}\big(\rs, \rva^{i_{1:m-1}}, \ra^{i_{m}}\big)\Big]$} is the last missing piece for HATRPO, which poses new challenges because  each agent's objective has to take into account all previous agents' updates, and the size of input values. 
Fortunately,  with the following proposition, we can efficiently estimate this objective by a joint advantage estimator. 
\begin{restatable}{proposition}{advantageestimation}
    \label{lemma:advantage-estimation}
    Let $\boldsymbol{\pi}=\prod_{j=1}^{n}\pi^{i_j}$ be a joint policy, and  $A_{\boldsymbol{\pi}}(\rs, \rva)$ be its joint advantage function. Let $\boldsymbol{\bar{\pi}}^{i_{1:m-1}} =\prod_{j=1}^{m-1}\bar{\pi}^{i_j}$ be some \textbf{other} joint policy of agents $i_{1:m-1}$, and $\hat{\pi}^{i_m}$ be some \textbf{other} policy of agent $i_m$. Then, for every state $s$,
    \begin{align}
   & \E_{\rva^{i_{1:m-1}}\sim\boldsymbol{\bar{\pi}}^{i_{1:m-1}}, \ra^{i_m}\sim\hat{\pi}^{i_m}}\big[ A^{i_{m}}_{\boldsymbol{\pi}}\big(s, \rva^{i_{1:m-1}}, \ra^{i_{m}}\big) \big] \nonumber \\
        & \ \ \ \ \ \ \ \ \ \ \ \ \ \ \ \ \ \ \ \ \ \ \ \ \ \ \ \ \ \  \ \ \ \ \ \ \ \ \ \ \ \ \ \  = \E_{\rva\sim\boldsymbol{\pi}}\Big[
        \Big(\frac{\hat{\pi}^{i_m}(\ra^{i_m}|s) }{ \pi^{i_m}(\ra^{i_m}|s)} 
        -1\Big)  \frac{\boldsymbol{\bar{\pi}}^{i_{1:m-1}}(\rva^{i_{1:m-1}}|s)}
        {\boldsymbol{\pi}^{i_{1:m-1}}(\rva^{i_{1:m-1}}|s)}A_{\boldsymbol{\pi}}(s, \rva)  \Big].  
        \label{eq:joint-adv}
    \end{align}
\end{restatable}
For proof see Appendix \ref{appendix:proof-of-advantage-estimation}. One benefit of applying Equation (\ref{eq:joint-adv}) is that agents only need to maintain a joint advantage estimator $A_{\boldsymbol{\pi}}(\rs, \rva)$ rather than one centralised critic for each individual agent (e.g., unlike  CTDE methods such as MADDPG). 
Another practical benefit one can draw is that, given an estimator {$\hat{A}(\rs, \rva)$} of the advantage function {$A_{\boldsymbol{\pi}_{\vtheta_{k}}}(\rs, \rva)$}, for example, GAE \citep{schulman2015high}, 
{\small $\E_{\rva^{i_{1:m-1}}\sim\boldsymbol{\pi}^{i_{1:m-1}}_{\vtheta^{i_{1:m-1}}_{k+1}}, \ra^{i_{m}}\sim\pi^{i_{m}}_{\theta^{i_m}}}\left[
A^{i_{m}}_{\boldsymbol{\pi}_{\vtheta_{k}}}\left(s, \rva^{i_{1:m-1}}, \ra^{i_{m}} \right)\right]$} can be estimated
with an estimator  of 
\begin{align}
    \label{eq:mad-estimator}
        \Big(\frac{ \pi^{i_{m}}_{\theta^{i_m}}(\ra^{i_{m}}|s) }{ \pi^{i_{m}}_{\theta_{k}^{i_m}}(\ra^{i_{m}}|s) } - 1\Big)M^{i_{1:m}}\big(s, \rva \big),
        \ \ \ \ \ \text{where} \ \ 
        M^{i_{1:m}} = \frac{\boldsymbol{\pi}^{i_{1:m-1}}_{\vtheta^{i_{1:m-1}}_{k+1}}(\rva^{i_{1:m-1}}|s)}
        {\boldsymbol{\pi}^{i_{1:m-1}}_{\vtheta^{i_{1:m-1}}_{k}}(\rva^{i_{1:m-1}}|s)}\hat{A}\big(s, \rva \big).
\end{align}
Notably, Equation (\ref{eq:mad-estimator}) aligns nicely with the sequential update scheme in HATRPO. 
For agent  $i_m$,  since previous agents $i_{1:m-1}$ have already made their updates, the compound policy ratio for $M^{i_{1:m}}$ in Equation (\ref{eq:mad-estimator}) is easy to compute. 
Given a batch $\mathcal{B}$ of trajectories with length $T$, we can estimate the gradient with respect to policy parameters (derived in Appendix \ref{appendix:derive-grad}) as follows, 
\begin{align}
    \hat{\vg}^{i_m}_k = \frac{1}{|\mathcal{B}|}\sum_{\tau\in \mathcal{B}}\sum\limits_{t=0}^{T}M^{i_{1:m}}(\rs_t, \rva_t)\nabla_{\theta^{i_m}}\log\pi^{i_m}_{\theta^{i_m}}(\ra^{i_m}_t|\rs_t)\big|_{\theta^{i_m}=\theta^{i_m}_k}.\nonumber
\end{align}
The term $-1\cdot M^{i_{1:m}}(\rs, \rva)$ of Equation (\ref{eq:mad-estimator}) is not reflected in $\hat{\vg}^{i_m}_k$, as it only introduces a constant with zero gradient. Along with the Hessian  of the expected KL-divergence, \emph{i.e.}, $\mH^{i_m}_k$, 
we can update $\theta_{k+1}^{i_m}$ by following Equation (\ref{eq:matrpo-update}). 
The detailed pseudocode of HATRPO is listed in Appendix \ref{appendix:matrpo}.

\subsubsection{HAPPO}
To further alleviate the computation burden from $\mH^{i_m}_k$ in HATRPO, one can follow the idea of  PPO  by considering only using first-order derivatives. 
This is achieved by making agent $i_m$ choose a policy parameter $\theta^{i_m}_{k+1}$ which maximises the  clipping objective of 
\begin{align}
    & \E_{\rs\sim\rho_{\boldsymbol{\pi}_{\vtheta_{k}}}, \rva\sim\boldsymbol{\pi}_{\vtheta_{k}}}
    \Bigg[  
    \min\Bigg( \frac{ \pi^{i_{m}}_{\theta^{i_{m}}}(\ra^{i_m}|\rs) }{ \pi^{i_{m}}_{\theta^{i_{m}}_{k}}(\ra^{i_m}|\rs)} 
    M^{i_{1:m}}\left( \rs, \rva \right) ,
    \text{clip}\bigg( 
    \frac{ \pi^{i_{m}}_{\theta^{i_{m}}}(\ra^{i_m}|\rs) }{ \pi^{i_{m}}_{\theta^{i_{m}}_{k}}(\ra^{i_m}|\rs)},
    1\pm \epsilon
    \bigg)
    M^{i_{1:m}}\left( \rs, \rva \right) 
    \Bigg)
    \Bigg].
    \end{align}
The optimisation process can be performed by stochastic gradient methods such as Adam \citep{kingma2014adam}. 
We refer to the above procedure as HAPPO and Appendix \ref{appendix:mappo} for its full pseudocode.

\subsection{Heterogeneous-Agent Mirror Learning: A Continuum of Solutions to Cooperative MARL}
\label{subsec:haml}

Recently, Mirror Learning \citep{kuba2022mirror} provided a theoretical explanation of the effectiveness of TRPO and PPO in addition to the original trust region interpretation, and unifies a class of policy optimisation algorithms. Inspired by their work, we further discover a novel theoretical framework
for cooperative MARL, 
named \textsl{Heterogeneous-Agent Mirror Learning} (HAML), which enhances theoretical guarantees of HATRPO and HAPPO. As a proven template for algorithmic designs, HAML substantially generalises the desired guarantees of monotonic improvement and NE convergence to a continuum of algorithms and naturally hosts HATRPO and HAPPO as its instances, further explaining their robust performance. We begin by introducing the necessary definitions of HAML attributes: the drift functional. 

\begin{restatable}{definition}{definedrift}
\label{def:mirror}
Let $i\in\mathcal{N}$, a \textbf{heterogeneous-agent drift functional} (HADF) $\mathfrak{D}^{i}$ of $i$ consists of a map, which is defined as 
\begin{align}
    \mathfrak{D}^{i}: \boldsymbol{\Pi}\times\color{RoyalPurple}\boldsymbol{\Pi}\color{black}\times\color{WildStrawberry}\mathbb{P}(-i)\color{black}\times\mathcal{S} \rightarrow \{ \mathfrak{D}^{i}_{\vpi}(\cdot|s, \color{RoyalPurple}\vbarpi\color{WildStrawberry}^{j_{1:m}}\color{black}): \mathcal{P}(\mathcal{A}^i) \rightarrow  \mathbb{R} \}, \nonumber
\end{align}
such that for all arguments, 
under notation $\mathfrak{D}^{i}_{\boldsymbol{\pi}}\big(\hat{\pi}^i|s, \vbarpi^{j_{1:m}}\big) \triangleq \mathfrak{D}^{i}_{\boldsymbol{\pi}}\big(\hat{\pi}^i(\cdot^i|s) | s, \vbarpi^{j_{1:m}}(\cdot|s)\big)$,
\begin{enumerate}
    \item $\mathfrak{D}^{i}_{\boldsymbol{\pi}}\big(\hat{\pi}^i|s, \vbarpi^{j_{1:m}}\big) \geq \mathfrak{D}^{i}_{\boldsymbol{\pi}}\big(\pi^i|s, \vbarpi^{j_{1:m}}\big) =0$ (non-negativity), 
    \item $\mathfrak{D}^{i}_{\boldsymbol{\pi}}\big(\hat{\pi}^i|s, \vbarpi^{j_{1:m}}\big)$ has all G\^ateaux derivatives zero at $\hat{\pi}^{i}=\pi^i$ (zero gradient).
\end{enumerate}
We say that the HADF is positive if $\mathfrak{D}^{i}_{\vpi}(\hat{\pi}^i|s, \vbarpi^{j_{1:m}})=0, \forall s \in \mathcal{S}$ implies $\hat{\pi}^i=\pi^i$, and trivial if $\mathfrak{D}^{i}_{\vpi}(\hat{\pi}^i|s, \vbarpi^{j_{1:m}})=0, \forall s \in \mathcal{S}$ for all $\vpi, \vbarpi^{j_{1:m}}$, and $\hat{\pi}^i$.
\end{restatable}

Intuitively, the drift $\mathfrak{D}_{\vpi}^{i}(\hat{\pi}^i|s, \vbarpi^{j_{1:m}})$ is a notion of distance between $\pi^i$ and $\hat{\pi}^i$, given that agents $j_{1:m}$ just updated to $\vbarpi^{j_{1:m}}$. We highlight that, under this conditionality, the same update (from $\pi^i$ to $\hat{\pi}^i$) can have different sizes---this will later enable HAML agents to \textsl{softly} constraint their learning steps in a coordinated way. Before that, we introduce a notion that renders \textsl{hard} constraints, which may be a part of an algorithm design, or an inherent limitation.

\begin{restatable}{definition}{neighbourhood}
Let $i\in\mathcal{N}$.
We say that, $\mathcal{U}^i:\boldsymbol{\Pi}\times\Pi^i\rightarrow\mathbb{P}(\Pi^i)$ is a \textsl{neighbourhood operator} if $\forall \pi^i \in \Pi^i$, $\mathcal{U}_{\vpi}^i(\pi^i)$ contains a closed ball, \emph{i.e.}, there exists a state-wise monotonically non-decreasing metric $\chi:\Pi^i\times \Pi^i \rightarrow \mathbb{R}$ such that $\forall \pi^i\in\Pi^i$ there exists $\delta^i > 0$ such that $\chi(\pi^i, \bar{\pi}^i) \leq \delta^i \implies \bar{\pi}^i \in \mathcal{U}^i_{\vpi}(\pi^i)$. 
\end{restatable}

For every joint policy $\vpi$, we will associate it with its \textsl{sampling distribution}---a positive state distribution $\beta_{\vpi}\in\mathcal{P}(\mathcal{S})$ that is continuous in $\vpi$ \citep{kuba2022mirror}. With these notions defined, we introduce the main definition for HAML framework.

\begin{restatable}{definition}{mirror}
Let $i\in\mathcal{N}$, $j^{1:m}\in\mathbb{P}(-i)$, and $\mathfrak{D}^{i}$ be a HADF of agent $i$. The \textbf{heterogeneous-agent mirror operator} (HAMO) integrates the advantage function as
\begin{align}
    \big[ \mathcal{M}^{(\hat{\pi}^i)}_{\mathfrak{D}^{i}, \vbarpi^{j_{1:m}}} A_{\vpi}\big](s) \triangleq \E_{\rva^{j_{1:m}}\sim\vbarpi^{j_{1:m}}, \ra^i\sim\hat{\pi}^i}\Big[ A^i_{\vpi}(s, \rva^{j_{1:m}}, \ra^i) \Big] - \mathfrak{D}^i_{\vpi}\Big(\hat{\pi}^i \big| s, \vbarpi^{j_{1:m}}\Big).\nonumber
\end{align}
\end{restatable}
Note that when $\hat{\pi}^i=\pi^i$, HAMO evaluates to zero. Therefore, as the HADF is non-negative, a policy $\hat{\pi}^i$ that improves HAMO must make it positive and thus leads to the improvement of the multi-agent advantage of agent $i$. It turns out that, under certain configurations, agents' local improvements result in the joint improvement of all agents, as described by the lemma below, proved in Appendix \ref{apx:proof_hamo}.

\begin{restatable}[HAMO Is All You Need]{lemma}{dpi}
\label{lemma:hamo}
Let $\boldsymbol{\pi}_{\text{old}}$ and $\vpi_{\text{new}}$ be joint policies and let $i_{1:n}\in\text{\normalfont Sym}(n)$ be an agent permutation. Suppose that, for every state $s\in\mathcal{S}$ and every $m=1,\ldots,n$,
\begin{align}
    \label{ineq:what-mirror-step-does}
    \big[\mathcal{M}^{(\pi^{i_m}_{\text{new}})}_{\mathfrak{D}^{i_{m}}, \vpi_{\text{new}}^{i_{1:m-1}}} A_{\boldsymbol{\vpi_{\text{old}}}}\big](s) \geq 
    \big[\mathcal{M}^{(\pi^{i_m}_{\text{old}})}_{\mathfrak{D}^{i_{m}}, \vpi_{\text{new}}^{i_{1:m-1}}} A_{\boldsymbol{\vpi_{\text{old}}}}\big](s).
\end{align}
Then, $\boldsymbol{\pi}_{\text{new}}$ is jointly better than $\boldsymbol{\pi}_{\text{old}}$, so that for every state $s$,
\begin{align}
    V_{\boldsymbol{\pi}_{\text{new}}}(s) \geq V_{\boldsymbol{\pi}_{\text{old}}}(s). \nonumber
\end{align}
\end{restatable}
Subsequently, the \textsl{monotonic improvement property} of the joint return follows naturally, as
\begin{align}
    J(\boldsymbol{\pi}_{\text{new}}) = \E_{\rs\sim d}\big[V_{\boldsymbol{\pi}_{\text{new}}}(s)\big] 
    \geq 
    \E_{\rs\sim d}\big[V_{\boldsymbol{\pi}_{\text{old}}}(s)\big] 
    = J(\boldsymbol{\pi}_{\text{old}}). \nonumber
\end{align}

However, the conditions of the lemma require every agent to solve $|\mathcal{S}|$ instances of Inequality (\ref{ineq:what-mirror-step-does}), which may be an intractable problem. We shall design a single optimisation objective whose solution satisfies those inequalities instead. Furthermore, to have a practical application to large-scale problems, such an objective should be estimatable via sampling. To handle these challenges, we introduce the following Algorithm Template \ref{algorithm:haml} which generates a continuum of HAML algorithms.

\begin{metaalgorithm}
\SetKwInOut{Input}{Input}
\SetKwInOut{Output}{Output}
\caption{Heterogeneous-Agent Mirror Learning}
\label{algorithm:haml}
Initialise a joint policy $\boldsymbol{\pi}_{0} = (\pi^{1}_{0}, \dots, \pi^{n}_{0})$\;
\For{$k=0, 1, \dots$}{
    Compute the advantage function $A_{\boldsymbol{\pi}_{k}}(s, \va)$ for all state-(joint)action pairs $(s, \va)$\;
    Draw a permutaion $i_{1:n}$ of agents at random \color{RoyalPurple} //from a positive distribution $p\in\mathcal{P}(\text{Sym}(n))$\color{black}\;
    \For{$m=1:n$}{
        Make an update $\pi^{i_{m}}_{k+1} = \argmax\limits_{\pi^{i_{m}}\in\mathcal{U}^{i_m}_{\vpi_k}(\pi^{i_m}_k)}\E_{\rs\sim\beta_{\vpi_k}}\Big[
        \big[\mathcal{M}^{(\pi^{i_m})}_{\mathfrak{D}^{i_{m}}, \vpi_{k+1}^{i_{1:m-1}}} A_{\boldsymbol{\vpi_{k}}}\big](s)
        \Big]$\;
    }
    }
\Output{A limit-point joint policy $\vpi_{\infty}$}
\end{metaalgorithm}

Based on Lemma \ref{lemma:hamo} and the fact that $\pi^i\in\mathcal{U}_{\vpi}^i(\pi^i), \forall i\in\mathcal{N}, \pi^i\in\Pi^i$, we can know any HAML algorithm (weakly) improves the joint return at every iteration. In practical settings, such as deep MARL, the maximisation step of a HAML method can be performed by a few steps of gradient ascent on a sample average of HAMO (see Definition \ref{def:mirror}). 
We also highlight that if the neighbourhood operators $\mathcal{U}^i$ can be chosen so that they produce small policy-space subsets, then the resulting updates will be not only improving but also small. This, again, is a desirable property while optimising neural-network policies, as it helps stabilise the algorithm. Similar to HATRL, the order of agents in HAML updates is randomised at every iteration; this condition has been necessary to establish convergence to NE, which is intuitively comprehensible: fixed-point joint policies of this randomised procedure assure that none of the agents is incentivised to make an update, namely reaching a NE. We provide the full list of the most fundamental HAML properties in Theorem \ref{theorem:fundamental} which shows that any method derived from Algorithm  Template \ref{algorithm:haml} solves the cooperative MARL problem.
\begin{restatable}[The Fundamental Theorem of Heterogeneous-Agent Mirror Learning]{theorem}{mamlfundamental}
\label{theorem:fundamental}
Let, for every agent $i\in\mathcal{N}$, $\mathfrak{D}^{i}$ be a HADF, $\mathcal{U}^i$ be a neighbourhood operator, and let the sampling distributions $\beta_{\vpi}$ depend continuously on $\vpi$. Let $\boldsymbol{\pi}_0 \in \boldsymbol{\Pi}$, and the sequence of joint policies $(\boldsymbol{\pi}_k)_{k=0}^{\infty}$ be obtained by a HAML algorithm induced by $\mathfrak{D}^{i}, \mathcal{U}^{i}, \forall i\in\mathcal{N}$, and $\beta_{\vpi}$. Then, the joint policies induced by the algorithm enjoy the following list of properties 
\begin{enumerate}
    \item \label{property1} Attain the monotonic improvement property, 
    \begin{align}
        J(\boldsymbol{\pi}_{k+1}) \ \geq \ J(\boldsymbol{\pi}_k),\nonumber
    \end{align}
    \item \label{property2} Their value functions converge to a Nash value function $V^{\text{NE}}$
    \begin{align}
        \lim_{k\rightarrow\infty}V_{\boldsymbol{\pi}_k} = V^{\text{NE}},\nonumber
    \end{align}
    \item \label{property3} Their expected returns converge to a Nash return,
    \begin{align}
        \lim_{k\rightarrow\infty}J(\boldsymbol{\pi}_k) = J^{\text{NE}},\nonumber
    \end{align}
    \item \label{property4} Their $\omega$-limit set consists of Nash equilibria.
\end{enumerate}

\end{restatable}
See the proof in Appendix \ref{apx:proof_t1}. With the above theorem, we can conclude that HAML provides a template for generating theoretically sound, stable, monotonically improving algorithms that enable agents to learn solving multi-agent cooperation tasks. 

\subsection{Casting HATRPO and HAPPO as HAML Instances}
\label{subsec:hatrpo-happo-haml-instances}
In this section, we show that HATRPO and HAPPO are in fact valid instances of HAML, which provides a more direct theoretical explanation for their excellent empirical performance.

We begin with the example of HATRPO, where agent $i_m$ (the permutation $i_{1:n}$ is drawn from the uniform distribution) updates its policy so as to maximise (in $\bar{\pi}^{i_m}$)
\begin{align}
    \E_{\rs\sim\rho_{\boldsymbol{\pi}_{\text{old}}},
    \rva^{i_{1:m-1}}\sim \vpi^{i_{1:m-1}}_{\text{new}},
    \ra^{i_m}\sim\bar{\pi}^{i_m}}\Big[ A^{i_m}_{\boldsymbol{\pi}_{\text{old}}}(\rs, \rva^{i_{1:m-1}}, \ra^{i_m}) \Big], \ \ \ \text{subject to} \ \overline{D}_{\text{KL}}(\pi^{i_m}_{\text{old}}, \bar{\pi}^{i_m})\leq \delta.\nonumber
\end{align}

This optimisation objective can be casted as a HAMO with the HADF $\mathfrak{D}^{i_m} \equiv 0$, and the KL-divergence neighbourhood operator
\begin{align}
\mathcal{U}^{i_m}_{\vpi}(\pi^{i_m}) = \Big\{ \bar{\pi}^{i_m} \ \Big| \ \E_{\rs\sim\rho_{\vpi}}\Big[\text{KL}\big(\pi^{i_m}(\cdot^{i_m}|\rs), \bar{\pi}^{i_m}(\cdot^{i_m}|\rs)\big)\Big] \leq \delta \Big\}. \nonumber
\end{align}

The sampling distribution used in HATRPO is $\beta_\vpi=\rho_\vpi$. Lastly, as the agents update their policies in a random loop, the algorithm is an instance of HAML. Hence, it is monotonically improving and converges to a Nash equilibrium set. 

In HAPPO, the update rule of agent $i_m$ is changed with respect to HATRPO as
\begin{align}
    &\E_{\rs\sim\rho_{\boldsymbol{\pi}_{\text{old}}},
    \rva^{i_{1:m-1}}\sim \vpi^{i_{1:m-1}}_{\text{new}},
    \ra^{i_m}\sim\pi^{i_m}_{\text{old}}}\Big[ 
    \min\big( \rr(\bar{\pi}^{i_m})A^{i_{1:m}}_{\vpi_{\text{old}}}(\rs, \rva^{i_{1:m}}), \text{clip}\big( \rr(\bar{\pi}^{i_m}), 1\pm \epsilon\big)A^{i_{1:m}}_{\vpi_{\text{old}}}(\rs, \rva^{i_{1:m}})\big) \Big],\nonumber
\end{align}
where $\rr(\bar{\pi}^{i}) = \frac{\bar{\pi}^i(\ra^i|\rs)}{\pi^{i}_{\text{old}}(\ra^i|\rs)}$. We show in Appendix \ref{appendix:happo} that this optimisation objective is equivalent to
\begin{align}
    &\E_{\rs\sim\rho_{\boldsymbol{\pi}_{\text{old}}}}\Big[ 
    \E_{\rva^{i_{1:m-1}}\sim \vpi^{i_{1:m-1}}_{\text{new}},
    \ra^{i_m}\sim\bar{\pi}^{i_m}}\big[ A^{i_m}_{\boldsymbol{\pi}_{\text{old}}}(\rs, \rva^{i_{1:m-1}}, \ra^{i_m}) \big] \nonumber\\
    &\quad \quad \quad \quad  -\color{RoyalPurple}\E_{\rva^{i_{1:m-1}}\sim \vpi^{i_{1:m-1}}_{\text{new}},
    \ra^{i_m}\sim\pi^{i_m}_{\text{old}}}\big[
    \text{ReLU}\big( \big[ \rr(\bar{\pi}^{i_m}) -\text{clip}\big( \rr(\bar{\pi}^{i_m}), 1\pm \epsilon\big)\big]A^{i_{1:m}}_{\boldsymbol{\pi}_{\text{old}}}(s, \rva^{i_{1:m}})\big) \big]\color{black}\Big]. \nonumber
\end{align}

The \color{RoyalPurple}purple \color{black} term is clearly non-negative due to the presence of the ReLU function. Furthermore, for policies $\bar{\pi}^{i_m}$ sufficiently close to $\pi^{i_m}_{\text{old}}$, the clip operator does not activate, thus rendering $\rr(\bar{\pi}^{i_m})$ unchanged. Therefore, the \color{RoyalPurple} purple \color{black} term is zero at and in a region around $\bar{\pi}^{i_m}=\pi^{i_m}_{\text{old}}$, which also implies that its G{\^a}teaux derivatives are zero. Hence, it evaluates a HADF for agent $i_m$, thus making HAPPO a valid HAML instance.

Finally, we would like to highlight that these conclusions about HATRPO and HAPPO strengthen the results in Section \ref{subsec:HATRL} and \ref{subsec:pracalgo}. In addition to their origin in HATRL, we now show that their optimisation objectives directly enjoy favorable theoretical properties endowed by HAML framework. Both interpretations underpin their empirical performance.

\subsection{More HAML Instances}
\label{subsec:new-haml}
In this subsection, we exemplify how HAML can be used for derivation of principled MARL algorithms, solely by constructing valid drift functional, neighborhood operator, and sampling distribution. Our goal is to verify the correctness of HAML theory and enrich the cooperative MARL with more theoretically guaranteed and practical algorithms. The results are more robust heterogeneous-agent versions of popular RL algorithms including A2C, DDPG, and TD3, different from those in Section \ref{sec:sota}. 

\begin{figure}[tbp]
  \centering
  \includegraphics[width=\linewidth]{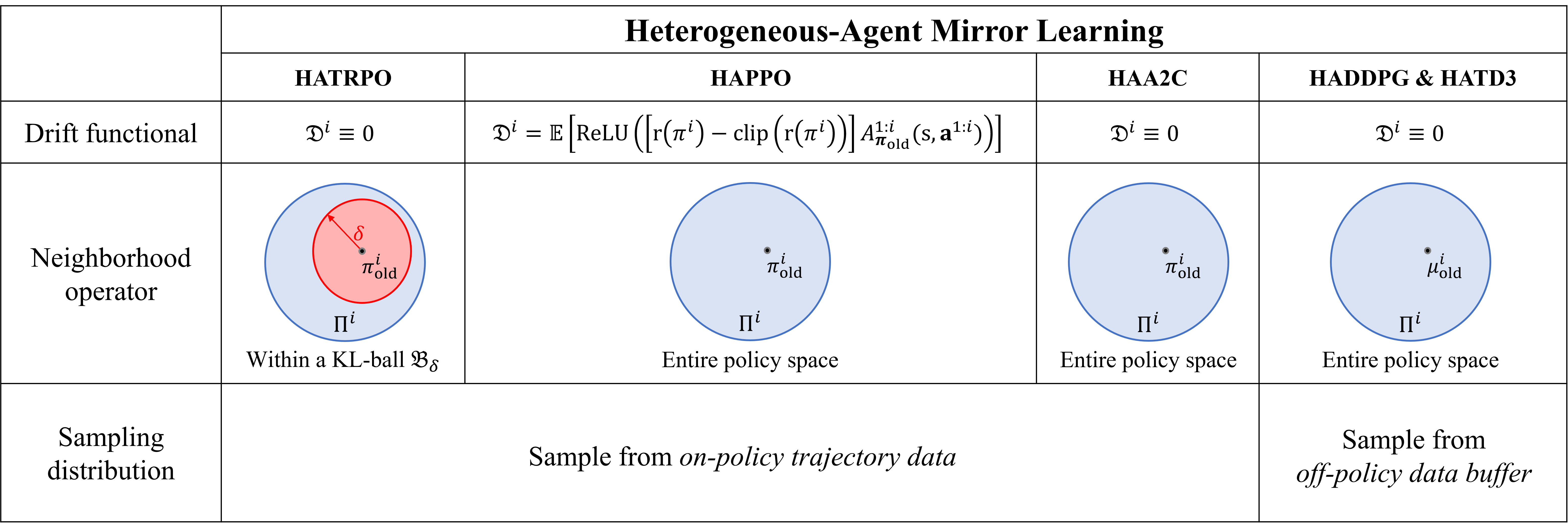}
    \caption{This figure presents a simplified schematic overview of HARL algorithms represented as valid instances of HAML. The complete details are available in Appendix \ref{appendix:haml-summary}. By recasting HATRPO and HAPPO as HAML formulations, we demonstrate that their guarantees pertaining to monotonic improvement and NE convergence are enhanced by leveraging the HAML framework. Moreover, HAA2C, HADDPG, and HATD3 are obtained by designing HAML components, thereby securing those same performance guarantees. The variety of drift functionals, neighborhood operators, and sampling distributions utilised by these approaches further attests to the versatility and richness of the HAML framework.
    }
  
  \label{fig:theory-algo-summary}
\end{figure}

\subsubsection{HAA2C}

HAA2C intends to optimise the policy for the joint advantage function at every iteration, and similar to A2C, does not impose any penalties or constraints on that procedure. This learning procedure is accomplished  by, first, drawing a random permutation of agents $i_{1:n}$, and then performing a few steps of gradient ascent on the objective of 
\begin{align}
    \label{eq:haa2c}
    \E_{\rs\sim\rho_{\vpi_{\text{old}}}, \rva^{i_{1:m}}\sim\vpi^{i_{1:m}}_{\text{old}}}\Big[ 
    \frac{\vpi^{i_{1:m-1}}_{\text{new}}(\rva^{i_{1:m-1}}|\rs)\pi^{i_m}(\ra^{i_m}|\rs) }{\vpi^{i_{1:m-1}}_{\text{old}}(\rva^{i_{1:m-1}}|\rs)\pi^{i_m}_{\text{old}}(\ra^{i_m}|\rs) }A_{\vpi_{\text{old}}}^{i_m}(\rs, \rva^{i_{1:m-1}}, \ra^{i_m})
    \Big], 
\end{align}
with respect to $\pi^{i_m}$ parameters,
for each agent $i_m$ in the permutation, sequentially. In practice, we replace the multi-agent advantage {\small $A_{\vpi_{\text{old}}}^{i_m}(\rs, \rva^{i_{1:m-1}}, \ra^{i_m})$} with the joint advantage estimate which, thanks to the joint importance sampling in Equation (\ref{eq:haa2c}), poses the same objective on the agent (see Appendix \ref{appendix:algos} for full pseudocode).

\subsubsection{HADDPG}

HADDPG exploits the fact that $\beta_{\boldsymbol{\pi}}$ can be independent of $\boldsymbol{\pi}$ and aims to maximise the state-action value function \emph{off-policy}. As it is a deterministic-action method, importance sampling in its case translates to replacement of the old action inputs to the critic with the new ones. Namely, agent $i_m$ in a random permutation $i_{1:n}$ maximises
\begin{align}
    \label{eq:haddpg}
    \E_{\rs\sim\beta_{\vmu_{\text{old}}}}\Big[ Q^{i_{1:m}}_{\vmu_{\text{old}}}\big(\rs, \vmu_{\text{new}}^{i_{1:m-1}}(\rs), \mu^{i_m}(\rs)
    \big)\Big],
\end{align}
with respect to $\mu^{i_m}$, also with a few steps of gradient ascent. Similar to HAA2C, optimising the state-action value function (with the old action replacement) is equivalent to the original multi-agent value (see Appendix \ref{appendix:algos} for full pseudocode). 

\subsubsection{HATD3}

HATD3 improves HADDPG with tricks proposed by \cite{fujimoto2018addressing}. Similar to HADDPG, HATD3 is also an off-policy algorithm and optimises the same target, but it employs target policy smoothing, clipped double Q-learning, and delayed policy updates techniques (see Appendix \ref{appendix:algos} for full pseudocode). We observe that HATD3 consistently outperforms HADDPG on all tasks, showing that relevant reinforcement learning can be directly applied to MARL without the need for rediscovery, another benefit of the HAML.

As the HADDPG and HATD3 algorithms have been derived, it is logical to consider the possibility of HADQN, given that DQN can be viewed as a pure value-based version of DDPG for discrete action problems. In light of this, we introduce HAD3QN, a value-based approximation of HADDPG that incorporates techniques proposed by \cite{van2016deep} and \cite{wang2016dueling}. The details of HAD3QN are presented in Appendix \ref{appendix:had3qn}, which includes the pseudocode, performance analysis, and an ablation study demonstrating the importance of the dueling double Q-network architecture for achieving stable and efficient multi-agent learning.

To elucidate the formulations and differences of HARL approaches in their HAML representation, we provide a simplified summary in Figure \ref{fig:theory-algo-summary} and list the full details in Appendix \ref{appendix:haml-summary}. While these approaches have already tailored HADFs, neighbourhood operators, and sampling distributions, we speculate that the entire abundance of the HAML framework can still be explored with more future work. Nevertheless, we commence addressing the heterogeneous-agent cooperation problem with these five methods, and analyse their performance in Section \ref{sec:experiments}.

\section{Related Work}

There have been previous attempts that tried to solve the cooperative MARL problem by developing multi-agent trust region learning theories.
Despite empirical successes, most of them did not manage to propose a theoretically-justified  trust region protocol in multi-agent learning, or maintain the monotonic improvement property.  Instead, they tend to impose certain assumptions to enable direct implementations of TRPO/PPO in MARL problems. For example,  IPPO \citep{de2020independent} assumes homogeneity of action spaces for all agents and enforces parameter sharing.  
 \cite{mappo} proposed MAPPO which enhances IPPO by considering a joint critic function and  finer implementation techniques for on-policy methods.  Yet, it still suffers similar drawbacks of IPPO due to the lack of monotonic improvement guarantee especially when the parameter-sharing condition is switched off. 
\cite{wen2022game} adjusted PPO for MARL by considering a game-theoretical approach at the meta-game level among agents. Unfortunately, it can only deal with two-agent cases due to the intractability of Nash equilibrium. 
Recently, \cite{li2023multiagent} tried to implement TRPO for MARL through distributed consensus optimisation;  however, they enforced the  same ratio  ${\bar{\pi}^i(a^i|s)}/{\pi^i(a^i|s)}$ for all agents (see their Equation (7)), which, similar to parameter sharing, largely limits the policy space for optimisation. Moreover, their method comes with a $\delta/n$ KL-constraint threshold that fails to consider scenarios with large agent number. While Coordinated PPO (CoPPO) \citep{wu2021coordinated} derived a theoretically-grounded joint objective and obtained practical algorithms through a set of approximations, it still suffers from the non-stationarity problem as it updates agents simultaneously.


One of the key ideas behind our Heterogeneous-Agent algorithm series is the sequential update scheme. A similar idea of multi-agent sequential update was also discussed in the context of dynamic programming  \citep{bertsekas2019multiagent} where artificial  "in-between" states have to be considered. On the contrary,  our sequential update scheme is developed based on Lemma \ref{lemma:maadlemma}, which does not  require  any artificial assumptions and holds for any cooperative games. The idea of sequential update also appeared in principal component analysis; in EigenGame \citep{gemp2021eigengame} eigenvectors, represented as players, maximise their own utility functions one-by-one. Although EigenGame provably solves the PCA problem, it is of little use in MARL, where a single iteration of sequential updates is insufficient to learn complex policies. Furthermore, its design and analysis rely on closed-form matrix calculus, which has no extension to MARL.

Lastly, we would like to highlight the importance of the decomposition result in Lemma \ref{lemma:maadlemma}. This result could serve as an effective solution to value-based methods in MARL where tremendous efforts have been made to decompose the joint Q-function into individual Q-functions when the joint Q-function is decomposable \citep{rashid2018qmix}.
Lemma \ref{lemma:maadlemma}, in contrast, is a general result that holds for any cooperative  MARL problems regardless of decomposability. As such, we think of it as an appealing contribution to future developments on value-based MARL methods.

Our work is an extension of previous work HATRPO / HAPPO, which was originally proposed in a conference paper \citep{kuba2022trust}. The main additions in our work are:

\begin{itemize}
    \item Introducing Heterogeneous-Agent Mirror Learning (HAML), a more general theoretical framework that strengthens theoretical guarantees for HATRPO and HAPPO and can induce a continuum of sound algorithms with guarantees of monotonic improvement and convergence to Nash Equilibrium;

    \item Designing novel algorithm instances of HAML including HAA2C, HADDPG, and HATD3, which attain better performance than their existing MA-counterparts, with HATD3 establishing the new SOTA results for off-policy algorithms;
    
    \item Releasing PyTorch-based implementation of HARL algorithms, which is more unified, modularised, user-friendly, extensible, and effective than the previous one;

    \item Conducting comprehensive experiments evaluating HARL algorithms on six challenging benchmarks Multi-Agent Particle Environment (MPE), Multi-Agent MuJoCo (MAMuJoCo), StarCraft Multi-Agent Challenge (SMAC), SMACv2, Google Research Football Environment (GRF), and Bi-DexterousHands.

\end{itemize}

\section{Experiments and Analysis}
\label{sec:experiments}

In this section, we evaluate and analyse HARL algorithms on six cooperative multi-agent benchmarks --- Multi-Agent Particle Environment (MPE) \citep{maddpg, mordatch2018emergence}, Multi-Agent MuJoCo (MAMuJoCo) \citep{peng2021facmac}, StarCraft Multi-Agent Challenge (SMAC) \citep{samvelyanstarcraft}, SMACv2 \citep{ellis2022smacv2}, Google Research Football Environment (GRF) \citep{kurach2020google}, and Bi-DexterousHands \citep{chen2022towards}, as shown in Figure \ref{fig:benchmark-overview} --- and compare their performance to existing SOTA methods. These benchmarks are diverse in task difficulty, agent number, action type, dimensionality of observation space and action space, and cooperation strategy required, and hence provide a comprehensive assessment of the effectiveness, stability, robustness, and generality of our methods. 
The experimental results demonstrate that HAPPO, HADDPG, and HATD3 generally outperform their MA-counterparts on heterogeneous-agent cooperation tasks. Moreover, HARL algorithms culminate in HAPPO and HATD3, which exhibit superior effectiveness and stability for heterogeneous-agent cooperation tasks over existing strong baselines such as MAPPO, QMIX, MADDPG, and MATD3, refreshing the state-of-the-art results. Our ablation study also reveals that the novel details introduced by HATRL and HAML theories, namely non-sharing of parameters and randomised order in sequential update, are crucial for obtaining the strong performance. Finally, we empirically show that the computational overhead introduced by sequential update does not need to be a concern.

Our implementation of HARL algorithms takes advantage of the sequential update scheme and the CTDE framework that HARL algorithms share in common, and unifies them into either the on-policy or the off-policy training pipeline, resulting in modularisation and extensibility. It also naturally hosts MAPPO, MADDPG, and MATD3 as special cases and provides the (re)implementation of these three algorithms along with HARL algorithms. For fair comparisons, we use our (re)implementation of MAPPO, MADDPG, and MATD3 as baselines on MPE and MAMuJoCo, where their publicly acknowledged performance report under exactly the same settings is lacking, and we ensure that their performance matches or exceeds the results reported by their original paper and subsequent papers; on the other benchmarks, the original implementations of baselines are used. To be consistent with the officially reported results of MAPPO, we let it utilize parameter sharing on all but Bi-DexterousHands and the Speaker Listener task in MPE. Details of hyper-parameters and experiment setups can be found in Appendix \ref{appendix:exp}.

\begin{figure}[tbp]
  \centering
  \includegraphics[width=\linewidth]{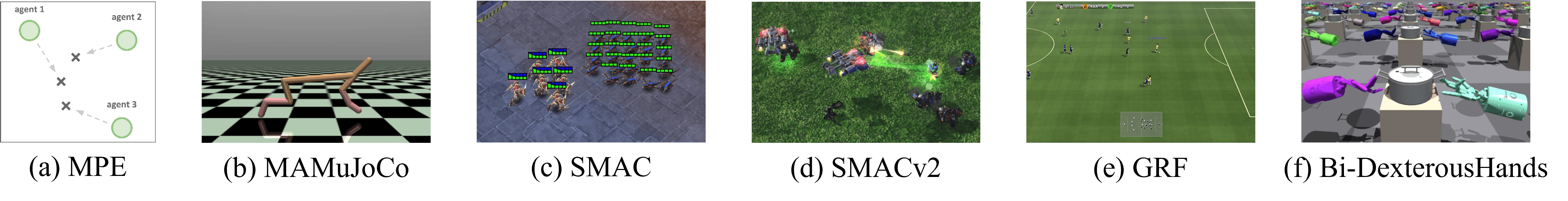}
    \caption{The six environments used for testing HARL algorithms.}
  
  \label{fig:benchmark-overview}
\end{figure}

\subsection{MPE Testbed}

We consider the three fully cooperative tasks in MPE \citep{maddpg}: Spread, Reference, and Speaker Listener. These tasks require agents to explore and then learn the optimal cooperation strategies, such as spreading to targets as quickly as possible without collision, instructing companions, and so on. The Speaker Listener scenario, in particular, explicitly designs different roles and fails the homogeneous agent approach. As the original codebase of MPE is no longer maintained, we choose to use its PettingZoo version \citep{terry2021pettingzoo}. To make it compatible with the cooperative MARL problem formulation in Section \ref{sec:prelim}, we implement the interface of MPE so that agents do not have access to their individual rewards, as opposed to the setting used by MADDPG. Instead, individual rewards of agents are summed up to form the joint reward, which is available during centralised training. We evaluate HAPPO, HATRPO, HAA2C, HADDPG, and HATD3 on the continuous action-space version of these three tasks against MAPPO, MADDPG, and MATD3, with on-policy algorithms running for 10 million steps and off-policy ones for 5 million steps. Since the stochastic policy algorithms, namely HAPPO, HATRPO, and HAA2C, can also be applied to discrete action-space scenarios, we additionally compare them with MAPPO on the discrete version of these three tasks, using the same number of timesteps. The learning curves plotted from training data across three random seeds are shown in Figure \ref{fig:mpe}.

\begin{figure}[tbp]
  \centering
  \includegraphics[width=\linewidth]{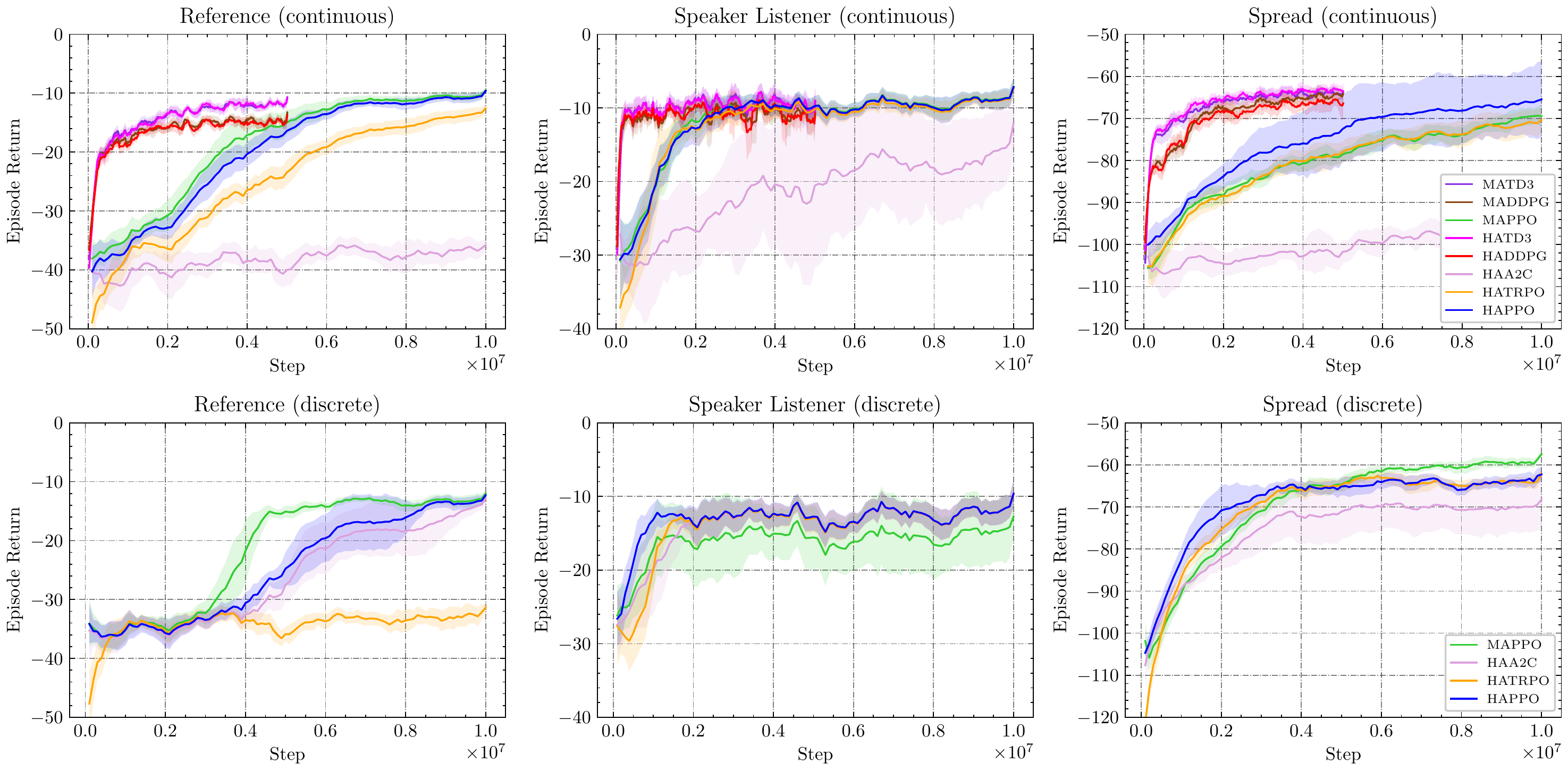}
  \caption{Comparisons of average episode return on Multi-Agent Particle Environments. The ``continuous'' and ``discrete'' in parenthesis refer to the type of action space in each task.}
  \label{fig:mpe}
\end{figure}

While MPE tasks are relatively simple, it is sufficient for identifying several patterns. HAPPO consistently solves all six combinations of tasks, with its performance comparable to or better than MAPPO. With a single set of hyper-parameters, HATRPO also solves five combinations easily and achieves steady learning curves due to the explicitly specified distance constraint and reward improvement between policy updates. It should be noted that the oscillations observed after convergence are due to the randomness of test environments which affects the maximum reward an algorithm can attain. HAA2C, on the other hand, is equally competitive on the discrete version of tasks, but shows higher variance and is empirically harder to achieve the same level of episode return on the continuous versions, which is a limitation of this method since its update rule can not be precisely realised in practice and meanwhile it imposes no constraint. Nevertheless, it still constitutes a potentially competitive solution. 

Furthermore, two off-policy HARL methods, HADDPG and HATD3, exhibit extremely fast mastery of the three tasks with small variance, demonstrating their advantage in high sample efficiency. Their performances are similar to MA-counterparts on these simple tasks, with TD3-based methods achieving faster convergence rate and higher total rewards, establishing new SOTA off-policy results. Off-policy HARL methods consistently converge with much fewer samples than on-policy methods across all tasks, holding the potential to alleviate the high sample complexity and slow training speed problems, which are commonly observed in MARL experiments.

These observations show that while HARL algorithms have the same improvement and convergence guarantees in theory, they differ in learning behaviours due to diverse algorithmic designs. In general, they complement each other and collectively solve all tasks.

\subsection{MAMuJoCo Testbed}

The Multi-Agent MuJoCo (MAMuJoCo) environment is a multi-agent extension of MuJoCo. While the MuJoCo tasks challenge a robot to learn an optimal way of motion, MAMuJoCo models each part of a robot as an independent agent --- for example,  a leg for a spider or an arm for a swimmer --- and requires the agents to collectively perform efficient motion. With the increasing variety of the body parts, modeling heterogeneous policies becomes \textbf{necessary}. Thus, we believe that MAMuJoCo is a suitable task suite for evaluating the effectiveness of our heterogeneous-agent methods. We evaluate HAPPO, HATRPO, HAA2C, HADDPG, and HATD3 on the five most representative tasks against MAPPO, MADDPG, and MATD3 and plot the learning curves across at least three seeds in Figure \ref{fig:mamujoco-on-policy} and \ref{fig:mamujoco-off-policy}.

\begin{figure}[tbp]
  \centering
  \includegraphics[width=\linewidth]{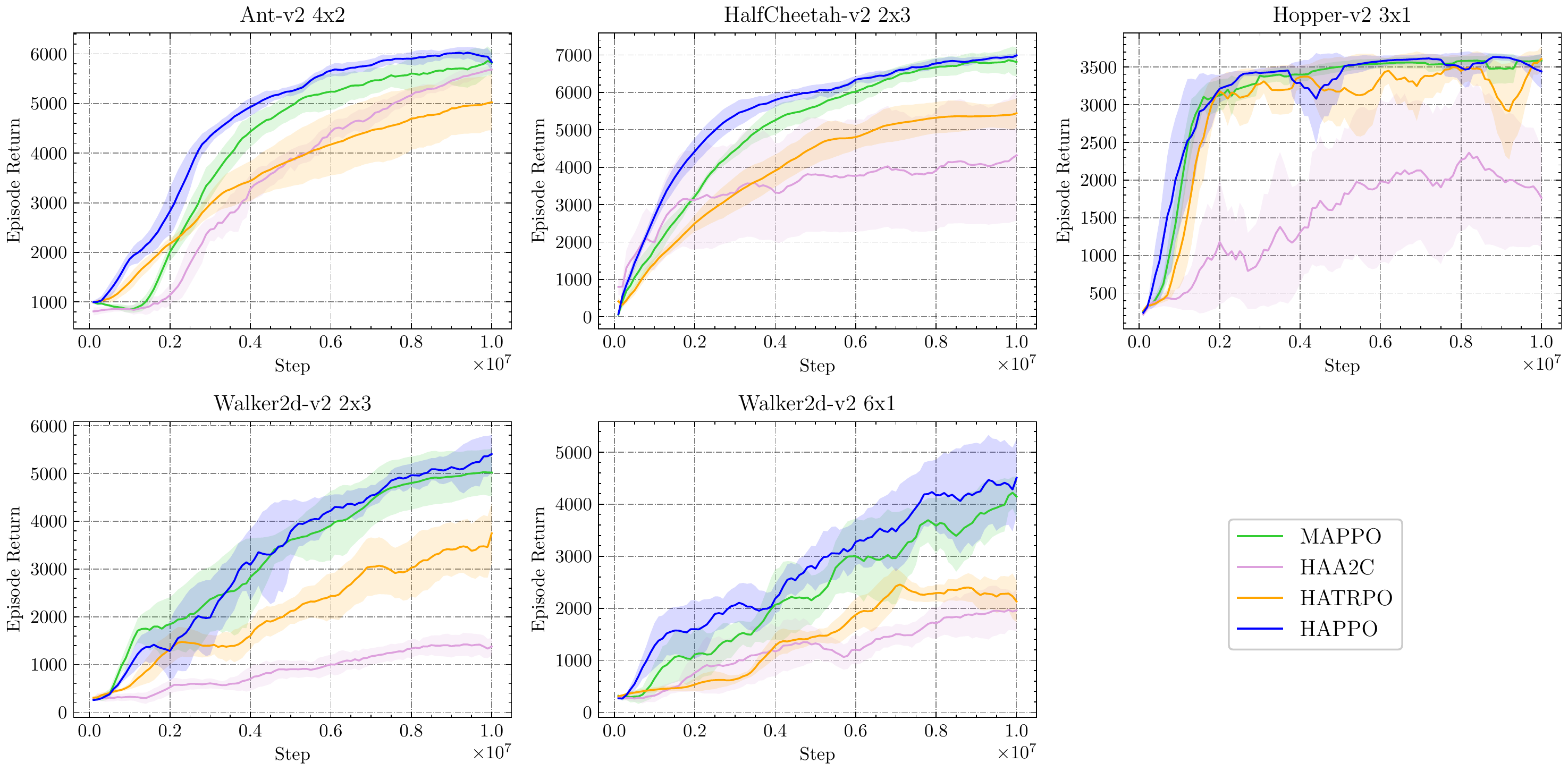}
  \caption{Comparisons of average episode return of on-policy algorithms on Multi-Agent MuJoCo. HAPPO generally outperforms MAPPO, refreshing the state-of-the-art (SOTA) results for on-policy algorithms.}
  \label{fig:mamujoco-on-policy}
\end{figure}

\begin{figure}[tbp]
  \centering
  \includegraphics[width=\linewidth]{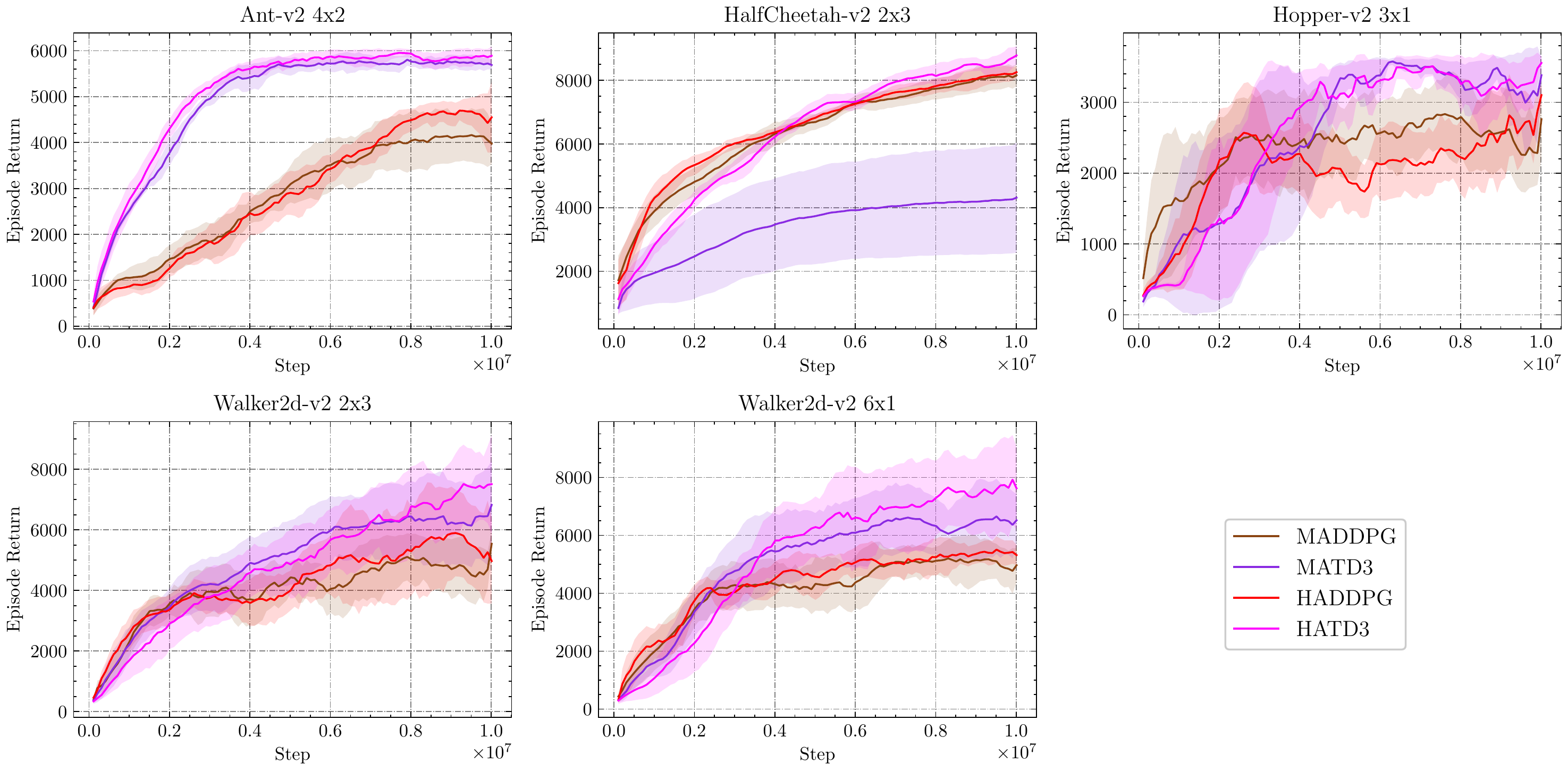}
  \caption{Comparisons of average episode return of off-policy algorithms on Multi-Agent MuJoCo. HADDPG and HATD3 generally outperform MADDPG and MATD3, while HATD3 achieves the highest average return across all tasks, thereby refreshing the state-of-the-art (SOTA) results for off-policy algorithms.}
  \label{fig:mamujoco-off-policy}
\end{figure}

We observe that on all five tasks, HAPPO, HADDPG, and HATD3 achieves generally better average episode return than their MA-counterparts. HATRPO and HAA2C also achieve strong and steady learning behaviours on most tasks.
Since the running motion are hard to be realised by any subset of all agents, the episode return metric measures the quality of agents' cooperation. Rendered videos from the trained models of HARL algorithms confirm that agents develop effective cooperation strategies for controlling their corresponding body parts. For example, on the 2-agent HalfCheetah task, agents trained by HAPPO learn to alternately hit the ground, forming a swift kinematic gait that resembles a real cheetah. The motion performed by each agent is meaningless alone and only takes effect when combined with the other agent's actions. In other words, all agents play indispensable roles and have unique contributions in completing the task, which is the most desirable form of cooperation. Empirically, HARL algorithms prove their capability to generate this level of cooperation from random initialisation. 

As for the off-policy HARL algorithms, HATD3 outperforms both MATD3 and HADDPG on all tasks, due to the beneficial combination of sequential update and the stabilising effects brought by twin critics, delayed actor update, and target action smoothing tricks. This also admits the feasibility of introducing RL tricks to MARL. Its performance is even generally better than HAPPO, showing the competence to handle continuous tasks. Experimental results on MAMuJoCo not only prove the superiority of HARL algorithms over existing strong baselines, but also reveal that HARL renders multiple effective solutions to multi-agent cooperation tasks.

\begin{figure}[tbp]
  \centering
  \includegraphics[width=0.4\linewidth]{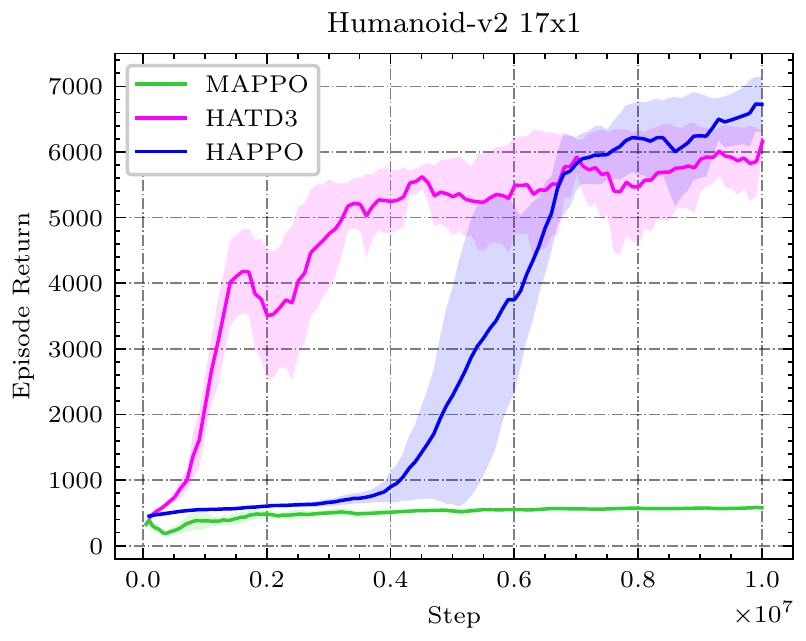}
  \caption{Comparisons of average episode return on the 17-agent Humanoid control task in Multi-Agent MuJoCo. In the face of this many-heterogeneous-agent task, HAPPO and HATD3 achieve state-of-the-art (SOTA) performance, while MAPPO fails completely. This highlights the superior effectiveness of HARL algorithms for promoting cooperation among heterogeneous agents.}
  \label{fig:humanoid}
\end{figure}

Though MAMuJoCo tasks are heterogeneous in nature, parameter sharing is still effective in scenarios where learning a     ``versatile'' policy to control all body parts by relying on the expressiveness of neural network is enough. As a result, on these five tasks, MAPPO underperforms HAPPO by not very large margins. To fully distinguish HAPPO from MAPPO, we additionally compare them on the 17-agent Humanoid task and report the learning curves averaged across three seeds in Figure \ref{fig:humanoid}. In this scenario, the 17 agents control dissimilar body parts and it is harder for a single policy to select the right action for each part. Indeed, MAPPO completely fails to learn. In contrast, HAPPO still manages to coordinate the agents' updates with its sequential update scheme which leads to a walking humanoid with the joint effort from all agents. With the same theoretical properties granted by HAML, HATD3 also successfully learns to control the 17-agent humanoid. Therefore, HARL algorithms are more applicable and effective for the general many-heterogeneous-agent cases. Their advantage becomes increasingly significant with the increasing heterogeneity of agents.

\subsection{SMAC \& SMACv2 Testbed}

The StarCraft Multi-Agent Challenge (SMAC) contains a set of StarCraft maps in which a team of mostly homogeneous ally units aims to defeat the opponent team. It challenges an algorithm to develop effective teamwork and decentralised unit micromanagement, and serves as a common arena for algorithm comparison. We benchmark HAPPO and HATRPO on five hard maps and five super hard maps in SMAC against QMIX \citep{rashid2018qmix} and MAPPO \citep{mappo}, which are known to achieve supreme results. Furthermore, as \cite{ellis2022smacv2} proposes SMACv2 to increase randomness of tasks and diversity among unit types in SMAC, we additionally test HAPPO and HATRPO on five maps in SMACv2 against QMIX and MAPPO. On these two sets of tasks, we adopt the implementations of QMIX and MAPPO that have achieved the best-reported results, \emph{i.e.} in SMAC we use the implementation by \cite{mappo} and in SMACv2 we use the implementation by \cite{ellis2022smacv2}.
Following the evaluation metric proposed by \cite{wang2020rode}, we report the win rates computed across at least three seeds in Table \ref{tab:smac} and provide the learning curves in Appendix \ref{appendix:smac-figures}.

\begin{table}[!t]
\centering
\begin{tabular}{ccccccc}
\hline
Map               & Difficulty & HAPPO                           & HATRPO                           & MAPPO                            & QMIX                             & Steps            \\ \hline
8m\_vs\_9m        & Hard       & 83.8\scriptsize{(4.1)}          & \textbf{92.5}\scriptsize{(3.7)}  & 87.5\scriptsize{(4.0)}           & \textbf{92.2}\scriptsize{(1.0)}  & $1 \mathrm{e} 7$ \\
25m               & Hard       & 95.0\scriptsize{(2.0)}          & \textbf{100.0}\scriptsize{(0.0)} & \textbf{100.0}\scriptsize{(0.0)} & 89.1\scriptsize{(3.8)}           & $1 \mathrm{e} 7$ \\
5m\_vs\_6m        & Hard       & \textbf{77.5}\scriptsize{(7.2)} & \textbf{75.0}\scriptsize{(6.5)}  & \textbf{75.0}\scriptsize{(18.2)} & \textbf{77.3}\scriptsize{(3.3)}  & $1 \mathrm{e} 7$ \\
3s5z              & Hard       & \textbf{97.5}\scriptsize{(1.2)} & 93.8\scriptsize{(1.2)}           & \textbf{96.9}\scriptsize{(0.7)}  & 89.8\scriptsize{(2.5)}           & $1 \mathrm{e} 7$ \\
10m\_vs\_11m      & Hard       & 87.5\scriptsize{(6.7)}          & \textbf{98.8}\scriptsize{(0.6)}  & 96.9\scriptsize{(4.8)}           & 95.3\scriptsize{(2.2)}           & $1 \mathrm{e} 7$ \\
MMM2              & Super Hard & 88.8\scriptsize{(2.0)}          & \textbf{97.5}\scriptsize{(6.4)}  & \textbf{93.8}\scriptsize{(4.7)}  & 87.5\scriptsize{(2.5)}           & $2 \mathrm{e} 7$ \\
3s5z\_vs\_3s6z    & Super Hard & 66.2\scriptsize{(3.1)}          & 72.5\scriptsize{(14.7)}          & 70.0\scriptsize{(10.7)}          & \textbf{87.5}\scriptsize{(12.6)} & $2 \mathrm{e} 7$ \\
27m\_vs\_30m      & Super Hard & 76.6\scriptsize{(1.3)}          & \textbf{93.8}\scriptsize{(2.1)}  & 80.0\scriptsize{(6.2)}           & 45.3\scriptsize{(14.0)}          & $2 \mathrm{e} 7$ \\
corridor          & Super Hard & 92.5\scriptsize{(13.9)}         & 88.8\scriptsize{(2.7)}           & \textbf{97.5}\scriptsize{(1.2)}  & 82.8\scriptsize{(4.4)}           & $2 \mathrm{e} 7$ \\
6h\_vs\_8z        & Super Hard & \textbf{76.2}\scriptsize{(3.1)} & \textbf{78.8}\scriptsize{(0.6)}  & \textbf{85.0}\scriptsize{(2.0)}  & \textbf{92.2}\scriptsize{(26.2)} & $4 \mathrm{e} 7$ \\ \hline
protoss\_5\_vs\_5 & -          & 57.5\scriptsize{(1.2)}          & 50.0\scriptsize{(2.4)}           & 56.2\scriptsize{(3.2)}           & \textbf{65.6}\scriptsize{(3.9)}           & $1 \mathrm{e} 7$ \\
terran\_5\_vs\_5  & -          & 57.5\scriptsize{(1.3)}          & 56.8\scriptsize{(2.9)}  & 53.1\scriptsize{(2.7)}           & \textbf{62.5}\scriptsize{(3.8)}           & $1 \mathrm{e} 7$ \\
zerg\_5\_vs\_5    & -          & 42.5\scriptsize{(2.5)}          & \textbf{43.8}\scriptsize{(1.2)}  & 40.6\scriptsize{(7.0)}           & 34.4\scriptsize{(2.2)}           & $1 \mathrm{e} 7$ \\
zerg\_10\_vs\_10  & -          & 28.4\scriptsize{(2.2)}          & 34.6\scriptsize{(0.2)}           & 37.5\scriptsize{(3.2)}           & \textbf{40.6}\scriptsize{(3.4)}           & $1 \mathrm{e} 7$ \\
zerg\_10\_vs\_11  & -          & 16.2\scriptsize{(0.6)}          & 19.3\scriptsize{(2.1)}           & \textbf{29.7}\scriptsize{(3.8)}           & 25.0\scriptsize{(3.9)}           & $1 \mathrm{e} 7$ \\ \hline
\end{tabular}
\caption{Median evaluation win rate and standard deviation on ten SMAC maps (upper in the table) and five SMACv2 maps (lower in the table) for different methods. All values within 1 standard deviation of the maximum win rate are marked in bold. The column labeled ``Steps'' specifies the number of steps used for training. Our results suggest that HAPPO and HATRPO perform comparably or better than MAPPO and QMIX on these tasks, which mainly involve homogeneous agents. Moreover, HAPPO and HATRPO do not rely on the restrictive parameter-sharing technique, demonstrating their versatility in various scenarios.}
\label{tab:smac}
\end{table}

We observe that HAPPO and HATRPO are able to achieve comparable or superior performance to QMIX and MAPPO across five hard maps and five super hard maps in SMAC, while not relying on the restrictive parameter-sharing trick, as opposed to MAPPO. From the learning curves, it shows that HAPPO and HATRPO exhibit steadily improving learning behaviours, while baselines experience large oscillations on \texttt{25m} and \texttt{27m\_vs\_30m}, again demonstrating the monotonic improvement property of our methods. On SMACv2, though randomness and heterogeneity increase, HAPPO and HATRPO robustly achieve competitive win rates and are comparable to QMIX and MAPPO.
Another important observation is that HATRPO is more effective than HAPPO in SMAC and SMACv2, outperforming HAPPO on 10 out of 15 tasks. This implies that HATRPO could enhance learning stability by imposing explicit constraints on update distance and reward improvement, making it a promising approach to tackling novel and challenging tasks. Overall, the performance of HAPPO and HATRPO in SMAC and SMACv2 confirms their capability to coordinate agents' training in largely homogeneous settings.

\subsection{Google Research Football Testbed}

Google Research Football Environment (GRF) composes a series of tasks where agents are trained to play football in an advanced, physics-based 3D simulator. 
From literature \citep{mappo}, it is shown that GRF is still challenging to existing methods. We apply HAPPO to the five academy tasks of GRF, namely 3 vs 1 with keeper (3v.1), counterattack (CA) easy and hard, 
pass and shoot with keeper (PS), and run pass and shoot with keeper (RPS),  with MAPPO and QMIX as baselines. As GRF does not provide a global state interface, our solution is to implement a global state based on agents' observations following the \texttt{Simple115StateWrapper} of GRF. Concretely, the global state consists of common components in agents' observations and the concatenation of agent-specific parts, and is taken as input by the centralised critic for value prediction. We also utilize the dense-reward setting in GRF. All methods are trained for 25 million environment steps in all scenarios with the exception of CA (hard), in which methods are trained for 50 million environment steps. We compute the success rate over 100 rollouts of the game and report the average success rate over the last 10 evaluations across 6 seeds in Table \ref{tab:football}. We also report the learning curves of the algorithms in Figure \ref{fig:football}.

\begin{table}[!t]
\centering
\begin{tabular}{cccc}
\hline
scenarios             & HAPPO                             & MAPPO                             & QMIX                      \\ \hline
PS                    & \textbf{96.93}\scriptsize{(1.11)} & 94.92\scriptsize{(0.85)}          & 8.05\scriptsize{(5.58)}   \\
RPS                   & \textbf{77.30}\scriptsize{(7.40)} & \textbf{76.83}\scriptsize{(3.57)} & 8.08\scriptsize{(3.29)}   \\
3v.1                  & \textbf{94.74}\scriptsize{(3.05)} & 88.03\scriptsize{(4.15)}          & 8.12\scriptsize{(4.46)}   \\
CA\scriptsize{(easy)} & \textbf{92.00}\scriptsize{(1.62)} & 87.76\scriptsize{(6.40)}          & 15.98\scriptsize{(11.77)} \\
CA\scriptsize{(hard)} & \textbf{88.14}\scriptsize{(5.77)} & 77.38\scriptsize{(10.95)}         & 3.22\scriptsize{(4.39)}   \\ \hline
\end{tabular}
\caption{Average evaluation score rate and standard deviation (over six seeds) on GRF scenarios for different methods. All values within 1 standard deviation of the maximum score rate are marked in bold. Our results reveal that HAPPO generally outperforms MAPPO and QMIX on all tasks, setting a new state-of-the-art performance benchmark.}
\label{tab:football}
\end{table}

\begin{figure}[tbp]
  \centering
  \includegraphics[width=\linewidth]{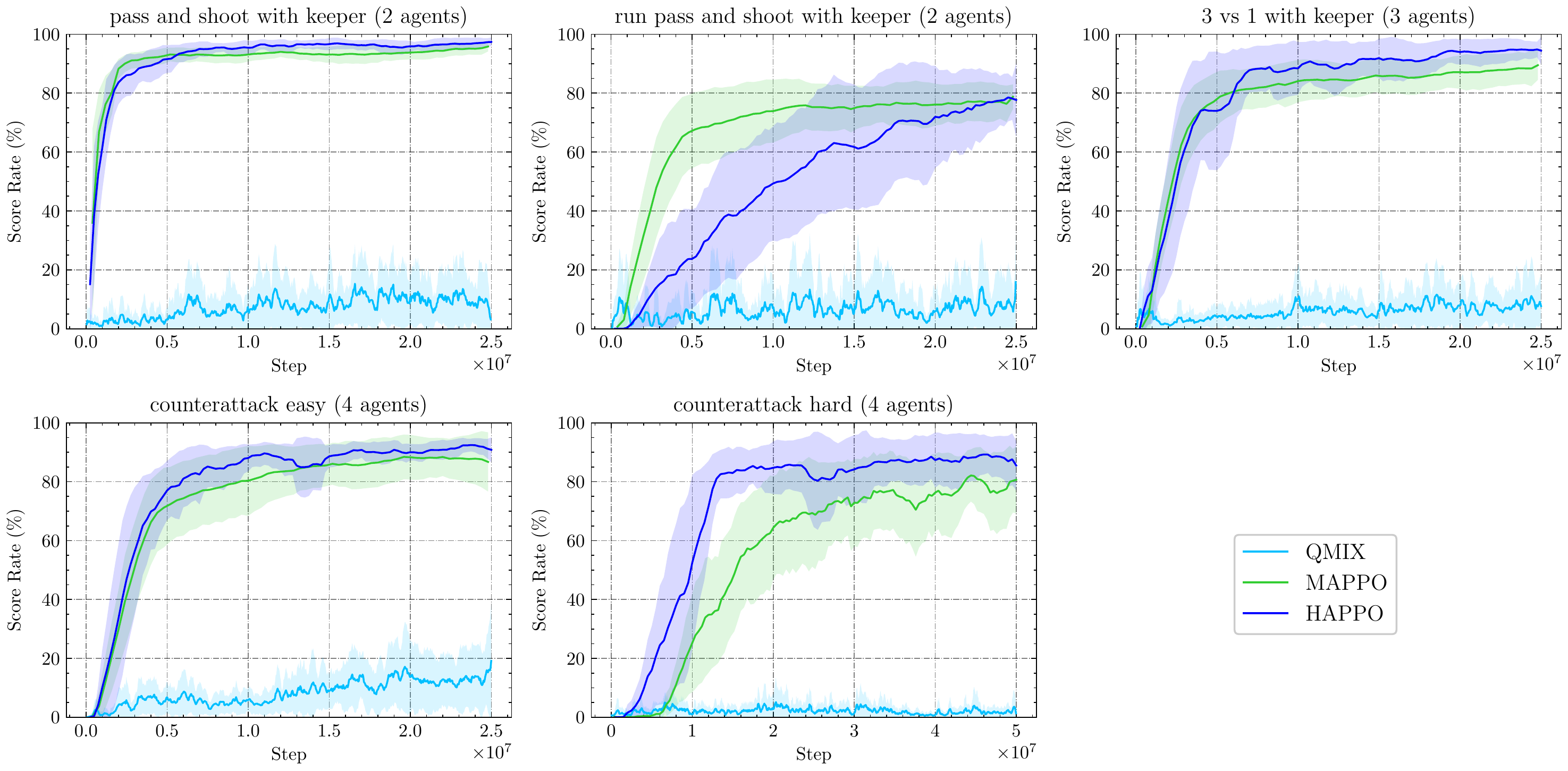}
  \caption{The figure displays the average score rate comparisons for different methods on GRF, and also illustrates how the performance gaps between HAPPO and MAPPO widen as the roles and difficulty levels of tasks increase. Overall, our results demonstrate that HAPPO outperforms MAPPO in tackling complex multi-agent scenarios.}
  \label{fig:football}
\end{figure}

We observe that HAPPO is generally better than MAPPO, establishing new state-of-the-art results, and they both significantly outperform QMIX. In particular, as the number of agents increases and the roles they play become more diverse, the performance gap between HAPPO and MAPPO becomes larger, again showing the effectiveness and advantage of HARL algorithms for the many-heterogeneous-agent settings. From the rendered videos, it is shown that agents trained by HAPPO develop clever teamwork strategies for ensuring a high score rate, such as cooperative breakthroughs to form one-on-one chances, etc. 
This result further supports the effectiveness of applying HAPPO to cooperative MARL problems.

\subsection{Bi-DexterousHands Testbed}

Based on IsaacGym, Bi-DexterousHands provides a suite of tasks for learning human-level bimanual dexterous manipulation. It leverages GPU parallelisation and enables simultaneous instantiation of thousands of environments. Compared with other CPU-based environments, Bi-DexterousHands significantly increases the number of samples generated in the same time interval, thus alleviating the sample efficiency problem of on-policy algorithms. We choose three representative tasks and compare HAPPO with MAPPO as well as PPO. As the existing reported results of MAPPO on these tasks do not utilize parameter sharing, we follow them in order to be consistent. The learning curves plotted from training data across three random seeds are shown in Figure \ref{fig:dexhands}. On all three tasks, HAPPO consistently outperforms MAPPO, and is at least comparable to or better than the single-agent baseline PPO, while also showing less variance. The comparison between HAPPO and MAPPO demonstrates the superior competence of the sequential update scheme adopted by HARL algorithms over simultaneous updates for coordinating multiple heterogeneous agents.

\begin{figure}[tbp]
  \centering
  \includegraphics[width=\linewidth]{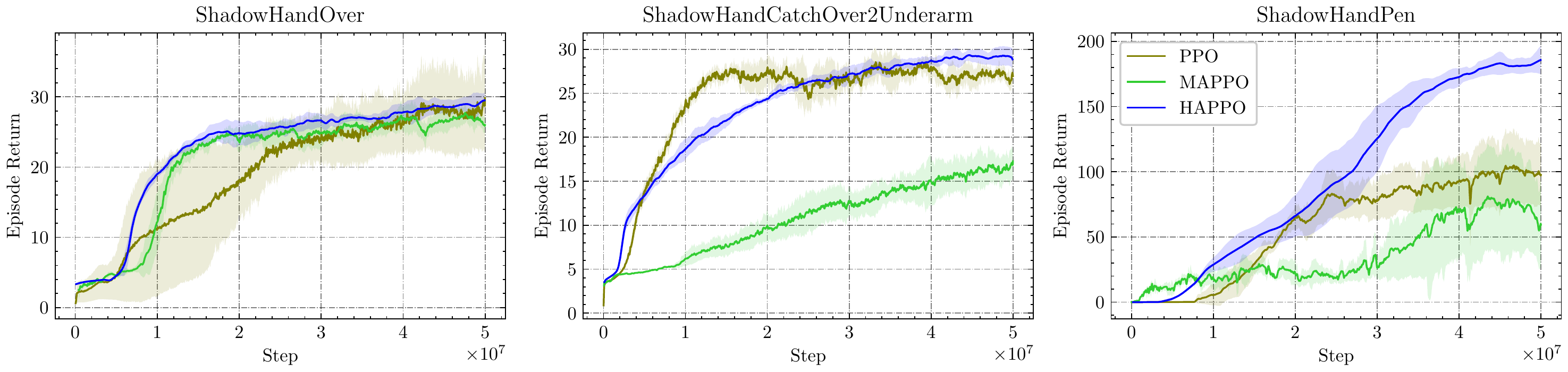}
  \caption{Comparisons of average episode return on Bi-DexterousHands. The learning curves demonstrate that HAPPO consistently achieves the highest return, outperforming both MAPPO and PPO.}
  \label{fig:dexhands}
\end{figure}

\subsection{Ablation Experiments}
\label{subsec:ablations}
In this subsection, we conduct ablation study to investigate the importance of two key novelties that our HARL algorithms introduced; they are heterogeneity of agents' parameters and the randomisation of order of agents in the sequential update scheme. 
We compare the performance of original HAPPO with a version that shares parameters, and with a version where the order in sequential update scheme is fixed throughout training. We run the experiments on two MAMuJoCo tasks, namely 2-agent Walker and 6-agent Walker.

\begin{figure}[tbp]
  \centering
  \includegraphics[width=0.8\linewidth]{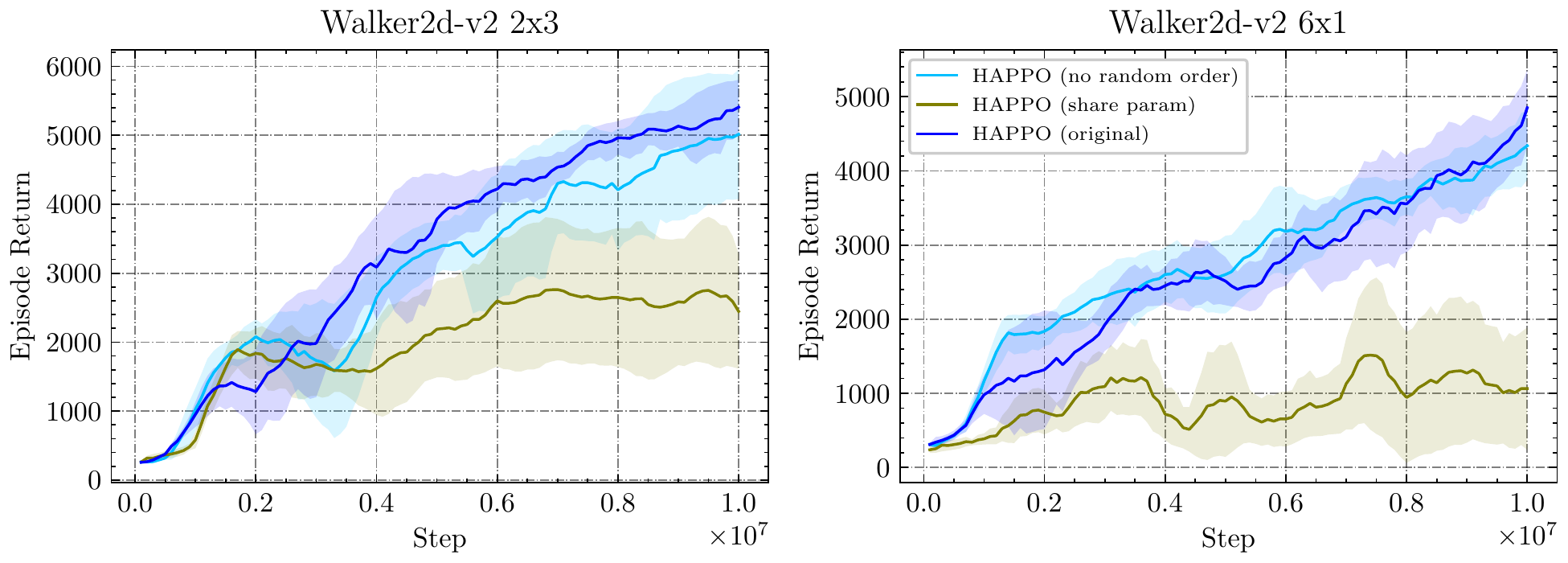}
  \caption{Performance comparison between original HAPPO, and its modified versions: HAPPO with parameter sharing, and HAPPO without randomisation of the sequential update scheme.}
  \label{fig:mamujoco-ablation}
\end{figure}
 
The experiments reveal that the deviation from the theory harms performance. In particular, parameter sharing introduces unreasonable policy constraints to training, harms the monotonic improvement property (Theorem \ref{theorem:monotonic-matrpo} assumes heterogeneity), and causes HAPPO to converge to suboptimal policies. The suboptimality is more severe in the task with more diverse agents, as discussed in Section \ref{sec:hh}. Similarly, fixed order in the sequential update scheme negatively affects the performance at convergence, as suggested by Theorem \ref{proposition:convergence-matrpo}. In the 2-agent task, fixing update order leads to inferior performance throughout the training process; in the 6-agent task, while the fixed order version initially learns faster, it is gradually overtaken by the randomised order version and achieves worse convergence results. We conclude that the fine performance of HARL algorithms relies strongly on the close connection between theory and implementation. 

\subsection{Analysis of Computational Overhead}
We then analyse the computational overhead introduced by the sequential update scheme. We mainly compare HAPPO with MAPPO in parameter-sharing setting, where our implementation conducts the single vectorized update \footnote{Corresponding to the original implementation at \url{https://github.com/marlbenchmark/on-policy/blob/0affe7f4b812ed25e280af8115f279fbffe45bbe/onpolicy/algorithms/r\_mappo/r\_mappo.py\#L205}.}. We conduct experiments on seven MAMuJoCo tasks with all hyperparameters fixed. Both methods are trained for 1 million steps and we record the computational performance in Table \ref{tab:comp-overhead}. The machine for experiments in this subsection is equipped with an AMD Ryzen 9 5950X 16-Core Processor and an NVIDIA RTX 3090 Ti GPU, and we ensure that no other experiments are running.

\begin{table}[tbp]
    \centering
\begin{tabular}{c|ccc|ccc}
\hline
\multirow{2}{*}{scenarios} & \multicolumn{3}{c|}{HAPPO}                                                                                                                           & \multicolumn{3}{c}{MAPPO}                                                                                                                                 \\ \cline{2-7} 
                           & \begin{tabular}[c]{@{}c@{}}experiment\\ time(s)\end{tabular} & \begin{tabular}[c]{@{}c@{}}agents\\ update\\ time(s)\end{tabular} & FLOPS & \begin{tabular}[c]{@{}c@{}}experiment\\ time(s)\end{tabular} & \begin{tabular}[c]{@{}c@{}}share param\\ update\\ time(s)\end{tabular} & FLOPS \\ \hline
HalfCheetah 2x3 & $203.4_{0.4}$ & $ 8.6_{0.0}$ & 368 & $197.6_{1.0}$ & $ 4.9_{0.1}$ & 588 \\
HalfCheetah 3x2 & $264.7_{1.0}$ & $ 12.9_{0.0}$ & 368 & $256.0_{1.1}$ & $ 6.2_{0.0}$ & 644 \\
HalfCheetah 6x1 & $451.0_{1.4}$ & $ 25.3_{0.0}$ & 366 & $441.4_{0.9}$ & $ 12.3_{0.0}$ & 717 \\
Walker 2x3 & $193.9_{2.8}$ & $ 8.6_{0.0}$ & 368 & $187.4_{4.8}$ & $ 4.9_{0.0}$ & 588 \\
Walker 3x2 & $245.0_{6.2}$ & $ 12.9_{0.2}$ & 368 & $232.6_{1.5}$ & $ 6.3_{0.0}$ & 649 \\
Walker 6x1 & $408.6_{9.9}$ & $ 25.4_{0.2}$ & 370 & $383.2_{5.6}$ & $ 12.2_{0.0}$ & 711 \\
Humanoid 17x1 & $912.0_{11.8}$ & $ 76.7_{0.9}$ & 568 & $988.0_{2.1}$ & $ 71.3_{0.1}$ & 738 \\ \hline
\end{tabular}
\caption{Computational performance comparisons between HAPPO and MAPPO on seven MAMuJoCo tasks across three seeds. As for the comparison items, ``experiment time'' denotes the overall running time of a single experiment; ``agents update time'' of HAPPO denotes the total time of all agent updates; ``share param update time'' of MAPPO denotes the total time consumed in updating the shared parameters; ``FLOPS'' (floating-point operations per second) during the update is calculated as the total floating-point operations in a network forward pass divided by data transfer time plus computation time (unit: GFLOPS). The main figure represents the mean and the subscript represents the standard deviation. These figures suggest that the sequential update scheme does not introduce much computational burden compared to a single vectorized update.}
\label{tab:comp-overhead}
\end{table}

\begin{figure}[tbp]
    \centering
    \begin{subfigure}{0.45\textwidth}
        \includegraphics[width=\linewidth]{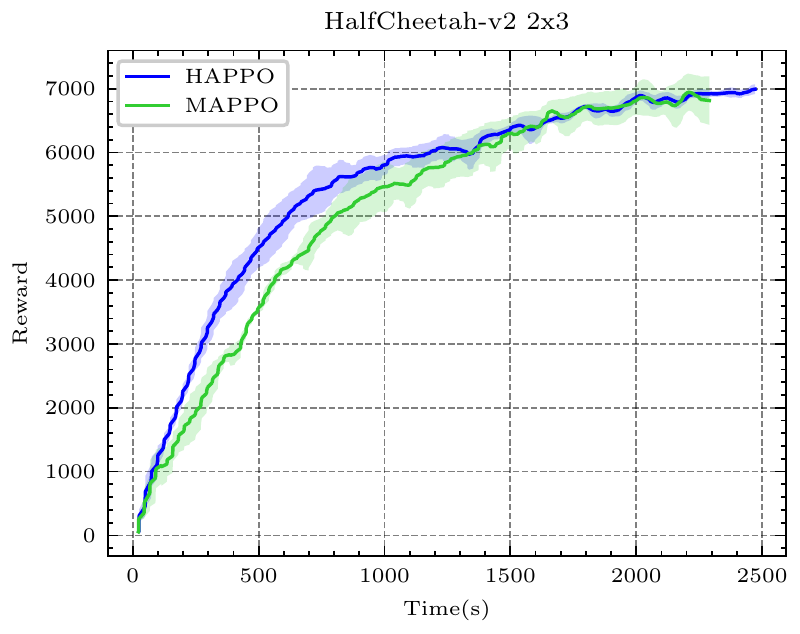} 
        \label{fig:halfcheetah-2x3-time} 
        \caption{Return vs. time on 2-agent HalfCheetah.}
    \end{subfigure}
    \hfill 
    \begin{subfigure}{0.45\textwidth}
        \includegraphics[width=\linewidth]{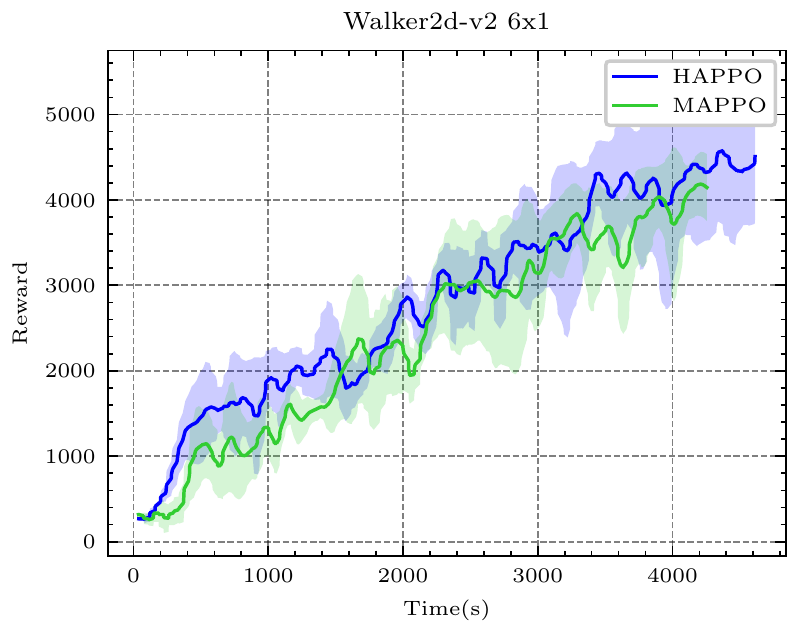} 
        \label{fig:walker-6x1-time} 
        \caption{Return vs. time on 6-agent Walker.}
    \end{subfigure}
    \caption{Performance comparison between HAPPO and MAPPO with the x-axis being the wall-time. At the same time, HAPPO generally outperforms parameter-sharing MAPPO.}
    \label{fig:compare-step-time}
\end{figure}

We generally observe a linear relationship between update times for both HAPPO and MAPPO and the number of agents. For HAPPO, each agent is trained on a constant-sized batch input, denoted as $|B|$, across tasks. Thus the total time consumed to update all agents correlates directly with the agent count. For MAPPO, on the other hand, the shared parameter is trained on a batch input of size $n\times|B|$ when the number of agents is $n$. However, as the batch size $|B|$ used in MAMuJoCo is typically large, in this case 4000, vectorizing agents data does not significantly enhance GPU parallelization. This is evidenced by the relatively consistent FLOPS recorded across tasks. As a result, the MAPPO update timeframe also exhibits linear growth with increasing agents. The ratio of HAPPO and MAPPO update time is almost constant on the first six tasks and it nearly degenerates to 1 when both of them sufficiently utilize the computational resources, as shown in the case of 17-agent Humanoid where the significantly higher-dimensional observation space leads to increased GPU utilization, \emph{i.e.} FLOPS, for HAPPO. These facts suggest that the sequential update scheme does not introduce much computational burden compared to the single vectorized update. As the update only constitutes a small portion of the whole experiment, such an additional computational overhead is almost negligible.

In Figure \ref{fig:compare-step-time}, we further provide the learning curves of HAPPO and MAPPO on two MAMuJoCo tasks corresponding to Figure \ref{fig:mamujoco-on-policy}, with the x-axis being wall-time. The oscillation observed in Figure \ref{fig:compare-step-time}(b) is due to a slight difference in training time across the seeds rather than the instability of algorithms. These figures demonstrate that HAPPO generally outperforms MAPPO at the same wall-time. To run 10 million steps, HAPPO needs $8.12\%$ and $8.64\%$ more time than MAPPO respectively, an acceptable tradeoff to enjoy the benefits of the sequential update scheme in terms of improved performance and rigorous theoretical guarantees. Thus, we justify that computational overhead does not need to be a concern.

\section{Conclusion}

In this paper, we present Heterogeneous-Agent Reinforcement Learning (HARL) algorithm series, a set of powerful solutions to cooperative multi-agent problems with theoretical guarantees of monotonic improvement and convergence to Nash Equilibrium. Based on the multi-agent advantage decomposition lemma and the sequential update scheme, we successfully develop Heterogeneous-Agent Trust Region Learning (HATRL) and introduce two practical algorithms --- HATRPO and HAPPO --- by tractable approximations. We further discover the Heterogeneous-Agent Mirror Learning (HAML) framework, which strengthens validations for HATRPO and HAPPO and is a general template for designing provably correct MARL algorithms whose properties are rigorously profiled. Its consequences are the derivation of more HARL algorithms, HAA2C, HADDPG, and HATD3, which significantly enrich the tools for solving cooperative MARL problems. 
Experimental analysis on MPE, MAMuJoCo, SMAC, SMACv2, GRF, and Bi-DexterousHands confirms that HARL algorithms generally outperform existing MA-counterparts and refresh SOTA results on heterogeneous-agent benchmarks, showing their superior effectiveness for heterogeneous-agent cooperation over strong baselines such as MAPPO and QMIX.
Ablation studies further substantiate the key novelties required in theoretical reasoning and enhance the connection between HARL theory and implementation. For future work, we plan to consider more possibilities of the HAML framework and validate the effectiveness of HARL algorithms on real-world multi-robot cooperation tasks.



\acks{We would like to thank Chengdong Ma for insightful discussions; the authors of MAPPO \citep{mappo} for providing original training data of MAPPO and QMIX on SMAC and GRF; the authors of SMACv2 \citep{ellis2022smacv2} for providing original training data of MAPPO and QMIX on SMACv2; and the authors of Bi-DexterousHands \citep{chen2022towards} for providing original training data of MAPPO and PPO on Bi-DexterousHands. 

This project is funded by National Key R\&D Program of China (2022ZD0114900) , Collective Intelligence \& Collaboration Laboratory (QXZ23014101) , CCF-Tencent Open Research Fund (RAGR20220109) , Young Elite Scientists Sponsorship Program by CAST (2022QNRC002), Beijing Municipal Science \& Technology Commission (Z221100003422004).}



\appendix

\section{Proofs of Example \ref{eg:suboptimal} and \ref{eg:needscare}}

\label{appendix:proof-of-examples}

\begin{restatable}{example}{suboptimal}
\label{eg:suboptimal}
Consider a fully-cooperative game with an even number of agents $n$, one state, and the joint action space $\{0, 1\}^{n}$, where the reward is given by $r(\boldsymbol{0}^{n/2}, \boldsymbol{1}^{n/2}) 
= r(\boldsymbol{1}^{n/2}, \boldsymbol{0}^{n/2}) = 1$, and $r(\va^{1:n}) =0$ for all other joint actions. Let $J^{*}$ be the optimal joint reward, and $J^*_{\text{share}}$ be the optimal joint reward under the shared policy constraint. Then 
\begin{align}
    \frac{J^*_{\text{share}}}{J^{*}} = \frac{2}{2^n}.\nonumber
\end{align}
\end{restatable}

\begin{proof}
   Clearly $J^* = 1$. An optimal joint policy in this case is, for example, the deterministic policy with joint action $(\boldsymbol{0}^{n/2}, \boldsymbol{1}^{n/2})$. 
   
   Now, let the shared policy be $(\theta, 1-\theta)$, where $\theta$ determines the probability that an agent takes action $0$. Then, the expected reward is 
   \begin{align}
       &J(\theta) = \text{Pr}\left( \va^{1:n} = (\boldsymbol{0}^{n/2}, \boldsymbol{1}^{n/2})\right)\cdot 1 +
       \text{Pr}\left( \va^{1:n} = (\boldsymbol{1}^{n/2}, \boldsymbol{0}^{n/2})\right)\cdot 1
       = 2\cdot \theta^{n/2}(1-\theta)^{n/2}.\nonumber
   \end{align}
   In order to maximise $J(\theta)$, we must maximise $\theta (1-\theta)$, or equivalently, $\sqrt{\theta(1-\theta)}$. By the artithmetic-geometric means inequality, we have
   \begin{align}
       \sqrt{\theta(1-\theta)} \leq \frac{\theta + (1-\theta)}{2} = \frac{1}{2},\nonumber
   \end{align}
   where the equality holds if and only if $\theta = 1-\theta$, that is $\theta = \frac{1}{2}$. In such case we have
   \begin{align}
       J^*_{\text{share}} = J\left(\frac{1}{2}\right) = 2\cdot 2^{-n/2}\cdot 2^{-n/2} = \frac{2}{2^n},\nonumber
   \end{align}
   which finishes the proof.
\end{proof}

\needscare*
\begin{proof}
As there is only one state, we can ignore the infinite horizon and the discount factor $\gamma$, thus making the state-action value and the reward functions equivalent, $Q\equiv r$.

Let us, for brevity, define $\pi^i = \pi^i_{\text{old}}(0) > 0.6$, for $i=1, 2$. We have 
\begin{align}
    J(\vpi_{\text{old}}) &= \text{Pr}(\ra^1 = \ra^2 = 0)r(0, 0) + \big(1 -\text{Pr}(\ra^1 = \ra^2 = 0)\big)\E[r(\ra^1, \ra^2)| (\ra^1, \ra^2)\neq(0, 0)] \nonumber\\
    &> 0.6^2\times 0 - (1-0.6^2) = -0.64.\nonumber
\end{align}
The update rule stated in the proposition can be equivalently written as
\begin{align}
    \label{eq:restate-prop}
    \pi^i_{\text{new}} = \argmax\limits_{\pi^i}\E_{\ra^i\sim\pi^i, \ra^{-i}\sim\pi^{-i}_{\text{old}}}\big[ Q_{\vpi_{\text{old}}}(\ra^i, \ra^{-i}) \big].
\end{align}
We have
\begin{align}
    \E_{\ra^{-i}\sim\pi^{-i}_{\text{old}}}\big[ Q_{\vpi_{\text{old}}}(0, \ra^{-i})\big] =
    \pi^{-i}Q(0, 0) + (1-\pi^{-i})Q(0, 1) = \pi^{-i}r(0, 0) + (1-\pi^{-i})r(0, 1) = 2(1-\pi^{-i}),\nonumber
\end{align}
and similarly
\begin{align}
    \E_{\ra^{-i}\sim\pi^{-i}_{\text{old}}}\big[ Q_{\vpi_{\text{old}}}(1, \ra^{-i})\big] = \pi^{-i}r(1, 0) + (1-\pi^{-i})r(1, 1)= 2\pi^{-i} - (1-\pi^{-i}) = 3\pi^{-i} - 1.\nonumber
\end{align}
Hence, if $\pi^{-i}>0.6$, then
\begin{align}
    \E_{\ra^{-i}\sim\pi^{-i}_{\text{old}}}\big[ Q_{\vpi_{\text{old}}}(1, \ra^{-i})\big] = 3\pi^{-i}-1>3\times 0.6 - 1= 0.8 > 2 - 2\pi^{-i} = 
    \E_{\ra^{-i}\sim\pi^{-i}_{\text{old}}}\big[ Q_{\vpi_{\text{old}}}(0, \ra^{-i})\big].\nonumber
\end{align}
Therefore, for every $i$, the solution to Equation (\ref{eq:restate-prop}) is the greedy policy $\pi^{i}_{\text{new}}(1) = 1$. Therefore, 
\begin{align}
    J(\vpi_{\text{new}}) = Q(1,1) = r(1, 1) = -1,\nonumber
\end{align}
which finishes the proof.
\end{proof}

\section{Derivation and Analysis of Algorithm \ref{algorithm:theoretical-matrpo}}
\label{appendix:theoretical-matrpo}

\subsection{Recap of Existing Results}

\begin{restatable}[Performance Difference]{lemma}{performancedifference}
\label{lemma:performance-difference}
Let $\bar{\pi}$ and $\pi$ be two policies. Then, the following identity holds,
\begin{align}
    J(\bar{\pi}) - J(\pi) = \E_{\rs\sim\rho_{\bar{\pi}}, \ra\sim\bar{\pi}}\left[ A_{\pi}(\rs, \ra) \right].\nonumber
\end{align}
\end{restatable}
\begin{proof}
    See \cite{kakade2002approximately} (Lemma 6.1) or \cite{trpo} (Appendix A).
\end{proof}

\begin{restatable}{theorem}{trpoinequality}\citep[Theorem 1]{trpo}
    \label{theorem:trpo-ineq}
    Let $\pi$  be the current policy and $\bar{\pi}$ be the next candidate policy. We define
        $ L_{\pi}(\bar{\pi}) = J(\pi) + \E_{\rs\sim\rho_{\pi}, \ra\sim\bar{\pi}}\left[ A_{\pi}(s, a) \right], 
         \text{{\normalfont D}}_{\text{KL}}^{\text{max}}(\pi, \bar{\pi}) = \max_{s}\text{{\normalfont D}}_{\text{KL}}\left( \pi(\cdot|s), \bar{\pi}(\cdot|s) \right).$
   Then the  inequality of 
    \begin{align}
        J(\bar{\pi}) \geq L_{\pi}(\bar{\pi}) - C\text{{\normalfont D}}_{\text{KL}}^{\text{max}}\big(\pi, \bar{\pi}\big)
    \end{align}
   holds, where $C = \frac{4\gamma\max_{s, a}|A_{\pi}(s, a)|}{(1-\gamma)^{2}}$.
\end{restatable}
\begin{proof}
    See \cite{trpo} (Appendix A and Equation (9) of the paper).
\end{proof}

\subsection{Analysis of Training of Algorithm \ref{algorithm:theoretical-matrpo}}
\label{appendix:analysis-training-the-matrpo}
\maadlemma*
\begin{proof}
    By the definition of multi-agent advantage function,
    \begin{align}
        &A^{i_{1:m}}_{\boldsymbol{\pi}}(s, \va^{i_{1:m}}) = 
        Q^{i_{1:m}}_{\boldsymbol{\pi}}(s, \va^{i_{1:m}})
        -  V_{\boldsymbol{\pi}}(s)\nonumber\\
        &= \sum_{k=1}^{m}\left[Q^{i_{1:k}}_{\boldsymbol{\pi}}(s, \va^{i_{1:k}}) - Q^{i_{1:k-1}}_{\boldsymbol{\pi}}(s, \va^{i_{1:k-1}})  \right]\nonumber\\
        &= \sum_{k=1}^{m}A^{i_k}_{\boldsymbol{\pi}}(s, \va^{i_{1:k-1}}, a^{i_k}),\nonumber
    \end{align}
    which finishes the proof.
    \newline
    Note that a similar finding has been shown in \cite{kuba2021settling}.
\end{proof}

\begin{restatable}{lemma}{localKLglobalKL}
    \label{lemma:kl-inequality}
     Let $\boldsymbol{\pi} = \prod_{i=1}^{n}\pi^{i}$ and $\boldsymbol{\bar{\pi}} = \prod_{i=1}^{n}\bar{\pi}^{i}$ be joint policies. Then
     \begin{align}
        \text{{\normalfont D}}_{\text{KL}}^{\text{max}}\left(\boldsymbol{\pi}, \boldsymbol{\bar{\pi}}\right) \leq
        \sum_{i=1}^{n}\text{{\normalfont D}}_{\text{KL}}^{\text{max}}\left( \pi^{i}, \bar{\pi}^{i} \right) \nonumber
     \end{align}
\end{restatable}
\begin{proof}
    For any state $s$, we have
    \begin{align}
        &\text{{\normalfont D}}_{\text{KL}}\left( \boldsymbol{\pi}(\cdot|s), \boldsymbol{\bar{\pi}}(\cdot|s) \right) = \E_{\rva\sim\boldsymbol{\pi}}\left[ \log \boldsymbol{\pi}(\rva|s) - \log \boldsymbol{\bar{\pi}}(\rva|s) \right]\nonumber\\
        &= \E_{\rva\sim\boldsymbol{\pi}}\left[ \log \left( \prod_{i=1}^{n}{\pi^{i}}(\ra^{i}|s) \right)- 
        \log \left(\prod_{i=1}^{n}\bar{\pi}^{i}(\ra^{i}|s) \right) \right]\nonumber\\
        &= \E_{\rva\sim\boldsymbol{\pi}}\left[ 
        \sum_{i=1}^{n}\log  \pi^{i}(\ra^{i}|s) - 
        \sum_{i=1}^{n}\log \bar{\pi}^{i}(\ra^{i}|s)  
        \right]\nonumber\\
        &= \sum_{i=1}^{n}\E_{\ra^{i}\sim\pi^{i}, \rva^{-i}\sim\boldsymbol{\pi}^{-i}}
        \left[  \log \pi^{i}(\ra^{i}|s) - \log \bar{\pi}^{i}(\ra^{i}|s)\right] = \sum_{i=1}^{n}\text{{\normalfont D}}_{\text{KL}}\left(\pi^{i}(\cdot|s), \bar{\pi}^{i}(\cdot|s)\right).
        \label{eq:kl}
    \end{align}
    Now, taking maximum over $s$ on both sides yields
    \begin{align}
         \text{{\normalfont D}}_{\text{KL}}^{\text{max}}\left(\boldsymbol{\pi}, \boldsymbol{\bar{\pi}}\right) \leq
        \sum_{i=1}^{n}\text{{\normalfont D}}_{\text{KL}}^{\text{max}}\left( \pi^{i}, \bar{\pi}^{i} \right), \nonumber
    \end{align}
    as required.
\end{proof}

\trpotosadtrpo*
\begin{proof}
    By Theorem \ref{theorem:trpo-ineq}
    \begin{align}
        &J(\boldsymbol{\bar{\pi}}) \geq L_{\boldsymbol{\pi}}(\boldsymbol{\bar{\pi}}) - C\text{{\normalfont D}}_{\text{KL}}^{\text{max}}(\boldsymbol{\pi}, \boldsymbol{\bar{\pi}}) \nonumber\\
        &= J(\boldsymbol{\pi}) + \E_{\rs\sim\rho_{\boldsymbol{\pi}}, \rva\sim\boldsymbol{\bar{\pi}}}
        \left[ A_{\boldsymbol{\pi}}(\rs, \rva) \right]
         - C\text{{\normalfont D}}_{\text{KL}}^{\text{max}}(\boldsymbol{\pi}, \boldsymbol{\bar{\pi}}) \nonumber\\
        &\text{which by Lemma \ref{lemma:maadlemma} equals}\nonumber\\
        &= J(\boldsymbol{\pi}) + \E_{\rs\sim\rho_{\boldsymbol{\pi}}, \rva\sim\boldsymbol{\bar{\pi}}}
        \left[ \sum_{m=1}^{n}A^{i_{m}}_{\boldsymbol{\pi}}\left(\rs, \rva^{i_{1:m-1}}, \ra^{i_{m}} \right)\right]
         - C\text{{\normalfont D}}_{\text{KL}}^{\text{max}}(\boldsymbol{\pi}, \boldsymbol{\bar{\pi}}) \nonumber
    \end{align}
    \begin{align}
        &\text{and by Lemma \ref{lemma:kl-inequality} this is at least}\nonumber\\
        &\geq J(\boldsymbol{\pi}) + \E_{\rs\sim\rho_{\boldsymbol{\pi}}, \rva\sim\boldsymbol{\bar{\pi}}}
        \left[ \sum_{m=1}^{n}A^{i_{m}}_{\boldsymbol{\pi}}\left(\rs, \rva^{i_{1:m-1}}, \ra^{i_{m}} \right)\right]
        - \sum_{m=1}^{n}C\text{{\normalfont D}}_{\text{KL}}^{\text{max}}(\pi^{i_{m}}, \bar{\pi}^{i_{m}}) \nonumber\\
        &= J(\boldsymbol{\pi}) +
        \sum_{m=1}^{n}
        \E_{\rs\sim\rho_{\boldsymbol{\pi}}, \rva^{i_{1:m-1}}\sim\boldsymbol{\bar{\pi}}^{i_{1:m-1}},
        \ra^{i_{m}}\sim\bar{\pi}^{i_{m}}}
        \left[A^{i_{m}}_{\boldsymbol{\pi}}\left(\rs, \rva^{i_{1:m-1}}, \ra^{i_{m}} \right)\right]
        - \sum_{m=1}^{n}C\text{{\normalfont D}}_{\text{KL}}^{\text{max}}(\pi^{i_{m}}, \bar{\pi}^{i_{m}}) \nonumber\\
        &= J(\boldsymbol{\pi}) +
        \sum_{m=1}^{n}\left(L^{i_{1:m}}_{\boldsymbol{\pi}}\left( 
        \boldsymbol{\bar{\pi}}^{i_{1:m-1}}, \bar{\pi}^{i_{m}}\right)
        - C\text{{\normalfont D}}_{\text{KL}}^{\text{max}}(\pi^{i_{m}}, \bar{\pi}^{i_{m}})\right).\nonumber
    \end{align}
\end{proof}

\matrpomonotonic*
\begin{proof}
    Let $\boldsymbol{\pi}_{0}$ be any joint policy. For every $k\geq 0$, the joint policy $\boldsymbol{\pi}_{k+1}$ is obtained from $\boldsymbol{\pi}_{k}$ by Algorithm \ref{algorithm:theoretical-matrpo} update; for $m=1, \dots, n$,
    \begin{align}
        \pi^{i_{m}}_{k+1} = \argmax_{\pi^{i_{m}}}\left[ L^{i_{1:m}}_{\boldsymbol{\pi}_{k}}\left( \boldsymbol{\pi}^{i_{1:m-1}}_{k+1}, \pi^{i_{m}} \right) 
        - C\text{{\normalfont D}}_{\text{KL}}^{\text{max}}\left(\pi^{i_{m}}_{k}, \pi^{i_{m}} \right)\right].\nonumber
    \end{align}
    By Theorem \ref{theorem:trpo-ineq},  we have
    \begin{align}
        \label{eq:sad-trpo-monotonic}
        &J(\boldsymbol{\pi}_{k+1}) \geq L_{\boldsymbol{\pi}_{k}}(\boldsymbol{\pi}_{k+1}) - C \text{{\normalfont D}}_{\text{KL}}^{\text{max}}(\boldsymbol{\pi}_{k}, \boldsymbol{\pi}_{k+1}),\nonumber\\
        &\text{which by Lemma \ref{lemma:kl-inequality} is lower-bounded by}\nonumber\\
        &\geq L_{\boldsymbol{\pi}_{k}}(\boldsymbol{\pi}_{k+1}) -  \sum_{m=1}^{n}C  \text{{\normalfont D}}_{\text{KL}}^{\text{max}}(\pi^{i_{m}}_{k}, \pi^{i_{m}}_{k+1})\nonumber\\
        &= J(\boldsymbol{\pi}_{k}) + 
        \sum_{m=1}^{n}\left( L^{i_{1:m}}_{\boldsymbol{\pi}_{k}}(\boldsymbol{\pi}^{i_{1:m-1}}_{k+1},\pi^{i_{m}}_{k+1}) 
        - C \text{{\normalfont D}}_{\text{KL}}^{\text{max}}(\pi^{i_{m}}_{k}, \pi^{i_{m}}_{k+1})\right),\\
        &\text{and as for every $m$, $\pi^{i_{m}}_{k+1}$ is the argmax, this is lower-bounded by}\nonumber\\
        &\geq J(\boldsymbol{\pi}_{k}) + 
        \sum_{m=1}^{n}\left( L^{i_{1:m}}_{\boldsymbol{\pi}_{k}}(\boldsymbol{\pi}^{i_{1:m-1}}_{k+1},\pi^{i_{m}}_{k}) 
        - C \text{{\normalfont D}}_{\text{KL}}^{\text{max}}(\pi^{i_{m}}_{k}, \pi^{i_{m}}_{k})\right),\nonumber\\
        &\text{which, as mentioned in Definition \ref{definition:localsurrogate}, equals}\nonumber\\
        &= J(\boldsymbol{\pi_{k}}) + \sum_{m=1}^{n}0 =
        J(\boldsymbol{\pi}_{k}),\nonumber
    \end{align}
    where the last inequality follows from Equation (\ref{eq:nice-maad-property}).
    This proves that Algorithm \ref{algorithm:theoretical-matrpo} achieves monotonic improvement.
\end{proof}

\subsection{Analysis of Convergence of Algorithm \ref{algorithm:theoretical-matrpo}}
\label{appendix:analysis-convergence-thematrpo}
\matrpoconvergence*
\begin{proof}
\paragraph{Step 1 (convergence). } Firstly, it is clear that the sequence $\left( J(\boldsymbol{\pi}_{k})\right)_{k=0}^{\infty}$ converges as, by Theorem \ref{theorem:monotonic-matrpo}, it is non-decreasing and bounded above by $\frac{R_{\text{max}}}{1-\gamma}$. Let us denote the limit by $\bar{J}$. For every $k$, we denote the tuple of agents, according to whose order the agents perform the sequential updates, by $i_{1:n}^k$, and we note that $\big( i_{1:n}^k\big)_{k\in\mathbb{N}}$ is a random process.
 Furthermore, we know that the sequence of policies $\left(\boldsymbol{\pi}_{k}\right)$ is bounded, so by Bolzano-Weierstrass Theorem, it has at least one convergent subsequence. Let $\boldsymbol{\bar{\pi}}$ be any limit point of the sequence (note that the set of limit points is a random set), and $\big(\boldsymbol{\pi}_{k_{j}}\big)_{j=0}^{\infty}$ be a subsequence converging to $\boldsymbol{\bar{\pi}}$ (which is a random subsequence as well). By continuity of $J$ in $\boldsymbol{\pi}$ 
 , we have
    \begin{align}
        J(\bar{\boldsymbol{\pi}}) = J\left( \lim_{j\to\infty}\boldsymbol{\pi}_{k_{j}}\right)
        =\lim_{j\to\infty}J\left( \boldsymbol{\pi}_{k_{j}}\right) = \bar{J}.
    \end{align}
    
    For now, we introduce an auxiliary definition.
    \begin{restatable}[TR-Stationarity]{definition}{tr-stationary}
    \label{def:tr-stationary}
     A joint policy $\boldsymbol{\bar{\pi}}$ is trust-region-stationary (TR-stationary) if, for every agent $i$,

     \begin{align}
         \bar{\pi}^{i} = \argmax_{\pi^{i}}\left[
         \E_{\rs\sim\rho_{\boldsymbol{\bar{\pi}}}, \ra^{i}\sim\pi^{i}}
         \left[ A^{i}_{\boldsymbol{\bar{\pi}}}(\rs, \ra^{i}) \right]
         -C_{\boldsymbol{\bar{\pi}}}\text{{\normalfont D}}_{\text{KL}}^{\text{max}}\left( \bar{\pi}^{i}, \pi^{i} \right)\right],\nonumber
     \end{align}
     where $C_{\boldsymbol{\bar{\pi}}} = \frac{4\gamma \epsilon}{(1-\gamma)^{2}}$, and $\epsilon = \max_{s, \va}|A_{\boldsymbol{\bar{\pi}}}(s, \va)|$.
\end{restatable}

We will now establish the TR-stationarity of any limit point joint policy $\bar{\boldsymbol{\pi}}$ (which, as stated above, is a random variable). Let $\E_{i_{1:n}^{0:\infty}}[\cdot]$ denote the expected value operator under the random process $(i_{1:n}^{0:\infty})$. Let also $\epsilon_k = \max_{s,\va}|A_{\boldsymbol{\pi}_k}(s, \va)|$, and $C_k = \frac{4\gamma\epsilon_k}{(1-\gamma)^2}$. We have
\begin{align}
    &0 = \lim_{k\to\infty}\E_{i_{1:n}^{0:\infty}}\left[ J(\boldsymbol{\pi}_{k+1}) -
    J(\boldsymbol{\pi}_k)\right]\nonumber\\
    &\geq \lim_{k\to\infty}\E_{i_{1:n}^{0:\infty}}\left[ 
    L_{\boldsymbol{\pi}_k}(\boldsymbol{\pi}_{k+1}) - C_k \text{{\normalfont D}}_{\text{KL}}^{\text{max}}(\boldsymbol{\pi}_k, \boldsymbol{\pi}_{k+1}) \right]
    \ \ \text{by Theorem \ref{theorem:trpo-ineq}}\nonumber\\
    &\geq \lim_{k\to\infty}\E_{i_{1:n}^{0:\infty}}\left[ 
    L^{i_1^k}_{\boldsymbol{\pi}_k}\left(\pi^{i_1^k}_{k+1}\right) - C_k \text{{\normalfont D}}_{\text{KL}}^{\text{max}}\left(\pi_k^{i_1^k}, \pi_{k+1}^{i_1^k}\right) \right]\nonumber\\
    &\text{by Equation (\ref{eq:sad-trpo-monotonic}) and the fact that each of its summands is non-negative.}\nonumber
\end{align}
Now, we consider an arbitrary limit point $\boldsymbol{\bar{\pi}}$ from the (random) limit set, and a (random) subsequence $\big(\boldsymbol{\pi}_{k_j}\big)_{j=0}^{\infty}$ that converges to $\boldsymbol{\bar{\pi}}$. We get
\begin{align}
    &0\geq \lim_{j\to\infty}\E_{i_{1:n}^{0:\infty}}\left[ 
    L^{i_1^{k_j}}_{\boldsymbol{\pi}_{k_j}}\left(\pi^{i_1^{k_j}}_{k_j+1}\right) - C_{k_j} \text{{\normalfont D}}_{\text{KL}}^{\text{max}}\left(\pi_{k_j}^{i_1^{k_j}}, \pi_{k_j+1}^{i_1^{k_j}}\right) \right].\nonumber
\end{align}
As the expectation is taken of non-negative random variables, and for every $i\in\mathcal{N}$ and $k\in\mathbb{N}$, with some positive probability $p_i$, we have $i_1^{k_j}=i$ (because every permutation has non-zero probability), the above is bounded from below by
\begin{align}
    &p_i\lim_{j\to\infty}\max_{\pi^i}\left[ 
    L^{i}_{\boldsymbol{\pi}_{k_j}}(\pi^i) - C_{k_j} \text{{\normalfont D}}_{\text{KL}}^{\text{max}}\left(\pi_{k_j}^i, \pi^i\right)
    \right],\nonumber\\
    &\text{which, as $\boldsymbol{\pi}_{k_j}$ converges to $\boldsymbol{\bar{\pi}}$, equals to}\nonumber\\
    &p_i\max_{\pi^i} \left[L^{i}_{\boldsymbol{\bar{\pi}}}(\pi^i) - C_{\boldsymbol{\bar{\pi}}} \text{{\normalfont D}}_{\text{KL}}^{\text{max}}\left(\bar{\pi}^i, \pi^i\right)\right] \geq 0, \ \  \text{by Equation (\ref{eq:nice-maad-property}).} \nonumber
\end{align}
This proves that, for any limit point $\boldsymbol{\bar{\pi}}$ of the random process $(\boldsymbol{\pi}_k)$ induced by Algorithm \ref{algorithm:theoretical-matrpo}, $\max_{\pi^i} \left[L^{i}_{\boldsymbol{\bar{\pi}}}(\pi^i) - C_{\boldsymbol{\bar{\pi}}} \text{{\normalfont D}}_{\text{KL}}^{\text{max}}\left(\bar{\pi}^i, \pi^i\right)\right] = 0$, which is equivalent with Definition \ref{def:tr-stationary}.

\paragraph{Step 2 (dropping the penalty term). }
Now, we have to prove that TR-stationary points are NEs of cooperative Markov games. The main step is to prove the following statement: \textsl{a TR-stationary joint policy $\boldsymbol{\bar{\pi}}$, for every state $s\in\mathcal{S}$, satisfies}
\begin{align}
    \label{eq:bold-claim}
    \bar{\pi}^{i} = \argmax_{\pi^i}\E_{\ra^i\sim\pi^i}\big[ A_{\boldsymbol{\bar{\pi}}}^i(s, \ra^i) \big].
\end{align}
We will use the technique of the proof by contradiction. Suppose that there is a state $s_0$ such that there exists a policy $\hat{\pi}^i$ with
\begin{align}
    \label{ineq:assume-to-contradict}
    \E_{\ra^i \sim \hat{\pi}^i}\big[ A_{\boldsymbol{\bar{\pi}}}^i(s_0, \ra^i) \big] > \E_{\ra^i \sim \bar{\pi}^i}\big[ A_{\boldsymbol{\bar{\pi}}}^i(s_0, \ra^i) \big].
\end{align}
Let us parametrise the policies $\pi^i$ according to the template
\begin{align}
  \pi^i(\cdot|s_0) = \big(x^i_1, \dots, x^i_{d^i-1}, 1-\sum\limits_{j=1}^{d^i-1}x^i_j\big)\nonumber
\end{align}
where the values of $x^i_j \ (j=1, \dots, d^i-1)$ are such that $\pi^i(\cdot|s_0)$ is a valid probability distribution.
Then we can rewrite our quantity of interest (the objective of Equation (\ref{eq:bold-claim}) as
\begin{align}
    \E_{\ra^i \sim \pi^i}\big[ A_{\boldsymbol{\bar{\pi}}}^i(s_0, \ra^i) \big] &= \sum_{j=1}^{d^i-1}x^i_j\cdot A_{\boldsymbol{\bar{\pi}}}^i\big(s_0, a^i_j\big) + (1-\sum_{h=1}^{d^i-1}x^i_h)A^i_{\boldsymbol{\bar{\pi}}}\big(s_0, a^i_{d^i}\big)\nonumber\\
    &= \sum_{j=1}^{d^i-1}x^i_j \big[A_{\boldsymbol{\bar{\pi}}}^i\big(s_0, a^i_j\big) - A^i_{\boldsymbol{\bar{\pi}}}\big(s_0, a^i_{d^i}\big)\big] +  A^i_{\boldsymbol{\bar{\pi}}}\big(s_0, a^i_{d^i}\big), \nonumber
\end{align}
which is an affine function of the policy parameterisation. It follows that its gradient (with respect to $x^i$) and directional derivatives are constant in the space of policies at state $s_0$. The existance of policy $\hat{\pi}^i(\cdot|s_0)$, for which Inequality (\ref{ineq:assume-to-contradict}) holds, implies that the directional derivative in the direction from $\bar{\pi}^i(\cdot|s_0)$ to $\hat{\pi}^i(\cdot|s_0)$ is strictly positive. 
We also have
\begin{align}
    \label{eq:partial-kl}
    \frac{\partial \text{{\normalfont D}}_{\text{KL}}(\bar{\pi}^i(\cdot|s_0), \pi^i(\cdot|s_0))}{\partial x^i_j} &= \frac{\partial}{\partial x^i_j}\left[ (\bar{\pi}^{i}(\cdot|s_0))^{T}(\log\bar{\pi}^i(\cdot|s_0) - \log\pi^i(\cdot|s_0)) \right]\nonumber\\
    &= \frac{\partial}{\partial x^i_j}\left[ -(\bar{\pi}^i)^{T}\log\pi^i \right] \ \text{(omitting state }s_0\text{ for brevity)}\nonumber\\
    &= -\frac{\partial}{\partial x^i_j}\sum_{k=1}^{d_i-1}\bar{\pi}^i_k\log x^i_k -
    \frac{\partial}{\partial x^i_j}\bar{\pi}^i_{d_i}\log\left( 1-\sum_{k=1}^{d_i-1}x^i_k\right)\nonumber\\
    &= -\frac{\bar{\pi}^i_j}{x^i_j} + \frac{\bar{\pi}^i_{d_i}}{1-\sum_{k=1}^{d_i-1}x^i_k}\nonumber\\
    &= -\frac{\bar{\pi}^i_j}{\pi^i_j} + \frac{\bar{\pi}^i_{d_i}}{\pi^i_{d_i}} = 0, \ \ \text{when evaluated at }\pi^i=\bar{\pi}^i,
\end{align}
which means that the KL-penalty has zero gradient at $\bar{\pi}^i(\cdot|s_0)$.
Hence, when evaluated at $\pi^i(\cdot|s_0) = \bar{\pi}^i(\cdot|s_0)$, the objective
\begin{align}
    \rho_{\boldsymbol{\bar{\pi}}}(s_0)\E_{\ra^i\sim\pi^i}\big[ A^i_{\boldsymbol{\bar{\pi}}}(s_0, \ra^i) \big] - C_{\boldsymbol{\bar{\pi}}}\text{{\normalfont D}}_{\text{KL}}\big( \bar{\pi}^i(\cdot|s_0), \pi^i(\cdot|s_0) \big) \nonumber
\end{align}
has a strictly positive directional derivative in the direction of $\hat{\pi}^i(\cdot|s_0)$. Thus, there exists a policy $\widetilde{\pi}^i(\cdot|s_0)$, sufficiently close to $\bar{\pi}^i(\cdot|s_0)$ on the path joining it with $\hat{\pi}^i(\cdot|s_0)$, for which
\begin{align}
    \rho_{\boldsymbol{\bar{\pi}}}(s_0)\E_{\ra^i\sim\widetilde{\pi}^i}\big[ A^i_{\boldsymbol{\bar{\pi}}}(s_0, \ra^i) \big] - C_{\boldsymbol{\bar{\pi}}}\text{{\normalfont D}}_{\text{KL}}\big( \bar{\pi}^i(\cdot|s_0), \widetilde{\pi}^i(\cdot|s_0) \big) > 0.\nonumber
\end{align}
Let ${\pi}^i_*$ be a policy such that $\pi^i_*(\cdot|s_0) = \widetilde{\pi}^i(\cdot|s_0)$, and $\pi^i_*(\cdot|s) = \bar{\pi}^i(\cdot|s)$ for states $s\neq s_0$. As for these states we have
\begin{align}
    \rho_{\boldsymbol{\bar{\pi}}}(s)\E_{\ra^i\sim\pi^i_*}\big[ A^i_{\boldsymbol{\bar{\pi}}}(s, \ra^i) \big] = 
    \rho_{\boldsymbol{\bar{\pi}}}(s)\E_{\ra^i\sim\bar{\pi}^i}\big[ A^i_{\boldsymbol{\bar{\pi}}}(s, \ra^i) \big] = 0, \ \
    \text{and} \ \ \text{{\normalfont D}}_{\text{KL}}(\bar{\pi}^i(\cdot|s), \pi^i_*(\cdot|s)) = 0, \nonumber
\end{align}
it follows that
\begin{align}
    L_{\boldsymbol{\bar{\pi}}}^i(\pi^i_*) - C_{\boldsymbol{\bar{\pi}}}\text{{\normalfont D}}_{\text{KL}}^{\text{max}}(\bar{\pi}^i, \pi^i_*) 
    &= \rho_{\boldsymbol{\bar{\pi}}}(s_0)\E_{\ra^i\sim\widetilde{\pi}^i}\big[ A^i_{\boldsymbol{\bar{\pi}}}(s_0, \ra^i) \big] - C_{\boldsymbol{\bar{\pi}}}\text{{\normalfont D}}_{\text{KL}}\big( \bar{\pi}^i(\cdot|s_0), \widetilde{\pi}^i(\cdot|s_0) \big) \nonumber\\
    &> 0 = L_{\boldsymbol{\bar{\pi}}}^i(\bar{\pi}^i) - C_{\boldsymbol{\bar{\pi}}}\text{{\normalfont D}}_{\text{KL}}^{\text{max}}(\bar{\pi}^i, \bar{\pi}^i), \nonumber
\end{align}
which is a contradiction with TR-stationarity of $\boldsymbol{\bar{\pi}}$. Hence, the claim of Equation (\ref{eq:bold-claim}) is proved. 
\paragraph{Step 3 (optimality). }
Now, for a fixed joint policy $\boldsymbol{\bar{\pi}}^{-i}$ of other agents, $\bar{\pi}^i$ satisfies
\begin{align}
    \bar{\pi}^{i} = \argmax_{\pi^i}\E_{\ra^i\sim\pi^i}\big[ A^i_{\boldsymbol{\bar{\pi}}}(s, \ra^i) \big] = \argmax_{\pi^i}\E_{\ra^i\sim\pi^i}\big[ Q^i_{\boldsymbol{\bar{\pi}}}(s, \ra^i) \big], \ \forall s\in\mathcal{S}, \nonumber
\end{align}
which is the Bellman optimality equation \citep{sutton2018reinforcement}. Hence, for a fixed joint policy $\boldsymbol{\bar{\pi}}^{-i}$, the policy $\bar{\pi}^i$ is optimal:
\begin{align}
    \bar{\pi}^i = \argmax_{\pi^i}J(\pi^i, \boldsymbol{\bar{\pi}}^{-i}). \nonumber
\end{align}
As agent $i$ was chosen arbitrarily, $\boldsymbol{\bar{\pi}}$ is a Nash equilibrium.
\end{proof}

\section{HATRPO and HAPPO}
\label{appendix:matrpo-mappo}

\subsection{Proof of Proposition \ref{lemma:advantage-estimation}}
\label{appendix:proof-of-advantage-estimation}
\advantageestimation*
\begin{proof}
    \begin{align}
    &\E_{\rva\sim\boldsymbol{\pi}}\Big[
        \Big(\frac{\hat{\pi}^{i_m}(\ra^{i_m}|s) }{ \pi^{i_m}(\ra^{i_m}|s)} 
        -1\Big)  \frac{\boldsymbol{\bar{\pi}}^{i_{1:m-1}}(\rva^{i_{1:m-1}}|s)}
        {\boldsymbol{\pi}^{i_{1:m-1}}(\rva^{i_{1:m-1}}|s)}A_{\boldsymbol{\pi}}(s, \rva)  \Big]\nonumber\\
        &= 
        \E_{\rva\sim\boldsymbol{\pi}}\left[  
         \frac{\hat{\pi}^{i_m}(\ra^{i_m}|s)\boldsymbol{\bar{\pi}}^{i_{1:m-1}}(\rva^{i_{1:m-1}}|s)}
        {\boldsymbol{\pi}^{i_{1:m}}(\rva^{i_{1:m}}|s)}A_{\boldsymbol{\pi}}(s, \rva) 
        - \frac{\boldsymbol{\bar{\pi}}^{i_{1:m-1}}(\rva^{i_{1:m-1}}|s)}
        {\boldsymbol{\pi}^{i_{1:m-1}}(\rva^{i_{1:m-1}}|s)}A_{\boldsymbol{\pi}}(s, \rva) 
        \right]\nonumber\\
        &=  \E_{\rva^{i_{1:m}}\sim\boldsymbol{\pi}^{i_{1:m}},
        \rva^{-i_{1:m}}\sim\boldsymbol{\pi}^{-i_{1:m}}
        }\left[  
         \frac{\hat{\pi}^{i_m}(\ra^{i_m}|s)\boldsymbol{\bar{\pi}}^{i_{1:m-1}}(\rva^{i_{1:m-1}}|s)}
        {\boldsymbol{\pi}^{i_{1:m}}(\rva^{i_{1:m}}|s)}A_{\boldsymbol{\pi}}(s, \rva^{i_{1:m}}, \rva^{-i_{1:m}}) 
        \right] \nonumber\\
        & \ \ \ -\E_{\rva^{i_{1:m-1}}\sim\boldsymbol{\pi}^{i_{1:m-1}},
        \rva^{-i_{1:m-1}}\sim\boldsymbol{\pi}^{-i_{1:m-1}}
        }\left[  \frac{\boldsymbol{\bar{\pi}}^{i_{1:m-1}}(\rva^{i_{1:m-1}}|s)}
        {\boldsymbol{\pi}^{i_{1:m-1}}(\rva^{i_{1:m-1}}|s)}A_{\boldsymbol{\pi}}(s, 
        \rva^{i_{1:m-1}}, \rva^{-i_{1:m-1}}) 
        \right]\nonumber\\
        &=  \E_{\rva^{i_{1:m-1}}\sim\boldsymbol{\bar{\pi}}^{i_{1:m-1}},
        \ra^{i_m}\sim\hat{\pi}^{i_m},
        \rva^{-i_{1:m}}\sim\boldsymbol{\pi}^{-i_{1:m}}
        }\left[  
        A_{\boldsymbol{\pi}}(s, \rva^{i_{1:m}}, \rva^{-i_{1:m}}) 
        \right] \nonumber\\
        & \ \ \ -\E_{\rva^{i_{1:m-1}}\sim\boldsymbol{\bar{\pi}}^{i_{1:m-1}},
        \rva^{-i_{1:m-1}}\sim\boldsymbol{\pi}^{-i_{1:m-1}}
        }\left[ A_{\boldsymbol{\pi}}(s, 
        \rva^{i_{1:m-1}}, \rva^{-i_{1:m-1}}) 
        \right]\nonumber\\
        &=  \E_{\rva^{i_{1:m-1}}\sim\boldsymbol{\bar{\pi}}^{i_{1:m-1}},
        \ra^{i_m}\sim\hat{\pi}^{i_m}
        }\left[  
        \E_{\rva^{-i_{1:m}}\sim\boldsymbol{\pi}^{-i_{1:m}}}\left[
        A_{\boldsymbol{\pi}}(s, \rva^{i_{1:m}}, \rva^{-i_{1:m}}) 
        \right] \right]\nonumber\\
        & \ \ \ -\E_{\rva^{i_{1:m-1}}\sim\boldsymbol{\bar{\pi}}^{i_{1:m-1}} }
        \left[ \E_{\rva^{-i_{1:m-1}}\sim\boldsymbol{\pi}^{-i_{1:m-1}}}
        \left[A_{\boldsymbol{\pi}}(s, 
        \rva^{i_{1:m-1}}, \rva^{-i_{1:m-1}}) 
        \right] \right]\nonumber\\
        &=  \E_{\rva^{i_{1:m-1}}\sim\boldsymbol{\bar{\pi}}^{i_{1:m-1}},
        \ra^{i_m}\sim\hat{\pi}^{i_m}
        }\left[  
        A_{\boldsymbol{\pi}}^{i_{1:m}}(s, \rva^{i_{1:m}}) 
        \right]\nonumber\\
        & \ \ \ -\E_{\rva^{i_{1:m-1}}\sim\boldsymbol{\bar{\pi}}^{i_{1:m-1}} }
        \left[ A_{\boldsymbol{\pi}}^{i_{1:m-1}}(s, 
        \rva^{i_{1:m-1}}) 
        \right], \nonumber\\
        &= \E_{\rva^{i_{1:m-1}}\sim\boldsymbol{\bar{\pi}}^{i_{1:m-1}},
        \ra^{i_m}\sim\hat{\pi}^{i_m}
        }\left[ A^{i_{1:m}}_{\boldsymbol{\pi}}(s, \rva^{i_{1:m}}) 
        - A^{i_{1:m-1}}_{\boldsymbol{\pi}}(s, \rva^{i_{1:m-1}}) \right]\nonumber\\
        &\text{which, by Lemma \ref{lemma:maadlemma}, equals}\nonumber\\
        &= \E_{\rva^{i_{1:m-1}}\sim\boldsymbol{\bar{\pi}}^{i_{1:m-1}},
        \ra^{i_m}\sim\hat{\pi}^{i_m}
        }\left[ A^{i_{m}}_{\boldsymbol{\pi}}(s, \rva^{i_{1:m-1}}, \ra^{i_{m}}) \right].\nonumber
    \end{align}
\end{proof}

\subsection{Derivation of the gradient estimator for HATRPO}
\label{appendix:derive-grad}
\begin{align}
    &\nabla_{\theta^{i_m}}\E_{\rs\sim\rho_{\boldsymbol{\pi}_{\vtheta_k}}, \rva\sim\boldsymbol{\pi}_{\vtheta_k}}\Bigg[ \Bigg(\frac{\pi^{i_m}_{\theta^{i_m}}(\ra^{i_m}|\rs)}{\pi^{i_m}_{\theta^{i_m}_k}(\ra^{i_m}|\rs)}-1\Bigg)M^{i_{1:m}}(\rs, \rva)\Bigg]\nonumber\\
    &= \nabla_{\theta^{i_m}}\E_{\rs\sim\rho_{\boldsymbol{\pi}_{\vtheta_k}}, \rva\sim\boldsymbol{\pi}_{\vtheta_k}}\Bigg[ \frac{\pi^{i_m}_{\theta^{i_m}}(\ra^{i_m}|\rs)}{\pi^{i_m}_{\theta^{i_m}_k}(\ra^{i_m}|\rs)}M^{i_{1:m}}(\rs, \rva)\Bigg]
    - 
    \nabla_{\theta^{i_m}}\E_{\rs\sim\rho_{\boldsymbol{\pi}_{\vtheta_k}}, \rva\sim\boldsymbol{\pi}_{\vtheta_k}}\Bigg[ M^{i_{1:m}}(\rs, \rva)\Bigg]\nonumber\\
    &= \E_{\rs\sim\rho_{\boldsymbol{\pi}_{\vtheta_k}}, \rva\sim\boldsymbol{\pi}_{\vtheta_k}}\Bigg[ \frac{\nabla_{\theta^{i_m}}\pi^{i_m}_{\theta^{i_m}}(\ra^{i_m}|\rs)}{\pi^{i_m}_{\theta^{i_m}_k}(\ra^{i_m}|\rs)}M^{i_{1:m}}(\rs, \rva)\Bigg]\nonumber\\
    &= \E_{\rs\sim\rho_{\boldsymbol{\pi}_{\vtheta_k}}, \rva\sim\boldsymbol{\pi}_{\vtheta_k}}\Bigg[ \frac{\pi^{i_m}_{\theta^{i_m}}(\ra^{i_m}|\rs)}{\pi^{i_m}_{\theta^{i_m}_k}(\ra^{i_m}|\rs)}\nabla_{\theta^{i_m}}\log\pi^{i_m}_{\theta^{i_m}}(\ra^{i_m}|\rs)M^{i_{1:m}}(\rs, \rva)\Bigg]. \nonumber
\end{align}
Evaluated at $\theta^{i_m}= \theta^{i_m}_k$, the above expression equals
\begin{align}
    \E_{\rs\sim\rho_{\boldsymbol{\pi}_{\vtheta_k}}, \rva\sim\boldsymbol{\pi}_{\vtheta_k}}\Big[ M^{i_{1:m}}(\rs, \rva)\nabla_{\theta^{i_m}}\log\pi^{i_m}_{\theta^{i_m}}(\ra^{i_m}|\rs)\big|_{\theta^{i_m}=\theta^{i_m}_k}\Big], \nonumber
\end{align}
which finishes the derivation.
\subsection{Pseudocode of HATRPO}
\label{appendix:matrpo}
\begin{algorithm}[!htbp]
    \caption{HATRPO}
    \label{MATRPO}
    
    \textbf{Input:} Stepsize $\alpha$, batch size $B$, number of: agents $n$, episodes $K$, steps per episode $T$, possible steps in line search $L$, line search acceptance threshold $\kappa$.

    \textbf{Initialize:} Actor networks $\{\theta^{i}_{0}, \ \forall i\in \mathcal{N}\}$, Global V-value network $\{\phi_{0}\}$, Replay buffer $\mathcal{B}$
    
    \For{$k = 0,1,\dots,K-1$}{ 
        Collect a set of trajectories by running the joint policy $\boldsymbol{\pi}_{\vtheta_{k}} =(\pi^{1}_{\theta^{1}_{k}}, \dots,  \pi^{n}_{\theta^{n}_{k}})$.\\
        Push transitions $\{(s_{t}, o^i_{t},a^i_{t},r_t,s_{t+1},o^i_{t+1}), \forall i\in \mathcal{N},t\in T\}$ into $\mathcal{B}$.\\
        Sample a random minibatch of $B$ transitions from $\mathcal{B}$.\\
        Compute advantage function {\small $\hat{A}(\rs, \rva)$} based on global V-value network with GAE.\\
        Draw a random permutation of agents $i_{1:n}$.\\
        Set $M^{i_{1}}(\rs, \rva) = \hat{A}(\rs, \rva)$.
        
        \For{agent $i_{m} = i_{1}, \dots, i_{n}$}{
            Estimate the gradient of the agent's maximisation objective
            \begin{center}
                $\hat{\vg}^{i_{m}}_{k} =  \frac{1}{B}\sum\limits^{B}_{b=1}\sum\limits^{T}_{t=1}\nabla_{\theta^{i_{m}}_{k} }\log\pi^{i_{m}}_{\theta^{i_{m}}_k}\left( a^{i_{m}}_{t}\mid o_{t}^{i_{m}} \right)M^{i_{1:m}}(s_{t}, \va_{t})$.
            \end{center}
            
            Use the conjugate gradient algorithm to compute the update direction
            \begin{center}
                $\hat{\vx}^{i_m}_k \approx (\hat{\mH}^{i_m }_k)^{-1} \hat{\vg}^{i_m}_k$,
            \end{center}
             where $\hat{\mH}^{i_m}_k$ is the Hessian of the average KL-divergence
            \begin{center}
                $\frac{1}{BT}\sum\limits_{b=1}^{B}\sum\limits_{t=1}^{T}\text{{\normalfont D}}_{\text{KL}}\left(\pi^{i_m}_{\theta^{i_m}_k}(\cdot|o^{i_m}_t), \pi^{i_m}_{\theta^{i_m}}(\cdot|o^{i_m}_t)\right)$.
            \end{center}
            
            Estimate the maximal step size allowing for meeting the KL-constraint
            \begin{center}
                $\hat{\beta}^{i_m}_k \approx \sqrt{\dfrac{2\delta}{(\hat{\vx}^{i_m}_k)^{T}\hat{\mH}^{i_m }_k \hat{\vx}^{i_m}_k}}$.
            \end{center}
            
            Update agent $i_m$'s policy by
            \begin{center}
                $ \theta^{i_m}_{k+1} = \theta^{i_m}_k  + \alpha^j \hat{\beta}^{i_m}_k \hat{\vx}^{i_m}_k$,
            \end{center}
            where $j\in\{0, 1, \dots, L\}$ is the smallest such $j$ which improves the sample loss by at least $\kappa \alpha^j \hat{\beta}^{i_m}_k \hat{\vx}^{i_m}_k \cdot \hat{\vg}^{i_m}_k$, found by the backtracking line search.
            
            Compute $M^{i_{1:m+1}}(\rs, \rva) = \frac{\pi^{i_{m}}_{\theta^{i_{m}}_{k+1}}\left(\ra^{i_{m}} \mid \ro^{i_{m}}\right)}{\pi^{i_{m}}_{\theta^{i_{m}}_{k}}\left(\ra^{i_{m}} \mid \ro^{i_{m}}\right)} M^{i_{1:m}}(\rs_{t}, \rva_{t})$.  \text{ \color{RoyalPurple} //Unless $m=n$.}
        }
        
        Update V-value network by following formula:
        \begin{center}$\phi_{k+1}=\arg \min _{\phi} \frac{1}{BT} \sum\limits^{B}_{b =1 } \sum\limits_{t=0}^{T}\left(V_{\phi}(s_{t}) - \hat{R_{t}}\right)^{2}$
        \end{center}
    }
\end{algorithm}

\subsection{Pseudocode of HAPPO}
\label{appendix:mappo}

\begin{algorithm}[!htbp]
    \caption{HAPPO}
    \label{MAPPO}
    
    \textbf{Input:} Stepsize $\alpha$, batch size $B$, number of: agents $n$, episodes $K$, steps per episode $T$.
    
    \textbf{Initialize:} Actor networks $\{\theta^{i}_{0}, \ \forall i\in \mathcal{N}\}$, Global V-value network $\{\phi_{0}\}$, Replay buffer $\mathcal{B}$
    
    \For{$k = 0,1,\dots,K-1$}{
        Collect a set of trajectories by running the joint policy $\boldsymbol{\pi}_{\vtheta_{k}} =(\pi^{1}_{\theta^{1}_{k}}, \dots,  \pi^{n}_{\theta^{n}_{k}})$.
        
        Push transitions $\{(s_t, o^i_{t},a^i_{t},r_t, s_{t+1}, o^i_{t+1}), \forall i\in \mathcal{N},t\in T\}$ into $\mathcal{B}$.
        
        Sample a random minibatch of $B$ transitions from $\mathcal{B}$.
        
        Compute advantage function $\hat{A}(\rs, \rva)$ based on global V-value network with GAE.
        
        Draw a random permutation of agents $i_{1:n}$.
        
        Set $M^{i_{1}}(\rs, \rva) = \hat{A}(\rs, \rva)$.
        
        \For{agent $i_{m} = i_{1}, \dots, i_{n}$}{
            Update actor $i_{m}$ with $\theta^{i_{m}}_{k+1}$, the argmax of the PPO-Clip objective
                $ \frac{1}{BT} \sum\limits^{B}_{b =1} \sum\limits_{t=0}^{T} \min \left( \frac{\pi^{i_{m}}_{\theta^{i_{m}}}\left(a^{i_{m}}_{t} \mid o^{i_{m}}_{t}\right)}{\pi^{i_{m}}_{\theta^{i_{m}}_{k}}\left(a^{i_{m}}_{t} \mid o^{i_{m}}_{t}\right)}  M^{i_{1:m}}(s_{t}, \va_{t}), \ 
                \text{clip}\bigg( \frac{\pi^{i_{m}}_{\theta^{i_{m}}}\left(a^{i_{m}}_{t} \mid o^{i_{m}}_{t}\right)}{\pi^{i_{m}}_{\theta^{i_{m}}_{k}}\left(a^{i_{m}}_{t} \mid o^{i_{m}}_{t}\right)}, 1\pm\epsilon\bigg) M^{i_{1:m}}(s_{t}, \va_{t})\right)$.\\
            Compute $M^{i_{1:m+1}}(\rs, \rva) = \frac{\pi^{i_{m}}_{\theta^{i_{m}}_{k+1}}\left(\ra^{i_{m}} \mid \ro^{i_{m}}\right)}{\pi^{i_{m}}_{\theta^{i_{m}}_{k}}\left(\ra^{i_{m}} \mid \ro^{i_{m}}\right)} M^{i_{1:m}}(\rs, \rva)$.  \text{ \color{RoyalPurple} //Unless $m=n$.}
        }
        
        Update V-value network by the following formula:
        \begin{center}
            $\phi_{k+1}=\arg \min _{\phi} \frac{1}{BT} \sum\limits^{B}_{b =1 } \sum\limits_{t=0}^{T}\left(V_{\phi}(s_{t}) - \hat{R_{t}}\right)^{2}$
        \end{center}
    }
\end{algorithm}


\section{Proof of \textsl{HAMO Is All You Need} Lemma}
\dpi*
\label{apx:proof_hamo}
\begin{proof}
    Let $\widetilde{\mathfrak{D}}_{\vpi_{\text{old}}}(\vpi_{\text{new}}|s) \triangleq \sum_{m=1}^{n}\mathfrak{D}^{i_m}_{\vpi_{\text{old}}}(\pi^{i_m}_{\text{new}}|s, \vpi^{i_{1:m-1}}_{\text{new}})$. Combining this with Lemma \ref{lemma:maadlemma} gives
    \begin{align}
        &\E_{\rva\sim\vpi_{\text{new}}}\big[ A_{\vpi_{\text{old}}}(s, \rva) \big] - \widetilde{\mathfrak{D}}_{\vpi_{\text{old}}}(\vpi_{\text{new}}|s) \nonumber\\
        &= \sum_{m=1}^{n}\big[ \E_{\rva^{i_{1:m-1}}\sim\vpi^{i_{1:m-1}}_{\text{new}}, \ra^{i_m}\sim\pi^{i_m}_{\text{new}}}\big[A_{\vpi_{\text{old}}}^{i_m}(s, \rva^{i_{1:m-1}}, \ra^{i_m})\big] - \mathfrak{D}^{i_m}_{\vpi_{\text{old}}}(\pi^{i_m}_{\text{new}}|s, \vpi^{i_{1:m-1}}_{\text{new}})\big]\nonumber\\
        &\quad \quad \quad \text{by Inequality (\ref{ineq:what-mirror-step-does})}\nonumber\\
        &\geq \sum_{m=1}^{n}\big[ \E_{\rva^{i_{1:m-1}}\sim\vpi^{i_{1:m-1}}_{\text{new}}, \ra^{i_m}\sim\pi^{i_m}_{\text{old}}}\big[A_{\vpi_{\text{old}}}^{i_m}(s, \rva^{i_{1:m-1}}, \ra^{i_m})\big] - \mathfrak{D}^{i_m}_{\vpi_{\text{old}}}(\pi^{i_m}_{\text{old}}|s, \vpi^{i_{1:m-1}}_{\text{new}})\big]\nonumber\\
        &=\E_{\rva\sim\vpi_{\text{old}}}\big[ A_{\vpi_{\text{old}}}(s, \rva) \big] - \widetilde{\mathfrak{D}}_{\vpi_{\text{old}}}(\vpi_{\text{old}}|s). \nonumber
    \end{align}
    The resulting inequality can be equivalently rewritten as 
    \begin{align}
        \label{ineq:equiv}
        \E_{\rva\sim\vpi_{\text{new}}}\big[ Q_{\vpi_{\text{old}}}(s, \rva) \big] - \widetilde{\mathfrak{D}}_{\vpi_{\text{old}}}(\vpi_{\text{new}}|s) \geq 
        \E_{\rva\sim\vpi_{\text{old}}}\big[ Q_{\vpi_{\text{old}}}(s, \rva) \big] - \widetilde{\mathfrak{D}}_{\vpi_{\text{old}}}(\vpi_{\text{old}}|s), \forall s\in\mathcal{S}.
    \end{align}
    We use it to prove the claim as follows,
    \begin{align}
        V_{\vpi_{\text{new}}}(s) &= \E_{\rva\sim\vpi_{\text{new}}}\big[ Q_{\vpi_{\text{new}}}(s, \rva) \big]
        \nonumber\\
        &= \E_{\rva\sim\vpi_{\text{new}}}\big[ Q_{\vpi_{\text{old}}}(s, \rva) \big] - \widetilde{\mathfrak{D}}_{\vpi_{\text{old}}}(\vpi_{\text{new}}|s) \nonumber\\
        &\quad + \widetilde{\mathfrak{D}}_{\vpi_{\text{old}}}(\vpi_{\text{new}}|s) + \E_{\rva\sim\vpi_{\text{new}}}\big[ Q_{\vpi_{\text{new}}}(s, \rva) -  Q_{\vpi_{\text{old}}}(s, \rva) \big],\nonumber\\
        &\text{by Inequality (\ref{ineq:equiv})} \nonumber\\
        & \geq 
        \E_{\rva\sim\vpi_{\text{old}}}\big[ Q_{\vpi_{\text{old}}}(s, \rva) \big] - \widetilde{\mathfrak{D}}_{\vpi_{\text{old}}}(\vpi_{\text{old}}|s) \nonumber\\
        &\quad + \widetilde{\mathfrak{D}}_{\vpi_{\text{old}}}(\vpi_{\text{new}}|s) + \E_{\rva\sim\vpi_{\text{new}}}\big[ Q_{\vpi_{\text{new}}}(s, \rva) - Q_{\vpi_{\text{old}}}(s, \rva)\big],\nonumber\\
        &= V_{\pi_{\text{old}}}(s) +  \widetilde{\mathfrak{D}}_{\vpi_{\text{old}}}(\vpi_{\text{new}}|s) + \E_{\rva\sim\vpi_{\text{new}}}\big[ Q_{\vpi_{\text{new}}}(s, \rva) - Q_{\vpi_{\text{old}}}(s, \rva)\big]\nonumber\\
        &= V_{\pi_{\text{old}}}(s) +  \widetilde{\mathfrak{D}}_{\vpi_{\text{old}}}(\vpi_{\text{new}}|s) + \E_{\rva\sim\vpi_{\text{new}}, \rs'\sim P}\big[ r(s, \rva) + \gamma V_{\vpi_{\text{new}}}(\rs') - r(s, \rva) - \gamma V_{\vpi_{\text{old}}}(\rs')\big]\nonumber\\
        &= V_{\pi_{\text{old}}}(s) +  \widetilde{\mathfrak{D}}_{\vpi_{\text{old}}}(\vpi_{\text{new}}|s) + \gamma\E_{\rva\sim\vpi_{\text{new}}, \rs'\sim P}\big[ V_{\vpi_{\text{new}}}(\rs') -  V_{\vpi_{\text{old}}}(\rs')\big]\nonumber\\
        &\geq V_{\pi_{\text{old}}}(s) + \gamma \inf_{s'}\big[ V_{\vpi_{\text{new}}}(s') -  V_{\vpi_{\text{old}}}(s') \big].
        \nonumber\\
        &\text{Hence} \quad V_{\vpi_{\text{new}}}(s) - V_{\pi_{\text{old}}}(s) \geq \gamma \inf_{s'}\big[ V_{\vpi_{\text{new}}}(s') -  V_{\vpi_{\text{old}}}(s') \big].\nonumber\\
        &\text{Taking infimum over }s\text{ and simplifying}\nonumber\\
        & (1-\gamma)\inf_{s}\big[ V_{\vpi_{\text{new}}}(s) -  V_{\vpi_{\text{old}}}(s) \big] \geq 0.\nonumber
    \end{align}
    Therefore, $\inf_{s}\big[ V_{\vpi_{\text{new}}}(s) -  V_{\vpi_{\text{old}}}(s) \big] \geq 0$, which proves the lemma.
\end{proof}

\clearpage

\section{Proof of Theorem \ref{theorem:fundamental}}
\label{apx:proof_t1}
\begin{restatable}{lemma}{hamoimprovement}
\label{lemma:hamoimp}
Suppose an agent $i_m$ maximises the expected HAMO
\begin{align}
    \pi^{i_m}_{\text{new}}= \argmax\limits_{\pi^{i_m}\in\mathcal{U}^{i_m}_{\vpi_{\text{old}}}(\pi^{i_m}_{\text{old}})}\E_{\rs\sim\beta_{\vpi_{\text{old}}}}\Big[ \big[\mathcal{M}^{(\pi^{i_m})}_{\mathfrak{D}^{i_m}, \vpi^{i_{1:m-1}}_{\text{new}} } A_{\vpi_{\text{old}}}\big](\rs)\Big].
\end{align}
Then, for every state $s\in\mathcal{S}$
\begin{align}
    \big[\mathcal{M}^{(\pi^{i_m}_{\text{new}})}_{\mathfrak{D}^{i_m}, \vpi^{i_{1:m-1}}_{\text{new}} } A_{\vpi_{\text{old}}}\big](s)
    \geq
    \big[\mathcal{M}^{(\pi^{i_m}_{\text{old}})}_{\mathfrak{D}^{i_m}, \vpi^{i_{1:m-1}}_{\text{new}} } A_{\vpi_{\text{old}}}\big](s).\nonumber
\end{align}
\end{restatable}

\begin{proof}
We will prove this statement by contradiction. Suppose that there exists $s_0 \in \mathcal{S}$ such that 
\begin{align}
    \label{ineq:suppose}
    \big[\mathcal{M}^{(\pi^{i_m}_{\text{new}})}_{\mathfrak{D}^{i_m}, \vpi^{i_{1:m-1}}_{\text{new}} } A_{\vpi_{\text{old}}}\big](s_0)
    <
    \big[\mathcal{M}^{(\pi^{i_m}_{\text{old}})}_{\mathfrak{D}^{i_m}, \vpi^{i_{1:m-1}}_{\text{new}} } A_{\vpi_{\text{old}}}\big](s_0). 
\end{align}
Let us define the following policy $\hat{\pi}^{i_m}$.
\begin{align}
    \hat{\pi}^{i_m}(\cdot^{i_m}|s) = \begin{cases}
    \pi^{i_m}_{\text{old}}(\cdot^{i_m}|s), \ \text{at} \ s=s_0\nonumber\\
    \pi^{i_m}_{\text{new}}(\cdot^{i_m}|s), \ \text{at} \ s\neq s_0
    \end{cases}
\end{align}
Note that $\hat{\pi}^{i_m}$ is (weakly) closer to $\pi^{i_m}_{\text{old}}$ than $\pi^{i_m}_{\text{new}}$ at $s_0$, and at the same distance at other states. Together with $\pi^{i_m}_{\text{new}}\in\mathcal{U}^{i_m}_{\vpi_{\text{old}}}(\pi^{i_m}_{\text{old}})$, 
this implies that $\hat{\pi}^{i_m}\in\mathcal{U}^{i_m}_{\vpi_{\text{old}}}(\pi^{i_m}_{\text{old}})$. Further, 
\begin{align}
    \E_{\rs\sim\beta_{\vpi_{\text{old}}}}\Big[ \big[\mathcal{M}^{(\hat{\pi}^{i_m})}_{\mathfrak{D}^{i_m}, \vpi^{i_{1:m-1}}_{\text{new}} } A_{\vpi_{\text{old}}}\big](\rs)\Big]
    -
    \E_{\rs\sim\beta_{\vpi_{\text{old}}}}\Big[ \big[\mathcal{M}^{(\pi^{i_m}_{\text{new}})}_{\mathfrak{D}^{i_m}, \vpi^{i_{1:m-1}}_{\text{new}} } A_{\vpi_{\text{old}}}\big](\rs)\Big]\nonumber\\
    =\beta_{\vpi_{\text{old}}}(s_0)\big( 
    \big[\mathcal{M}^{(\pi^{i_m}_{\text{old}})}_{\mathfrak{D}^{i_m}, \vpi^{i_{1:m-1}}_{\text{new}} } A_{\vpi_{\text{old}}}\big](s_0)
    -
    \big[\mathcal{M}^{(\pi^{i_m}_{\text{new}})}_{\mathfrak{D}^{i_m}, \vpi^{i_{1:m-1}}_{\text{new}} } A_{\vpi_{\text{old}}}\big](s_0)
    \big) > 0.\nonumber
\end{align}
The above contradicts $\pi^{i_m}_{\text{new}}$ as being the argmax of Inequality (\ref{ineq:suppose}), as $\hat{\pi}^{i_m}$ is strictly better. The contradiction finishes the proof.
\end{proof}

\mamlfundamental*

\clearpage
\begin{proof}
\paragraph{Proof of Property \ref{property1}.}

It follows from combining Lemmas \ref{lemma:hamo} \&  \ref{lemma:hamoimp}.

\paragraph{Proof of Properties \ref{property2}, \ref{property3} \&  \ref{property4}.}

\paragraph{Step 1: convergence of the value function.} By Lemma \ref{lemma:hamo}, we have that $V_{\vpi_{k}}(s)\leq V_{\vpi_{k+1}}(s), \ \forall s\in\mathcal{S}$, and that the value function is upper-bounded by $V_{\max}$. Hence, the sequence of value functions $(V_{\vpi_k})_{k\in\mathbb{N}}$ converges. We denote its limit by $V$.

\paragraph{Step 2: characterisation of limit points.} As the joint policy space $\boldsymbol{\Pi}$ is bounded, by Bolzano-Weierstrass theorem, we know that the sequence $(\vpi_{k})_{k\in\mathbb{N}}$ has a convergent subsequence. Therefore, it has at least one limit point policy. Let $\vbarpi$ be such a limit point. 
We introduce an auxiliary notation: for a joint policy $\vpi$ and a permutation $i_{1:n}$, let $\text{HU}(\vpi, i_{1:n})$ be a joint policy obtained by a HAML update from $\vpi$ along the permutation $i_{1:n}$.

\paragraph{Claim:} For any permutation $z_{1:n}\in\text{\normalfont Sym}(n)$,
\begin{align}
    \label{eq:claim}
    \vbarpi=\text{HU}(\vbarpi, z_{1:n}).
\end{align}

\paragraph{Proof of Claim.}
Let $\hat{\vpi}=\text{HU}(\vbarpi, z_{1:n})\neq \vbarpi$ and
$(\vpi_{k_r})_{r\in\mathbb{N}}$ be a subsequence converging to $\vbarpi$. Let us recall that the limit value function is unique and denoted as $V$. Writing $\E_{i^{0:\infty}_{1:n}}[\cdot]$ for the expectation operator under the stochastic process $(i^{k}_{1:n})_{k\in\mathbb{N}}$ of update orders, for a state $s\in\mathcal{S}$, we have
\begin{align}
    0 &= \lim_{r\rightarrow \infty}\E_{i^{0:\infty}_{1:n}}\big[ V_{\vpi_{k_r +1}}(s) - V_{\vpi_{k_r}}(s) \big] \nonumber\\
    &\text{as every choice of permutation improves the value function}\nonumber\\
    &\geq \lim_{r\rightarrow \infty} \text{P}(i^{k_r}_{1:n} = z_{1:n})\big[ V_{\text{HU}(\vpi_{k_r},z_{1:n})}(s) - V_{\vpi_{k_r}}(s) \big] \nonumber\\
    &= p(z_{1:n})\lim_{r\rightarrow \infty} \big[ V_{\text{HU}(\vpi_{k_r},z_{1:n})}(s) - V_{\vpi_{k_r}}(s) \big]. \nonumber
\end{align}
By the continuity of the expected HAMO 
, we obtain that the first component of $\text{HU}(\vpi_{k_r}, z_{1:n})$, which is $\pi^{z_1}_{k_r +1}$, is continuous in $\vpi_{k_r}$ by Berge's Maximum Theorem \citep{ausubel1993generalized}. Applying this argument recursively for $z_2, \dots, z_n$, we have that $\text{HU}(\vpi_{k_r}, z_{1:n})$ is continuous in $\vpi_{k_r}$. Hence, as $\vpi_{k_r}$ converges to $\vbarpi$, its HU converges to the HU of $\vbarpi$, which is $\hat{\vpi}$. Hence, we continue writing the above derivation as 
\begin{align}
    &= p(z_{1:n})\big[ V_{\hat{\vpi}}(s) - V_{\vbarpi}(s)\big]\geq 0, \ \text{by Lemma \ref{lemma:hamo}}.\nonumber
\end{align}
As $s$ was arbitrary, the state-value function of $\hat{\vpi}$ is the same as that of $\vpi$: $V_{\hat{\vpi}}=V_{\vpi}$, by the Bellman equation \citep{sutton2018reinforcement}: $Q(s, \va) = r(s, \va) + \gamma \E V(\rs')$, this also implies that their state-value and advantage functions are the same: $Q_{\hat{\vpi}}=Q_{\vbarpi}$ and $A_{\hat{\vpi}}=A_{\vbarpi}$. Let $m$ be the smallest integer such that $\hat{\pi}^{z_m}\neq \bar{\pi}^{z_m}$. This means that $\hat{\pi}^{z_m}$ achieves a greater expected HAMO than $\bar{\pi}^{z_m}$, for which it is zero. Hence,
\begin{align}
    &0 < \E_{\rs\sim\beta_{\vpi}}\Big[ \big[ \mathcal{M}^{(\hat{\pi}^{z_m})}_{\mathfrak{D}^{z_m}, \vbarpi^{z_{1:m-1}}} A_{\vbarpi}\big](s) \Big]\nonumber\\
    &= \E_{\rs\sim\beta_{\vpi}}\Big[ \E_{\rva^{z_{1:m}}\sim\vbarpi^{z_{1:m-1}}, \ra^{z_m}\sim\hat{\pi}^{z_m}}\big[ A^{z_m}_{\vbarpi}(s, \rva^{z_{1:m-1}}, \ra^{z_m}) \big] - \mathfrak{D}^{z_m}_{\vpi}(\hat{\pi}^{z_m}|s, \vbarpi^{z_{1:m-1}}) \Big]\nonumber\\
    &= \E_{\rs\sim\beta_{\vpi}}\Big[ \E_{\rva^{z_{1:m}}\sim\vbarpi^{z_{1:m-1}}, \ra^{z_m}\sim\hat{\pi}^{z_m}}\big[ A^{z_m}_{\hat{\vpi}}(s, \rva^{z_{1:m-1}}, \ra^{z_m}) \big] - \mathfrak{D}^{z_m}_{\vpi}(\hat{\pi}^{z_m}|s, \vbarpi^{z_{1:m-1}}) \Big]\nonumber\\
    &\text{and as the expected value of the multi-agent advantage function is zero}\nonumber\\
    &= \E_{\rs\sim\beta_{\vpi}}\Big[ - \mathfrak{D}^{z_m}_{\vpi}(\hat{\pi}^{z_m}|s, \vbarpi^{z_{1:m-1}}) \Big]\leq 0.\nonumber
\end{align}
This is a contradiction, and so the claim in Equation (\ref{eq:claim}) is proved, and the \text{Step 2} is finished.

\paragraph{Step 3: dropping the HADF.} Consider an arbitrary limit point joint policy $\vbarpi$. By Step 2, for any permutation $i_{1:n}$, considering the first component of the HU,
\begin{align}
    \label{eq:limit-point}
    \bar{\pi}^{i_1} &= \argmax_{\pi^{i_1}\in\mathcal{U}^{i_1}_{\vbarpi}(\bar{\pi}^{i_1})}\E_{\rs\sim\beta_{\vbarpi}}\Big[ \big[ \mathcal{M}^{(\pi^{i_1})}_{\mathfrak{D}^{i_{1}}} A_{\vbarpi}\big] (\rs) \Big] \\
    &= \argmax_{\pi^{i_1}\in\mathcal{U}^{i_1}_{\vbarpi}(\bar{\pi}^{i_1})}\E_{\rs\sim\beta_{\vbarpi}}\Big[ \E_{\ra^{i_1}\sim\pi^{i_1}}\big[A^{i_1}_{\vbarpi}(\rs, \ra^{i_1})\big] - \mathfrak{D}^{i_{1}}_{\vbarpi}(\pi^{i_1}|\rs)\Big].\nonumber
\end{align}
As the HADF is non-negative, and at $\pi^{i_1}=\bar{\pi}^{i_1}$ its value and of its all G{\^ a}teaux derivatives are zero, it follows by Step 3 of Theorem 1 of \cite{kuba2022mirror} that for every $s\in\mathcal{S}$,
\begin{align}
    \bar{\pi}^{i_1}(\cdot^{i_1}|s) = \argmax\limits_{\pi^{i_1}\in\mathcal{P}(\mathcal{A}^{i_1})} \E_{\ra^{i_1}\sim\pi^{i_1}}\big[ Q_{\vbarpi}^{i_1}(s, \ra^{i_1})\big].\nonumber
\end{align}

\paragraph{Step 4: Nash equilibrium.} We have proved that $\vbarpi$ satisfies
\begin{align}
    \bar{\pi}^i(\cdot^{i}|s) &= \argmax_{\pi^{i}(\cdot^{i}|s)\in\mathcal{P}(\mathcal{A}^i)}\E_{\ra^{i}\sim\pi^i}\big[ Q^{i}_{\vbarpi}(s, \ra^i) \big] \nonumber\\
    &= \argmax_{\pi^{i}(\cdot^{i}|s)\in\mathcal{P}(\mathcal{A}^i)}\E_{\ra^{i}\sim\pi^i, \rva^{-i}\sim\vbarpi^{-i}}\big[ Q_{\vbarpi}(s, \rva) \big],\ \forall i\in\mathcal{N}, s\in\mathcal{S}.\nonumber
\end{align}
Hence, by considering $\vbarpi^{-i}$ fixed, we see that $\bar{\pi}^i$ satisfies the condition for the optimal policy \cite{sutton2018reinforcement}, and hence
\begin{align}
    \bar{\pi}^i = \argmax\limits_{\pi^i\in\Pi^i}J(\pi^i, \vbarpi^{-i}).\nonumber
\end{align}
Thus, $\vbarpi$  is a Nash equilibrium. Lastly, this implies that the value function corresponds to a Nash value function $V^{\text{NE}}$, the return corresponds to a Nash return $J^{\text{NE}}$.
\end{proof}

\clearpage

\section{Casting HAPPO as HAML}
\label{appendix:happo}
The maximisation objective of agent $i_m$ in HAPPO is
\begin{align}
    &\E_{\rs\sim\rho_{\boldsymbol{\pi}_{\text{old}}},
    \rva^{i_{1:m-1}}\sim \vpi^{i_{1:m-1}}_{\text{new}},
    \ra^{i_m}\sim\pi^{i_m}_{\text{old}}}\Big[ 
    \min\Big( \rr(\bar{\pi}^{i_m})A^{i_{1:m}}_{\vpi_{\text{old}}}(\rs, \rva^{i_{1:m}}), \text{clip}\big( \rr(\bar{\pi}^{i_m}), 1\pm \epsilon\big)A^{i_{1:m}}_{\vpi_{\text{old}}}(\rs, \rva^{i_{1:m}})\Big) \Big].\nonumber
\end{align}
Fixing $s$ and $\va^{i_{1:m-1}}$, we can rewrite it as
\begin{align}
    &\E_{
    \ra^{i_m}\sim\bar{\pi}^{i_m}}\big[
    A^{i_{1:m}}_{\vpi_{\text{old}}}(\rs, \va^{i_{1:m-1}}, \ra^{i_m})\big] - 
    \E_{\ra^{i_m}\sim\pi^{i_m}_{\text{old}}}\Big[
    \rr(\bar{\pi}^{i_m})A^{i_{1:m}}_{\vpi_{\text{old}}}(s, \va^{i_{1:m-1}}, \ra^{i_m}) \nonumber\\
    &-
    \min\Big( \rr(\bar{\pi}^{i_m})A^{i_{1:m}}_{\vpi_{\text{old}}}(\rs, \va^{i_{1:m-1}}, \ra^{i_m}), \text{clip}\big( \rr(\bar{\pi}^{i_m}), 1\pm \epsilon\big)A^{i_{1:m}}_{\vpi_{\text{old}}}(s, \va^{i_{1:m-1}}, \ra^{i_m}) \Big)\Big]. \nonumber
\end{align}
By the multi-agent advantage decomposition, 
\begin{align}
    &\E_{
    \ra^{i_m}\sim\bar{\pi}^{i_m}}\big[
    A^{i_{1:m}}_{\vpi_{\text{old}}}(\rs, \va^{i_{1:m-1}}, \ra^{i_m})\big]\nonumber\\
    &= A_{\vpi_{\text{old}}}^{i_{1:m-1}}(s, \va^{i_{1:m-1}}) + 
    \E_{
    \ra^{i_m}\sim\bar{\pi}^{i_m}}\big[
    A^{i_{m}}_{\vpi_{\text{old}}}(\rs, \va^{i_{1:m-1}}, \ra^{i_m})\big]. \nonumber
\end{align}
Hence, the presence of the joint advantage of agents $i_{1:m}$ is equivalent to the multi-agent advantage of $i_m$ given $\va^{i_{1:m-1}}$ that appears in HAMO, since the term $A_{\vpi_{\text{old}}}^{i_{1:m-1}}(s, \va^{i_{1:m-1}})$ cancels out with $-1\cdot M^{i_{1:m}}(\rs, \rva)$ of Equation \ref{eq:mad-estimator} that we drop due to its zero gradient.  Hence, we only need to show that that the subtracted term is an HADF.
Firstly, we change $\min$ into $\max$ with the identity $-\min f(x) = \max [ -f(x) ]$.
\begin{align}
    &\E_{\ra^{i_m}\sim\pi^{i_m}_{\text{old}}}\Big[
    \rr(\bar{\pi}^{i_m})A^{i_{1:m}}_{\vpi_{\text{old}}}(s, \va^{i_{1:m-1}}, \ra^{i_m}) \nonumber\\
    &+
    \max\Big(-\rr(\bar{\pi}^{i_m})A^{i_{1:m}}_{\vpi_{\text{old}}}(\rs, \va^{i_{1:m-1}}, \ra^{i_m}), -\text{clip}\big( \rr(\bar{\pi}^{i_m}), 1\pm \epsilon\big)A^{i_{1:m}}_{\vpi_{\text{old}}}(s, \va^{i_{1:m-1}}, \ra^{i_m})\Big) \Big]\nonumber\\
    &\text{which we then simplify}\nonumber\\
    &\E_{\ra^{i_m}\sim\pi_{\text{old}}^{i_m}}\Big[
    \max\Big( 0, \big[ \rr(\bar{\pi}^{i_m}) -\text{clip}\big( \rr(\bar{\pi}^{i_m}), 1\pm \epsilon\big)\big]A^{i_{1:m}}_{\boldsymbol{\pi}_{\text{old}}}(s, \va^{i_{1:m-1}}, \ra^{i_m})\Big) \Big]  \nonumber\\
    & =\E_{\ra^{i_m}\sim\pi^{i_m}_{\text{old}}}\Big[
    \text{ReLU}\Big( \big[ \rr(\bar{\pi}^{i_m}) -\text{clip}\big( \rr(\bar{\pi}^{i_m}), 1\pm \epsilon\big)\big]A^{i_{1:m}}_{\boldsymbol{\pi}_{\text{old}}}(s, \va^{i_{1:m-1}}, \ra^{i_m})\Big) \Big]. \nonumber
\end{align}
As discussed in the main body of the paper, this is an HADF.

\clearpage
\section{Algorithms}

\label{appendix:algos}

\begin{algorithm}[!htbp]
    \caption{HAA2C}
    \label{HAA2C}
    \textbf{Input:} stepsize $\alpha$, batch size $B$, number of: agents $n$, episodes $K$, steps per episode $T$, mini-epochs $e$\;
    \textbf{Initialize:} the critic network: $\phi$, the policy networks: $\{\theta^{i}\}_{i\in\mathcal{N}}$, replay buffer $\mathcal{B}$\;
    
    \For{$k = 0,1,\dots,K-1$}{ 
        Collect a set of trajectories by letting the agents act according to their policies, $\ra^i\sim\pi^i_{\theta^i}(\cdot^i|\ro^i)$\;
        Push transitions $\{(s_t, o^i_{t},a^i_{t},r_t, s_{t+1}, o^i_{t+1}), \forall i\in \mathcal{N},t\in T\}$ into $\mathcal{B}$\;
        Sample a random minibatch of $B$ transitions from $\mathcal{B}$\;
        Estimate the returns $R$ and the advantage function, $\hat{A}(\rs, \rva)$, using $\hat{V}_{\phi}$ and GAE\;
        Draw a permutation of agents $i_{1:n}$ at random\;
        Set $M^{i_1}(\rs, \rva) = \hat{A}(\rs, \rva)$\;
        \For{agent $i_{m} = i_{1}, \dots, i_{n}$}{
            Set $\pi^{i_m}_{0}(\ra^{i_m}|\ro^{i_m})= \pi^{i_m}_{\theta^{i_m}}(\ra^{i_m}|\ro^{i_m})$\;
            \For{mini-epoch$=1, \dots, e$}{
            Compute agent $i_m$'s policy gradient
            \begin{center}
                $\rvg^{i_m} =
                \nabla_{\theta^{i_m}}
                \frac{1}{B}\sum\limits_{b=1}^{B}M^{i_m}(s_b, \va_b)
                \frac{\pi^{i_m}_{\theta^{i_m}}(a^{i_m}_b|o^{i_m}_b)}{\pi^{i_m}_{0}(a^{i_m}_b|o^{i_m}_b)}$.
            \end{center}
            Update agent $i_m$'s policy by
            \begin{center}
                $\theta^{i_m} = \theta^{i_m} +  \alpha \rvg^{i_m}$.
            \end{center}
            }
            Compute $M^{i_{m+1}}(\rs, \rva) = \frac{\pi^{i_m}_{\theta^{i_m}}(\ra^{i_m}|\ro^{i_m})}{\pi^{i_m}_{0}(\ra^{i_m}|\ro^{i_m})} M^{i_m}(\rs, \rva)$ \text{ \color{RoyalPurple} //Unless $m=n$.}
        }
        Update the critic by gradient descent on
        \begin{center}
            $\frac{1}{B}\sum\limits_{b}\big( \hat{V}_{\phi}(s_b) - R_b\big)^2$.
        \end{center}
    }
    Discard $\phi$. Deploy $\{\theta^{i}\}_{i\in\mathcal{N}}$ in execution\;
\end{algorithm}

\clearpage
\begin{algorithm}[!htbp]
    \caption{HADDPG}
    \label{HADDPG}
    \textbf{Input:} stepsize $\alpha$, Polyak coefficient $\tau$, batch size $B$, number of: agents $n$, episodes $K$, steps per episode $T$, mini-epochs $e$\;
    \textbf{Initialize:} the critic networks: $\phi$ and $\hat \phi$ and policy networks: $\{\theta^{i}\}_{i\in\mathcal{N}}$ and $\{\hat \theta^{i}\}_{i\in\mathcal{N}}$, replay buffer $\mathcal{B}$, random processes $\{\mathcal{X}^i\}_{i\in\mathcal{N}}$ for exploration\;
    
    \For{$k = 0,1,\dots,K-1$}{ 
        Collect a set of transitions by letting the agents act according to their deterministic policies with the exploratory noise
        \begin{center}
        $a_t^i = \mu^i_{\theta^i}(o_t^{i}) + \mathcal{X}^i_t$.
        \end{center}
        Push transitions $\{(s_t, o^i_{t}, a^i_{t}, r_t, s_{t+1}, o^i_{t+1}), \forall i\in \mathcal{N},t\in T\}$ into $\mathcal{B}$\;
        Sample a random minibatch of $B$ transitions from $\mathcal{B}$\;
        Compute the critic targets
        \begin{center}
            $y_t = r_t + \gamma Q_{\hat \phi}(s_{t+1}, \hat \va_{t+1})$, where $\hat \va_{t+1}$ is sampled by $\{\hat \theta^{i}\}_{i\in\mathcal{N}}$.
        \end{center}
        Update the critic by minimising the loss
        \begin{center}
        $\phi = \argmin_{\phi}\frac{1}{B}\sum_{t}\big(y_t - Q_{\phi}(s_{t}, \va_{t}) \big)^2$.
        \end{center}
        Draw a permutation of agents $i_{1:n}$ at random\;
        \For{agent $i_{m} = i_{1}, \dots, i_{n}$}{
            Update agent $i_m$  by solving 
            \begin{align*}
                &\theta_{\text{new}}^{i_m} = \\
                &\argmax_{\tilde{\theta}^{i_m}}\frac{1}{B}\sum_t Q_{\phi}\big(s_t, \vmu^{i_{1:m-1}}_{\boldsymbol{\theta}_{\text{new}}^{i_{1:m-1}}}(\vo^{i_{1:m-1}}_t), \mu^{i_m}_{\tilde{\theta}^{i_m}}(o^{i_m}_t), \vmu^{i_{m+1:n}}_{\boldsymbol{\theta}_{\text{old}}^{i_{m+1:n}}}(\vo^{i_{m+1:n}}_t)\big).
            \end{align*}
            with $e$ mini-epochs of deterministic policy gradient ascent\;
        }
        Update the target networks smoothly
        \begin{center}
            $\hat \phi = \tau \phi + (1-\tau)\hat \phi$.
        \end{center}
        \begin{center}
            $\hat \theta^{i} = \tau \theta^{i} + (1-\tau)\hat \theta^{i}$.
        \end{center}
    }
    Discard $\phi, \hat{\phi}$, and $\hat{\theta}^{i}, \forall i\in\mathcal{N}$. Deploy $\theta^{i}, \forall i\in\mathcal{N}$ in execution.
\end{algorithm}

\clearpage
\begin{algorithm}[!htbp]
    \caption{HATD3}
    \label{HATD3}
    \textbf{Input:} stepsize $\alpha$, Polyak coefficient $\tau$, batch size $B$, number of: agents $n$, episodes $K$, steps per episode $T$, mini-epochs $e$, target noise range $c$\;
    \textbf{Initialize:} the critic networks: $\phi_1, \phi_2$ and $\hat \phi_1, \hat \phi_2$ and policy networks: $\{\theta^{i}\}_{i\in\mathcal{N}}$ and $\{\hat \theta^{i}\}_{i\in\mathcal{N}}$, replay buffer $\mathcal{B}$, random processes $\{\mathcal{X}^i\}_{i\in\mathcal{N}}$ for exploration\;
    
    \For{$k = 0,1,\dots,K-1$}{ 
        Collect a set of transitions by letting the agents act according to their deterministic policies with the exploratory noise
        \begin{center}
        $a_t^i = \mu^i_{\theta^i}(o_t^{i}) + \mathcal{X}^i_t$.
        \end{center}
        Push transitions $\{(s_t, o^i_{t}, a^i_{t}, r_t, s_{t+1}, o^i_{t+1}), \forall i\in \mathcal{N},t\in T\}$ into $\mathcal{B}$\;
        Sample a random minibatch of $B$ transitions from $\mathcal{B}$\;
        Compute the critic targets
                \begin{center}
        $y_t = r_t + \gamma\min_{j=1,2} Q_{\hat \phi_j}(s_{t+1}, \hat \va_{t+1})$, where
        \end{center}
        \begin{center}
            $\hat a_{t+1}^i = \text{clip}(\mu^i_{\hat \theta^i}(o_{t+1}^{i}) + \epsilon, a^i_{\text{Low}}, a^i_{\text{High}})$, $\epsilon \sim \text{clip}(\mathcal{N}(0, \Tilde{\sigma}), -c, c)$. \color{RoyalPurple}$\vartriangleright$ Here $\mathcal{N}$ denotes Normal distribution.\color{black}
        \end{center}
        Update the critic by minimising the loss
        \begin{center}
        $\phi_j = \argmin_{\phi_j}\frac{1}{B}\sum_{t}\big(y_t - Q_{\phi_j}(s_{t}, \va_{t}) \big)^2, j = 1, 2$.
        \end{center}
        \If{$k \mod \text{policy\_delay} = 0$}{
        Draw a permutation of agents $i_{1:n}$ at random\;
        \For{agent $i_{m} = i_{1}, \dots, i_{n}$}{
            Update agent $i_m$  by solving 
            \begin{align*}
                &\theta_{\text{new}}^{i_m} = \\
                &\argmax_{\tilde{\theta}^{i_m}}\frac{1}{B}\sum_t Q_{\phi_1}\big(s_t, \vmu^{i_{1:m-1}}_{\boldsymbol{\theta}_{\text{new}}^{i_{1:m-1}}}(\vo^{i_{1:m-1}}_t), \mu^{i_m}_{\tilde{\theta}^{i_m}}(o^{i_m}_t), \vmu^{i_{m+1:n}}_{\boldsymbol{\theta}_{\text{old}}^{i_{m+1:n}}}(\vo^{i_{m+1:n}}_t)\big).
            \end{align*}
            with $e$ mini-epochs of deterministic policy gradient ascent\;
        }
        Update the target networks smoothly
        \begin{center}
            $\hat \phi_1 = \tau \phi_1 + (1-\tau)\hat \phi_1$.
        \end{center}
        \begin{center}
            $\hat \phi_2 = \tau \phi_2 + (1-\tau)\hat \phi_2$.
        \end{center}
        \begin{center}
            $\hat \theta^{i} = \tau \theta^{i} + (1-\tau)\hat \theta^{i}$.
        \end{center}
        }
    }
    Discard $\phi_1, \phi_2, \hat{\phi}_1, \hat{\phi}_2$, and $\hat{\theta}^{i}, \forall i\in\mathcal{N}$. Deploy $\theta^{i}, \forall i\in\mathcal{N}$ in execution.
\end{algorithm}

\clearpage
\section{The Summary of HARL algorithms as Instances of HAML}
\label{appendix:haml-summary}

\subsection*{Recap of HAML}

\begin{itemize}
    \item Definition of HAMO:
\begin{align}
    \big[\mathcal{M}^{(\pi^{i_m})}_{\mathfrak{D}^{i_{m}}, \vpi_{k+1}^{i_{1:m-1}}} A_{\boldsymbol{\vpi}_{k}}\big](s) \triangleq &\ \E_{\rva^{i_{1:m-1}}\sim\vpi_{k+1}^{i_{1:m-1}}, \ra^{i_m}\sim\pi^{i_m}}\Big[ A_{\boldsymbol{\pi}_k}^{i_{m}}\left(s, \rva^{i_{1:m-1}}, \ra^{i_{m}}\right) \Big] \nonumber \\
    &- \mathfrak{D}^{i_m}_{\vpi_k}\Big(\pi^{i_{m}} \big| s, \vpi_{k+1}^{i_{1:m-1}}\Big).\nonumber
\end{align}
\item Optimisation target: $\pi^{i_{m}}_{k+1} = \argmax\limits_{\pi^{i_{m}}\in\mathcal{U}^{i_m}_{\vpi_k}(\pi^{i_m}_k)}\E_{\rs\sim\beta_{\vpi_k}}\Big[
        \big[\mathcal{M}^{(\pi^{i_m})}_{\mathfrak{D}^{i_{m}}, \vpi_{k+1}^{i_{1:m-1}}} A_{\boldsymbol{\vpi}_k}\big](\rs)
        \Big]$
\end{itemize}

\subsection*{HATRPO}

\begin{align}
\pi_{k+1}^{i_m} = &\  \underset{\pi^{i_m}}{\arg\max}\ \mathbb{E}_{\mathrm{s} \sim \rho_{\boldsymbol{\pi}_{k}}, \mathbf{a}^{i_{1: m-1}} \sim \bm{\pi}_{k+1}^{i_{1: m-1}}, \mathrm{a}^{i_m} \sim \pi^{i m}}\left[A_{\boldsymbol{\pi}_{k}}^{i_{m}}\left(\mathrm{s}, \mathbf{a}^{i_{1: m-1}}, \mathrm{a}^{i_{m}}\right)\right], \nonumber\\
            &\ \text { subject to } \bar{\text{D}}_{\text{KL}}\left(\pi_k^{i_{m}}, \pi^{i_{m}}\right) \leq \delta .
\end{align}

\begin{itemize}
    \item Drift functional: HADF $\mathfrak{D}^{i_m}_{\vpi_k}\Big(\pi^{i_{m}} \big| s,  \vpi_{k+1}^{i_{1:m-1}}\Big) \equiv 0$.
    \item Neighborhood operator:
    \begin{align}
\mathcal{U}^{i_m}_{\vpi_k}(\pi^{i_m}_k) = \Big\{ \pi^{i_m} \in \Pi^{i_m} \ \Big| \ \E_{\rs\sim\rho_{\vpi_k}}\Big[\text{D}_{\text{KL}}\big( \pi^{i_m}_k(\cdot|\rs), \pi^{i_m}(\cdot|\rs) \big)\Big] \leq \delta \Big\}. \nonumber
\end{align}
    \item Sampling distribution: $\beta_{\vpi_k} = \rho_{\boldsymbol{\pi}_k}$.
\end{itemize}

\subsection*{HAPPO}

\begin{align}
\pi_{k+1}^{i_m} = \  \underset{\pi^{i_m}}{\arg\max}\ &\E_{\rs\sim\rho_{\boldsymbol{\pi}_k},
    \rva^{i_{1:m-1}}\sim \vpi^{i_{1:m-1}}_{k+1},
    \ra^{i_m}\sim\pi^{i_m}_k} \nonumber\\
    &\Big[ 
    \min\Big( \rr(\pi^{i_m})A^{i_{1:m}}_{\vpi_k}(\rs, \rva^{i_{1:m}}), \text{clip}\big( \rr(\pi^{i_m}), 1\pm \epsilon\big)A^{i_{1:m}}_{\vpi_k}(\rs, \rva^{i_{1:m}})\Big) \Big],\nonumber \\            &\ \text { where } \rr(\pi^{i_m}) = \frac{\pi^{i_m}(\ra^{i_m}|\rs)}{\pi^{i_m}_k(\ra^{i_m}|\rs)}.
\end{align}

\begin{itemize}
    \item Drift functional:
    \begin{align}
    &\mathfrak{D}^{i_m}_{\vpi_k}\Big(\pi^{i_{m}} \big| s, \vpi_{k+1}^{i_{1:m-1}}\Big) = \nonumber \\
    &\E_{\rva^{i_{1:m-1}}\sim \vpi^{i_{1:m-1}}_{k+1},
    \ra^{i_m}\sim\pi^{i_m}_{k}}\big[
    \text{ReLU}\big( \big[ \rr(\pi^{i_m}) -\text{clip}\big( \rr(\pi^{i_m}), 1\pm \epsilon\big)\big]A^{i_{1:m}}_{\boldsymbol{\pi}_{k}}(s, \rva^{i_{1:m}})\big) \big]
    \end{align}
    \item Neighborhood operator: $\mathcal{U}^{i_m}_{\vpi_k}(\pi^{i_m}_k) \equiv \Pi^{i_m}$.
    \item Sampling distribution: $\beta_{\vpi_k} = \rho_{\boldsymbol{\pi}_k}$.
\end{itemize}

\subsection*{HAA2C}

\begin{align}
\pi_{k+1}^{i_m} = &\  \underset{\pi^{i_m}}{\arg\max}\ \mathbb{E}_{\mathrm{s} \sim \rho_{\boldsymbol{\pi}_{k}}, \mathbf{a}^{i_{1: m-1}} \sim \bm{\pi}_{k+1}^{i_{1: m-1}}, \mathrm{a}^{i_m} \sim \pi^{i m}}\left[A_{\boldsymbol{\pi}_{k}}^{i_{m}}\left(\mathrm{s}, \mathbf{a}^{i_{1: m-1}}, \mathrm{a}^{i_{m}}\right)\right]
\end{align}

\begin{itemize}
    \item Drift functional: HADF $\mathfrak{D}^{i_m}_{\vpi_k}\Big(\pi^{i_{m}} \big| s, \vpi_{k+1}^{i_{1:m-1}}\Big) \equiv 0$.
    \item Neighborhood operator: $\mathcal{U}^{i_m}_{\vpi_k}(\pi^{i_m}_k) \equiv \Pi^{i_m}$.
    \item Sampling distribution: $\beta_{\vpi_k} = \rho_{\boldsymbol{\pi}_k}$.
\end{itemize}

\subsection*{HADDPG \& HATD3}

\begin{align}
    \mu_{k+1}^{i_m} = \ \underset{\mu^{i_m}}{\arg\max}\ \E_{\rs\sim\beta_{\vmu_{k}}}\Big[ Q^{i_{1:m}}_{\vmu_{k}}\big(\rs, \vmu_{k+1}^{i_{1:m-1}}(\rs), \mu^{i_m}(\rs)
    \big)\Big],
\end{align}

\begin{itemize}
    \item Drift functional: HADF $\mathfrak{D}^{i_m}_{\vmu_k}\Big(\mu^{i_{m}} \big| s, \vmu_{k+1}^{i_{1:m-1}}\Big) \equiv 0$.
    \item Neighborhood operator: 
    \begin{align}
    \mathcal{U}^{i_m}_{\vmu_k}(\mu^{i_m}_k) \equiv \Pi^{i_m} \ \ \text{(the deterministic policy space)}.\nonumber
    \end{align}
    \item Sampling distribution: $\beta_{\vmu_k}$ is a uniform distribution over the states in the off-policy replay buffer.
\end{itemize}

\clearpage
\section{HAD3QN: A Pure Value-based Approximation to HADDPG}
\label{appendix:had3qn}

In this section, we propose HAD3QN, which is a pure value-based approximation of HADDPG. Corresponding to HADDPG where each agent learns to maximise the joint target given the previous agents' updates, HAD3QN models decentralised agents as individual Q networks that predict the centralised critic's output. In particular, the centralised critic's output is sequentially maximised for sequential learning. During execution, for each observation each agent chooses the action that maximises its individual Q network. We provide its pseudocode in Algorithm \ref{HADQN}.

\begin{figure}[h]
  \centering
  \includegraphics[width=0.8\linewidth]{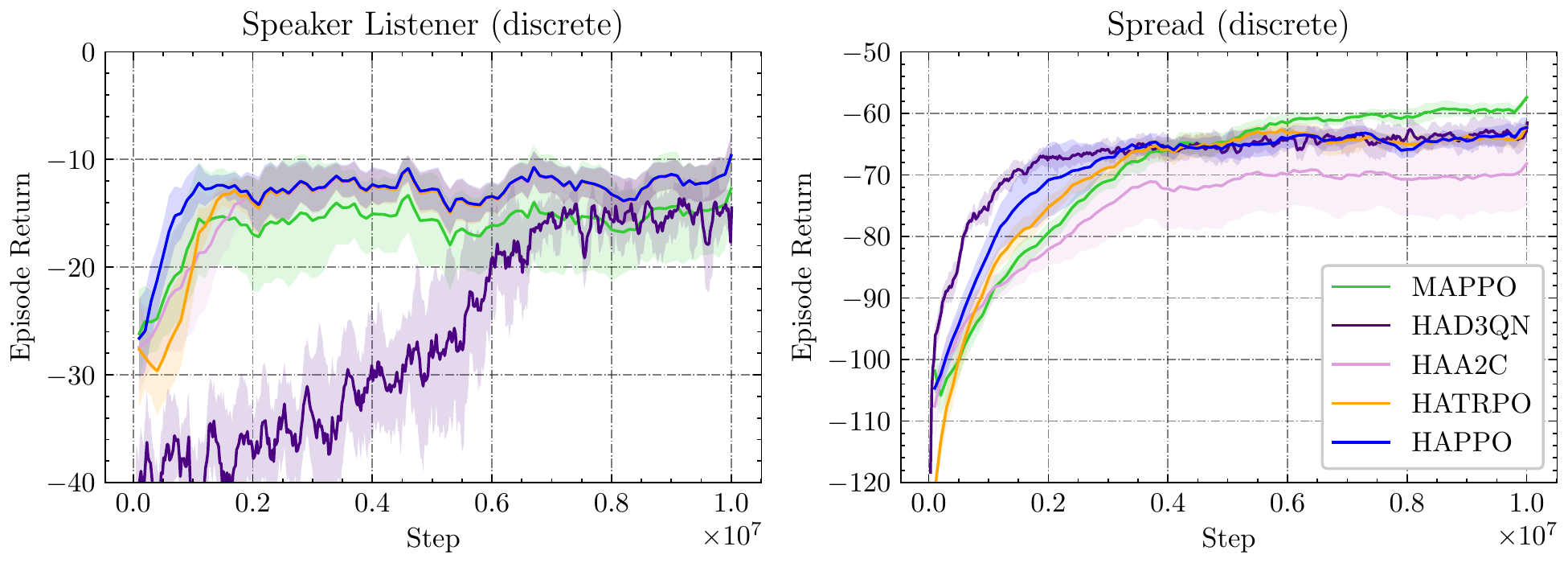}
  \caption{Average episode return of HAD3QN on Speaker Listener and Spread compared with existing methods.}
  \label{fig:had3qn_report}
\end{figure}

\begin{figure}[h]
  \centering
  \includegraphics[width=0.8\linewidth]{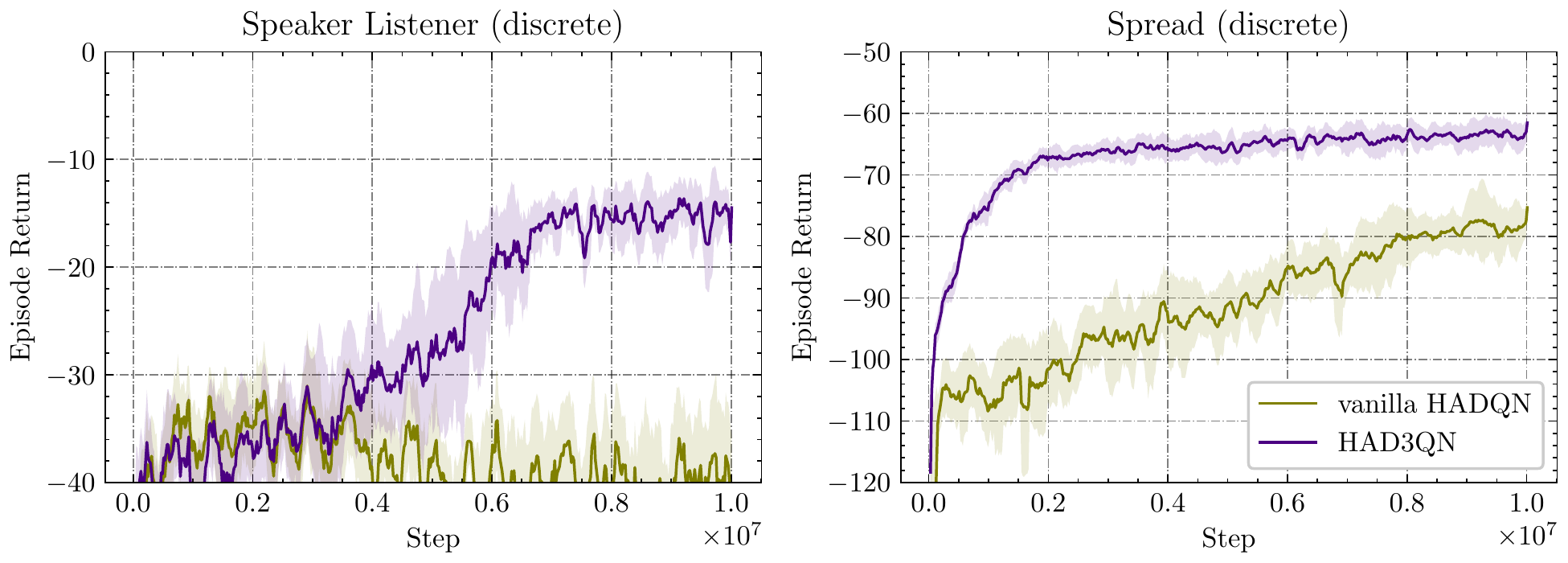}
  \caption{Ablation study on the effect of dueling network architecture in HAD3QN.}
  \label{fig:had3qn_ablation}
\end{figure}

Empirically, we test it on the Speaker Listener and Spread task in MPE, and observe that HAD3QN is able to solve them within 10 million steps (Figure \ref{fig:had3qn_report}). Compared with the vanilla HADQN where dueling architecture is not utilised (Figure \ref{fig:had3qn_ablation}), we find that the dueling network architecture effectively improves learning efficiency and stability, and is crucial for HAD3QN to achieve higher return. The hyperparameters are reported in Section \ref{appendix:exp}. 

However, we note that HAD3QN does not scale well as it suffers from the curse of dimensionality with the growing number of agents and increasing dimensionality of individual action space. This phenomenon is similar to what has been discussed in the DQN case in RL by \cite{lillicrap2015continuous}. The purpose of proposing HAD3QN is not to refresh SOTA methods, but to show that discretised approximation of HADDPG is also possible and it performs well on low-dimensional tasks. It also shows that our HARL framework allows direct extension of RL research results, in this case being the dueling network design, which is potentially powerful as the efforts to re-derive similar multi-agent results can be saved.

\begin{algorithm}[!htbp]
    \caption{HAD3QN}
    \label{HADQN}
    
    \textbf{Input:} stepsize $\alpha$, Polyak coefficient $\tau$, batch size $B$, exploration parameter $\epsilon$, number of: agents $n$, episodes $K$, steps per episode $T$.\\
    \textbf{Initialize:} global critic and target networks: $\phi$, and $\hat{\phi}$, distributed critic and target networks: $\{\theta^{i}, \ \forall i\in \mathcal{N}\}$ and $\{\hat{\theta}^{i}, \ \forall i\in \mathcal{N}\}$, replay buffer $\mathcal{B}$.\\
    
    \For{$k = 0,1,\dots,K-1$}{
        Collect a set of trajectories by letting the agents act $\epsilon$-greedily with respect to the distributed critics
        \begin{align}
            a_t^{i} = 
            \begin{cases}
            \argmax_{a^i} Q^{i}_{\theta^i}(o^i_t, a^i) \quad \text{with probability }1-\epsilon\nonumber\\
            \text{random} \quad \quad \quad \quad \quad \quad \ \ \  \text{with probability }\epsilon.
            \end{cases}
        \end{align}
        
        Push transitions $\{(s_{t}, o^i_{t},a^i_{t}, r_t, s_{t+1}, o^i_{t+1}), \forall i\in \mathcal{N},t\in T\}$ into $\mathcal{B}$.
        
        Sample a random minibatch of $B$ transitions from $\mathcal{B}$.

        Compute the global target 
        \begin{center}
            $y_{t} = r_t + \gamma \cdot Q_{\hat{\phi}}(s_{t+1}, \va_{*})$,\\
            where $a^{i}_* = \argmax_{a^{i}} Q^{i}_{\hat{\theta}^i}(o^i_{t+1}, a^{i})$, for all $i\in\mathcal{N}$.
        \end{center}
        Compute the global loss
        \begin{center}
                $L(\phi) = \frac{1}{B}\sum\limits_{b=1}^{B}\big( Q_{\phi}(s_b, \va_b) - y_{b} \big)^2$.
           \end{center}
           Update the critic parameters
           \begin{center}
               $\phi = \phi - \alpha \nabla_{\phi}L(\phi)$.
           \end{center}
        
        Draw a permutation of agents $i_{1:n}$ at random\;
        \For{agent $i_{m} = i_{1}, \dots, i_{n}$}{
            Compute the local targets
            \begin{center}
               $y^{i_{m}}_{t} =  Q_{\phi}(s_{t}, \va^{i_{1:m-1}}_{*}, \va^{-i_{1:m-1}}_{t})$,\\
               where $a^{i_j}_* = \argmax_{a^{i_j}} Q_{\phi}(s_{t}, \va^{i_{1:j-1}}_{*}, a^{i_j}, \va^{-i_{1:j}}_t)$, for $j<m$.
              \end{center}
           Compute the agent's local loss
           \begin{center}
                $L(\theta^{i_m}) = \frac{1}{B}\sum\limits_{b=1}^{B}\big( Q^{i_m}_{\theta^{i_m}}(o^{i_m}_b, a^{i_m}_b) - y^{i_m}_{b} \big)^2$.
           \end{center}
           Update the critic parameters
           \begin{center}
               $\theta^{i_m} = \theta^{i_m} - \alpha \nabla_{\theta^{i_m}}L(\theta^{i_m})$.
           \end{center}
        }
        Update the target networks smoothly
        \begin{center}
            $\hat \phi = \tau \phi + (1-\tau)\hat \phi$,\quad  $\hat \theta^{i} = \tau \theta^{i} + (1-\tau)\hat \theta^{i}$.
        \end{center}
    }
    Discard $\phi, \hat{\phi}$, and $\hat{\theta}^{i}, \forall i\in\mathcal{N}$. Deploy $\theta^{i}, \forall i\in\mathcal{N}$ in execution.
\end{algorithm}

\section{Additional Experiment Results}
\label{appendix:smac-figures}

In this section, we present the learning curves of HAPPO, HATRPO, MAPPO, and QMIX across at least three seeds on ten SMAC maps and five SMACv2 maps in Figure \ref{fig:smac}.

\begin{figure}[h]
  \centering
  \includegraphics[width=\linewidth]{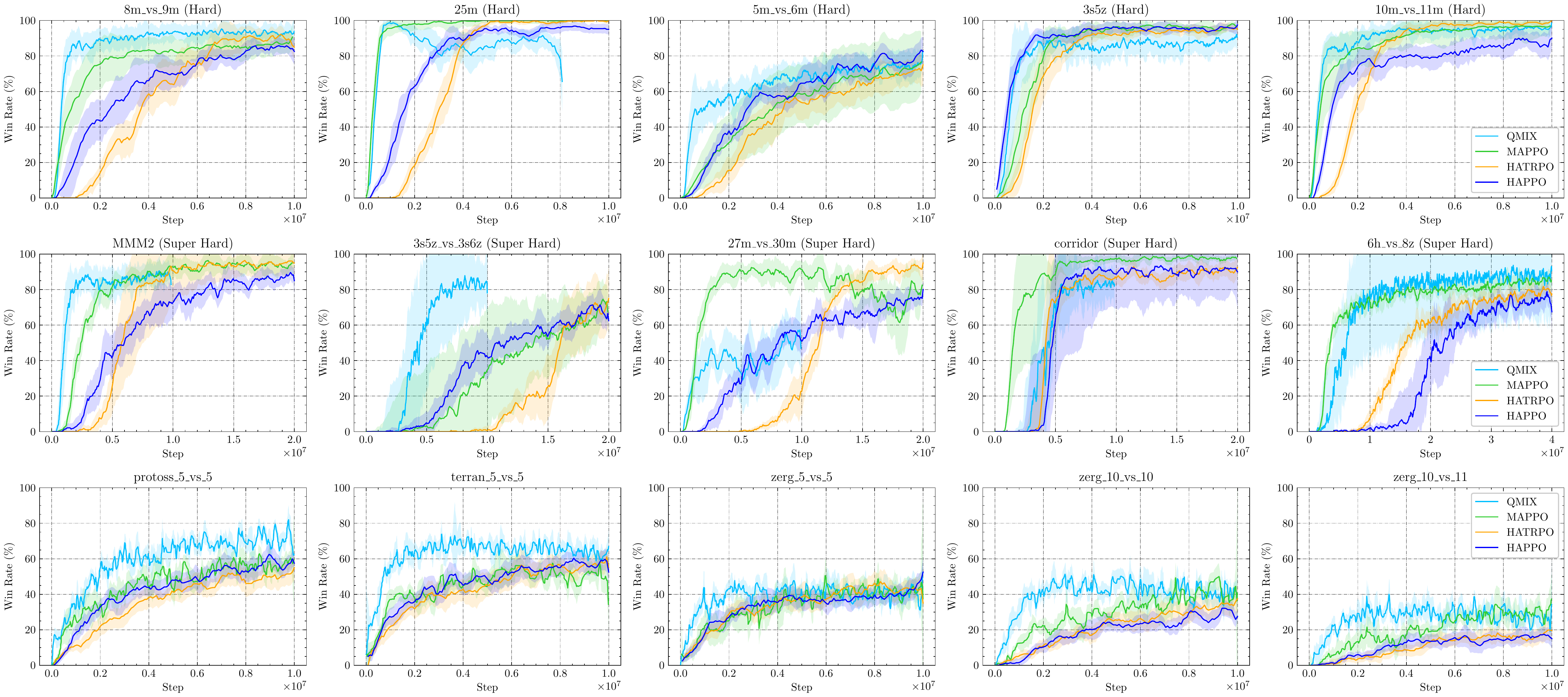}
  \caption{Comparisons of average win rate on SMAC and SMACv2. It should be noted that some of the QMIX experiments were terminated early if they had already converged, as observed in \texttt{MMM2}, \texttt{3s5z\_vs\_3s6z}, and \texttt{corridor}, or if the computational resources required were excessive, as observed in the case of \texttt{27m\_vs\_30m}. Specifically, running QMIX for a single seed for 20 million steps in \texttt{27m\_vs\_30m} would have necessitated more than 250 GB memory and 10 days, which exceeded the computational budget allocated for this study. Consequently, we executed the experiment for only 10 million steps.}
  \label{fig:smac}
\end{figure}

\newpage
\section{Hyperparameter Settings for Experiments}
\label{appendix:exp}

Before we report the hyperparameters used in the experiments, we would like to clarify the reporting conventions that we follow. Firstly, for simplicity and clarity reasons, we specify the network architecture to be MLP or RNN, but in configuration files the corresponding term is a boolean value \texttt{use\_recurrent\_policy} . The only difference between RNN network and MLP network is that the former has a GRU layer after the same MLP backbone, and the related configuration of this GRU layer is provided in Table \ref{tab:on-policy-common}. Secondly, the hyperparameters will only take effect when they are used. For example, the number of GRU layers is set to 1 across all environments, but it should only be considered when the network architecture is RNN; as another example, while we report \texttt{kl\_threshold} in on-policy hyperparameter tables, it is only useful when HATRPO is applied. Finally, the \texttt{batch\_size} reported for on-policy algorithms is calculated as the product of \texttt{n\_rollout\_threads} and \texttt{episode\_length}.

\subsection{Common Hyperparameters Across All Environments}

In this part, we present the common hyperparameters used for on-policy algorithms in Table \ref{tab:on-policy-common} and for off-policy algorithms in Table \ref{tab:off-policy-common} across all environments. 

\begin{table}[htbp]
    \caption{Common hyperparameters used for on-policy algorithms HAPPO, HATRPO, HAA2C, and MAPPO (when our MAPPO implementation is used) across all environments.}
    \label{tab:on-policy-common}
    \centering
\begin{tabular}{cc|cc}
\hline
hyperparameters            & value      & hyperparameters        & value            \\ \hline
use valuenorm                     & True       & use proper time limits     & True             \\
activation                 & ReLU       & use feature normalization  & True             \\
initialization method      & orthogonal & gain                   & 0.01             \\
use naive recurrent policy & False      & num GRU layers         & 1                \\
data chunk length          & 10         & optim eps              & $1 \mathrm{e}-5$ \\
weight decay               & 0          & std x coef             & 1                \\
std y coef                 & 0.5        & use clipped value loss & True             \\
value loss coef            & 1          & use max grad norm      & True             \\
max grad norm              & $10.0$     & use GAE                & True             \\
GAE lambda                 & 0.95       & use huber loss         & True             \\
use policy active masks    & True       & huber delta            & 10.0             \\
action aggregation         & prod       & ls step                & 10               \\
accept ratio               & 0.5        &                        &                  \\ \hline
\end{tabular}
\end{table}


\begin{table}[htbp]
    \caption{Common hyperparameters used for off-policy algorithms HADDPG, HATD3, HAD3QN, MADDPG, and MATD3 across all environments.}
    \label{tab:off-policy-common}
    \centering
\begin{tabular}{cc|cc}
\hline
hyperparameters      & value     & hyperparameters         & value            \\ \hline
proper time limits   & True      & warmup steps            & $1 \mathrm{e} 4$ \\
activation           & ReLU      & final activation        & Tanh             \\
base activation      & ReLU      & dueling v activation    & Hardswish        \\
dueling a activation & Hardswish & buffer size             & $1 \mathrm{e} 6$ \\
batch size           & 1000      & polyak                  & 0.005            \\
epsilon              & 0.05      & policy noise            & 0.2              \\
noise clip           & 0.5       & &            \\ \hline
\end{tabular}
\end{table}

\subsection{Multi-Agent Particle Environment (MPE)}
\label{appendix:mpe-hyperparam}

In this part, we present the hyperparameters used in MPE tasks for HAPPO, HATRPO, HAA2C, and MAPPO in Table \ref{tab:on-policy-mpe-common}, for HADDPG, HATD3, MADDPG, and MATD3 in Table \ref{tab:off-policy-mpe-common}, and for HAD3QN in Table \ref{tab:had3qn-mpe-common}.

\begin{table}[htbp]
    \caption{Common hyperparameters used for HAPPO, HATRPO, HAA2C, and MAPPO in the MPE domain.}
    \label{tab:on-policy-mpe-common}
    \centering
\begin{tabular}{cc|cc|cc}
\hline
hyperparameters & value      & hyperparameters  & value            & hyperparameters   & value            \\ \hline
batch size      & 4000       & linear lr decay  & False            & network           & MLP              \\
hidden sizes    & [128, 128] & actor lr         & $5 \mathrm{e}-4$ & critic lr         & $5 \mathrm{e}-4$ \\
ppo epoch       & 5          & critic epoch     & 5                & a2c epoch         & 5                \\
clip param      & 0.2        & actor mini batch & 1                & critic mini batch & 1                \\
entropy coef    & 0.01       & gamma            & 0.99             & kl threshold      & 0.005            \\
backtrack coeff & 0.8        &                  &                  &                   &                  \\ \hline
\end{tabular}
\end{table}

\begin{table}[htbp]
    \caption{Common hyperparameters used for HADDPG, HATD3, MADDPG, and MATD3 in the MPE domain.}
    \label{tab:off-policy-mpe-common}
    \centering
\begin{tabular}{cc|cc|cc}
\hline
hyperparameters & value            & hyperparameters & value      & hyperparameters  & value            \\ \hline
rollout threads & 20               & train interval  & 50         & update per train & 1                \\
linear lr decay & False            & hidden sizes    & [128, 128] & actor lr         & $5 \mathrm{e}-4$ \\
critic lr       & $1 \mathrm{e}-3$ & gamma           & 0.99       & n step           & 1                \\
policy update frequency & 2    & & & & \\ \hline
\end{tabular}
\end{table}

\begin{table}[htbp]
    \caption{Common hyperparameters used for HAD3QN in the MPE domain.}
    \label{tab:had3qn-mpe-common}
    \centering
\begin{tabular}{cc|cc|cc}
\hline
hyperparameters & value            & hyperparameters  & value            & hyperparameters        & value          \\ \hline
rollout threads & 20               & train interval   & 50               & base hidden sizes      & {[}128, 128{]} \\
linear lr decay & False            & update per train & 1                & dueling v hidden sizes & {[}128{]}      \\
actor lr        & $5 \mathrm{e}-4$ & critic lr        & $1 \mathrm{e}-3$ & dueling a hidden sizes & {[}128{]}      \\
gamma           & 0.95             & n step           & 1                &                        &                \\ \hline
\end{tabular}
\end{table}

\subsection{Multi-Agent MuJoCo (MAMuJoCo)}
\label{appendix:mamujoco-hyperparam}

In this part, we report the hyperparameters used in MAMuJoCo tasks for HAPPO, HATRPO, HAA2C, and MAPPO in Table \ref{tab:on-policy-mamujoco-common}, \ref{tab:happo-mappo-mamujoco-different}, \ref{tab:hatrpo-mamujoco-different}, and \ref{tab:haa2c-mamujoco-different}, and for HADDPG, HATD3, MADDPG, and MATD3 in Table \ref{tab:haddpg-maddpg-mamujoco-common}, \ref{tab:haddpg-maddpg-mamujoco-different}, \ref{tab:hatd3-mamujoco-common}, and \ref{tab:hatd3-mamujoco-different}.

\begin{table}[htbp]
    \caption{Common hyperparameters used for HAPPO, HATRPO, HAA2C, and MAPPO in the MAMuJoCo domain.}
    \label{tab:on-policy-mamujoco-common}
    \centering
\begin{tabular}{cc|cc|cc}
\hline
hyperparameters & value & hyperparameters & value & hyperparameters & value           \\ \hline
batch size      & 4000  & network         & MLP   & hidden sizes    & [128, 128, 128] \\
gamma           & 0.99  & backtrack coeff & 0.8   &                 &                 \\ \hline
\end{tabular}
\end{table}

\begin{table}[htbp]
    \caption{Different hyperparameters used for HAPPO and MAPPO in the MAMuJoCo domain.}
    \label{tab:happo-mappo-mamujoco-different}
    \centering
\begin{tabular}{c|ccccccc}
\hline
scenarios      & \begin{tabular}[c]{@{}c@{}}linear\\ lr decay\end{tabular} & \begin{tabular}[c]{@{}c@{}}actor/critic\\ lr\end{tabular} &  \begin{tabular}[c]{@{}c@{}}ppo/critic\\ epoch\end{tabular} & \begin{tabular}[c]{@{}c@{}}clip\\ param\end{tabular} & \begin{tabular}[c]{@{}c@{}}actor/critic\\ mini batch\end{tabular} & \begin{tabular}[c]{@{}c@{}}entropy\\ coef\end{tabular} \\ \hline
Ant 4x2         & False                                                     & $5 \mathrm{e}-4$                                   & 5                                                            & 0.1                                                  & 1                                                                   & 0                                                      \\
HalfCheetah 2x3 & False                                                     & $5 \mathrm{e}-4$                                   & 15                                                           & 0.05                                                 & 1                                                                   & 0.01                                                   \\
Hopper 3x1      & True                                                      & $5 \mathrm{e}-4$                                   & 10                                                           & 0.05                                                 & 1                                                                   & 0                                                      \\
Walker 2x3      & True                                                      & $1 \mathrm{e}-3$                                   & 5                                                            & 0.05                                                 & 2                                                                   & 0                                                      \\
Walker 6x1      & False                                                     & $5 \mathrm{e}-4$                                   & 5                                                            & 0.1                                                  & 1                                                                   & 0.01                                                   \\
Humanoid 17x1   & True                                                      & $5 \mathrm{e}-4$                                   & 5                                                            & 0.1                                                  & 1                                                                   & 0                                                      \\ \hline
\end{tabular}
\end{table}

\begin{table}[htbp]
    \caption{Different hyperparameters used for HATRPO in the MAMuJoCo domain.}
    \label{tab:hatrpo-mamujoco-different}
    \centering
\begin{tabular}{c|cccccc}
\hline
scenarios      & \begin{tabular}[c]{@{}c@{}}linear\\ lr decay\end{tabular} & \begin{tabular}[c]{@{}c@{}}critic\\ lr\end{tabular} & \begin{tabular}[c]{@{}c@{}}critic\\ epoch\end{tabular} & \begin{tabular}[c]{@{}c@{}}clip\\ param\end{tabular} & \begin{tabular}[c]{@{}c@{}}critic\\ mini batch\end{tabular} & \begin{tabular}[c]{@{}c@{}}kl\\ threshold\end{tabular} \\ \hline
Ant 4x2         & False                                                     & $5 \mathrm{e}-4$                                    & 5                                                      & 0.2                                                  & 1                                                           & $5 \mathrm{e}-3$                                       \\
HalfCheetah 2x3 & False                                                     & $5 \mathrm{e}-4$                                    & 5                                                      & 0.2                                                  & 1                                                           & $1 \mathrm{e}-2$                                       \\
Hopper 3x1      & False                                                     & $5 \mathrm{e}-4$                                    & 5                                                      & 0.2                                                  & 1                                                           & $1 \mathrm{e}-3$                                       \\
Walker 2x3      & False                                                     & $5 \mathrm{e}-4$                                    & 5                                                      & 0.2                                                  & 1                                                           & $1 \mathrm{e}-2$                                       \\
Walker 6x1      & False                                                     & $5 \mathrm{e}-4$                                    & 5                                                      & 0.2                                                  & 1                                                           & $5 \mathrm{e}-3$                                       \\ \hline
\end{tabular}
\end{table}

\begin{table}[htbp]
    \caption{Different hyperparameters used for HAA2C in the MAMuJoCo domain.}
    \label{tab:haa2c-mamujoco-different}
    \centering
\begin{tabular}{c|ccccccc}
\hline
scenarios      & \begin{tabular}[c]{@{}c@{}}linear\\ lr decay\end{tabular} & \begin{tabular}[c]{@{}c@{}}actor/critic\\ lr\end{tabular} & \begin{tabular}[c]{@{}c@{}}a2c/critic\\ epoch\end{tabular} & \begin{tabular}[c]{@{}c@{}}clip\\ param\end{tabular} & \begin{tabular}[c]{@{}c@{}}actor/critic\\ mini batch\end{tabular} & \begin{tabular}[c]{@{}c@{}}entropy\\ coef\end{tabular} \\ \hline
Ant 4x2         & True                                                      & $5 \mathrm{e}-4$                                   & 5                                                            & 0.1                                                  & 1                                                                   & 0                                                      \\
HalfCheetah 2x3 & True                                                      & $5 \mathrm{e}-4$                                   & 5                                                            & 0.1                                                  & 1                                                                   & 0                                                      \\
Hopper 3x1      & True                                                      & $1 \mathrm{e}-4$                                   & 3                                                            & 0.1                                                  & 1                                                                   & 0                                                      \\
Walker 2x3      & True                                                      & $1 \mathrm{e}-4$                                   & 5                                                            & 0.1                                                  & 1                                                                   & 0                                                      \\
Walker 6x1      & True                                                      & $1 \mathrm{e}-4$                                   & 5                                                            & 0.1                                                  & 1                                                                   & 0                                                      \\ \hline
\end{tabular}
\end{table}

\begin{table}[htbp]
    \caption{Common hyperparameters used for HADDPG and MADDPG in the MAMuJoCo domain.}
    \label{tab:haddpg-maddpg-mamujoco-common}
    \centering
\begin{tabular}{cc|cc|cc}
\hline
hyperparameters & value      & hyperparameters & value            & hyperparameters & value            \\ \hline
rollout threads & 10         & train interval  & 50               & linear lr decay & False            \\
hidden sizes    & [256, 256] & actor lr        & $5 \mathrm{e}-4$ & critic lr       & $1 \mathrm{e}-3$ \\
gamma           & 0.99       &                 &                  &                 &                  \\ \hline
\end{tabular}
\end{table}

\begin{table}[htbp]
    \centering
    \caption{Different hyperparameters used for HADDPG and MADDPG in the MAMuJoCo domain.}
    \label{tab:haddpg-maddpg-mamujoco-different}
\begin{tabular}{c|ccc}
\hline
scenarios       & \begin{tabular}[c]{@{}c@{}}update\\ per train\end{tabular} & \begin{tabular}[c]{@{}c@{}}exploration\\ noise\end{tabular} & n step \\ \hline
Ant 4x2         & 0.5                                                        & 0.05                                                        & 20     \\
HalfCheetah 2x3 & 1                                                          & 0.1                                                         & 20     \\
Hopper 3x1      & 1                                                          & 0.1                                                         & 20     \\
Walker 2x3      & 1                                                          & 0.1                                                         & 10     \\
Walker 6x1      & 1                                                          & 0.1                                                         & 20     \\ \hline
\end{tabular}
\end{table}

\begin{table}[htbp]
    \caption{Common hyperparameters used for HATD3 and MATD3 in the MAMuJoCo domain.}
    \label{tab:hatd3-mamujoco-common}
    \centering
\begin{tabular}{cc|cc|cc}
\hline
hyperparameters & value            & hyperparameters & value          & hyperparameters  & value            \\ \hline
rollout threads & 10               & train interval  & 50             & update per train & 1                \\
linear lr decay & False            & hidden sizes    & {[}256, 256{]} & actor lr         & $5 \mathrm{e}-4$ \\
critic lr       & $1 \mathrm{e}-3$ & gamma           & 0.99           & exploration noise       & 0.1                 \\ \hline
\end{tabular}
\end{table}

\begin{table}[htbp]
    \centering
    \caption{Different hyperparameters used for HATD3 and MATD3 in the MAMuJoCo domain.}
    \label{tab:hatd3-mamujoco-different}
\begin{tabular}{c|cc}
\hline
scenarios       & policy update frequency & n step \\ \hline
Ant 4x2         & 2                                                                 & 5      \\
HalfCheetah 2x3 & 2                                                                 & 10     \\
Hopper 3x1      & 2                                                                 & 5      \\
Walker 2x3      & 8                                                                 & 20     \\
Walker 6x1      & 2                                                                 & 25     \\
Humanoid 17x1   & 2                                                                 & 5      \\ \hline
\end{tabular}
\end{table}

\subsection{StarCraft Multi-Agent Challenge (SMAC)}
\label{appendix:smac-hyperparam}

In the SMAC domain, for MAPPO and QMIX baselines we adopt the implementation and tuned hyperparameters reported in the MAPPO paper. Here we report the hyperparameters for HAPPO and HATRPO in Table \ref{tab:happo-smac-common}, \ref{tab:happo-smac-different}, \ref{tab:hatrpo-smac-common}, and \ref{tab:hatrpo-smac-different}, which are kept comparable with the baselines for fairness purposes. The \texttt{state type} hyperparameter can take ``EP'' (for \emph{Environment-Provided global state}) and ``FP'' (for \emph{Featured-Pruned agent-specific global state}), as named by \cite{mappo}.

\begin{table}[htbp]
    \caption{Common hyperparameters used for HAPPO in the SMAC domain.}
    \label{tab:happo-smac-common}
    \centering
\begin{tabular}{cc|cc|cc}
\hline
hyperparameters & value            & hyperparameters & value            & hyperparameters & value        \\ \hline
batch size      & 3200             & linear lr decay & False            & hidden sizes    & [64, 64, 64] \\
actor lr        & $5 \mathrm{e}-4$ & critic lr       & $5 \mathrm{e}-4$ & entropy coef    & 0.01         \\ \hline
\end{tabular}
\end{table}

\begin{table}[htbp]
    \caption{Different hyperparameters used for HAPPO in the SMAC domain.}
    \label{tab:happo-smac-different}
    \centering
\begin{tabular}{c|cccccc}
\hline
Map          & network & \begin{tabular}[c]{@{}c@{}}ppo/critic\\ epoch\end{tabular} & \begin{tabular}[c]{@{}c@{}}clip\\ param\end{tabular} & \begin{tabular}[c]{@{}c@{}}actor/critic\\ mini batch\end{tabular} & gamma & \begin{tabular}[c]{@{}c@{}}state\\ type\end{tabular} \\ \hline
8m\_vs\_9m     & RNN     & 5                                                          & 0.05                                                 & 1                                                                 & 0.95  & EP                                                   \\
25m          & RNN     & 5                                                          & 0.2                                                  & 1                                                                 & 0.99  & EP                                                   \\
5m\_vs\_6m     & RNN     & 5                                                          & 0.05                                                 & 1                                                                 & 0.95  & FP                                                   \\
3s5z         & RNN     & 5                                                          & 0.2                                                  & 1                                                                 & 0.99  & EP                                                   \\
10m\_vs\_11m   & RNN     & 5                                                          & 0.05                                                 & 1                                                                 & 0.95  & FP                                                   \\
MMM2         & MLP     & 5                                                          & 0.2                                                  & 1                                                                 & 0.95  & EP                                                   \\
3s5z\_vs\_3s6z & RNN     & 5                                                          & 0.1                                                  & 2                                                                 & 0.95  & FP                                                   \\
27m\_vs\_30m   & RNN     & 5                                                          & 0.05                                                 & 1                                                                 & 0.95  & FP                                                   \\
6h\_vs\_8z     & MLP     & 10                                                         & 0.05                                                 & 2                                                                 & 0.95  & FP                                                   \\
corridor     & MLP     & 5                                                          & 0.2                                                  & 1                                                                 & 0.99  & FP                                                   \\ \hline
\end{tabular}
\end{table}

\begin{table}[htbp]
    \caption{Common hyperparameters used for HATRPO in the SMAC domain.}
    \label{tab:hatrpo-smac-common}
    \centering
\begin{tabular}{cc|cc|cc}
\hline
hyperparameters & value & hyperparameters & value & hyperparameters   & value        \\ \hline
batch size      & 3200  & linear lr decay & False & hidden sizes      & [64, 64, 64] \\
critic epoch    & 5     & clip param      & 0.2   & critic mini batch & 1            \\ \hline
\end{tabular}
\end{table}

\begin{table}[htbp]
    \caption{Different hyperparameters used for HATRPO in the SMAC domain.}
    \label{tab:hatrpo-smac-different}
    \centering
\begin{tabular}{c|cccccc}
\hline
Map          & network & \begin{tabular}[c]{@{}c@{}}critic\\ lr\end{tabular} & gamma & \begin{tabular}[c]{@{}c@{}}kl\\ threshold\end{tabular} & \begin{tabular}[c]{@{}c@{}}backtrack\\ coeff\end{tabular} & \begin{tabular}[c]{@{}c@{}}state\\ type\end{tabular} \\ \hline
8m\_vs\_9m     & MLP     & $5 \mathrm{e}-4$                                    & 0.99  & $5 \mathrm{e}-3$                                       & 0.5                                                      & FP                                                   \\
25m          & RNN     & $5 \mathrm{e}-4$                                    & 0.99  & $1 \mathrm{e}-2$                                       & 0.5                                                      & EP                                                   \\
5m\_vs\_6m     & RNN     & $5 \mathrm{e}-4$                                    & 0.99  & $1 \mathrm{e}-2$                                       & 0.5                                                      & FP                                                   \\
3s5z         & MLP     & $5 \mathrm{e}-4$                                    & 0.95  & $1 \mathrm{e}-2$                                       & 0.5                                                      & EP                                                   \\
10m\_vs\_11m   & MLP     & $5 \mathrm{e}-4$                                    & 0.95  & $5 \mathrm{e}-3$                                       & 0.5                                                      & FP                                                   \\
MMM2         & MLP     & $5 \mathrm{e}-4$                                    & 0.95  & $6 \mathrm{e}-2$                                       & 0.5                                                      & EP                                                   \\
3s5z\_vs\_3s6z & MLP     & $5 \mathrm{e}-4$                                    & 0.99  & $5 \mathrm{e}-3$                                       & 0.5                                                      & FP                                                   \\
27m\_vs\_30m   & RNN     & $5 \mathrm{e}-4$                                    & 0.99  & $1 \mathrm{e}-3$                                       & 0.8                                                      & FP                                                   \\
6h\_vs\_8z     & MLP     & $1 \mathrm{e}-3$                                    & 0.99  & $1 \mathrm{e}-3$                                       & 0.8                                                      & FP                                                   \\
corridor     & RNN     & $5 \mathrm{e}-4$                                    & 0.99  & $6 \mathrm{e}-2$                                       & 0.5                                                      & FP                                                   \\ \hline
\end{tabular}
\end{table}

\subsection{SMACv2}
\label{appendix:smacv2-hyperparam}

In the SMACv2 domain, for MAPPO and QMIX baselines we adopt the implementation and tuned hyperparameters reported in \cite{ellis2022smacv2}. Here we report the hyperparameters for HAPPO and HATRPO in Table \ref{tab:happo-smacv2} and \ref{tab:hatrpo-smacv2-common}, which are kept comparable with the baselines for fairness purposes.

\begin{table}[htbp]
    \caption{Hyperparameters used for HAPPO in the SMACv2 domain.}
    \label{tab:happo-smacv2}
    \centering
\begin{tabular}{cc|cccc}
\hline
hyperparameters & value            & hyperparameters           & \multicolumn{1}{c|}{value}            & hyperparameters & value    \\ \hline
batch size      & 3200             & linear lr decay           & \multicolumn{1}{c|}{False}            & hidden sizes    & {[}64{]} \\
network         & RNN              & ppo / critic epoch        & \multicolumn{1}{c|}{5}                & clip param      & 0.05     \\
actor lr        & $5 \mathrm{e}-4$ & critic lr                 & \multicolumn{1}{c|}{$5 \mathrm{e}-4$} & entropy coef    & 0.01     \\ \cline{3-6} 
gamma           & 0.99             & actor / critic mini batch & \multicolumn{3}{c}{2 for \texttt{terran\_5\_vs\_5} and 1 otherwise}         \\ \hline
\end{tabular}
\end{table}

\begin{table}[htbp]
    \caption{Hyperparameters used for HATRPO on all tasks in the SMACv2 domain.}
    \label{tab:hatrpo-smacv2-common}
    \centering
\begin{tabular}{cc|cc|cc}
\hline
hyperparameters & value            & hyperparameters   & value            & hyperparameters & value    \\ \hline
batch size      & 3200             & linear lr decay   & False            & hidden sizes    & {[}64{]} \\
network         & RNN              & critic lr         & $5 \mathrm{e}-4$ & critic epoch    & 5        \\
clip param      & 0.2              & critic mini batch & 1                & gamma           & 0.99     \\
kl threshold    & $5 \mathrm{e}-3$ & backtrack coeff   & 0.5              &                 &          \\ \hline
\end{tabular}
\end{table}

\subsection{Google Research Football Environment (GRF)}
\label{appendix:football-hyperparam}

In the GRF domain, for MAPPO and QMIX baselines we adopt the implementation and tuned hyperparameters reported in the MAPPO paper. Here we report the hyperparameters for HAPPO in Table \ref{tab:common-happo-grf} and \ref{tab:different-happo-grf}, which are kept similar and comparable to the baselines for fairness purposes.

\begin{table}[htbp]
    \caption{Common hyperparameters used for HAPPO in the GRF domain.}
    \label{tab:common-happo-grf}
\centering
\begin{tabular}{cc|cc|cc}
\hline
hyperparameters & value            & hyperparameters  & value    & hyperparameters   & value            \\ \hline
rollout threads & 50               & hidden sizes     & [64, 64] & actor lr          & $5 \mathrm{e}-4$ \\
critic lr       & $5 \mathrm{e}-4$ & ppo epoch        & 15       & critic epoch      & 15               \\
clip param      & 0.2              & actor mini batch & 2        & critic mini batch & 2                \\
entropy coef    & 0.01             & gamma            & 0.99     &                   &                  \\ \hline
\end{tabular}
\end{table}

\begin{table}[htbp]
    \caption{Different hyperparameters used for HAPPO in the GRF domain.}
    \label{tab:different-happo-grf}
\centering
\begin{tabular}{c|ccc}
\hline
scenarios             & network & \begin{tabular}[c]{@{}c@{}}episode\\ length\end{tabular} & \begin{tabular}[c]{@{}c@{}}linear\\ lr decay\end{tabular} \\ \hline
PS                    & RNN     & 200                                                      & True                                                      \\
RPS                   & MLP     & 200                                                      & False                                                     \\
3v.1                  & MLP     & 200                                                      & True                                                      \\
CA\scriptsize{(easy)} & MLP     & 200                                                      & True                                                      \\
CA\scriptsize{(hard)} & MLP     & 1000                                                     & True                                                      \\ \hline
\end{tabular}
\end{table}

\subsection{Bi-DexterousHands}
\label{appendix:dexhands-hyperparam}

In the Bi-DexterousHands domain, we use the PPO and MAPPO baselines implemented in the Bi-DexterousHands benchmark for comparison and for them we adopt the officially reported hyperparameters. Here we report the hyperparameters used for HAPPO in Table \ref{tab:common-happo-dexhands}. As Bi-DexterousHands tasks are GPU-parallelised, we reload the configuration term \texttt{n\_rollout\_threads} with a meaning of number of parallel environments. Thus, \texttt{parallel envs} in Table \ref{tab:common-happo-dexhands} refers to \texttt{n\_rollout\_threads}.

\begin{table}
    \caption{Common hyperparameters used for HAPPO in the Bi-DexterousHands domain.}
    \label{tab:common-happo-dexhands}
\centering
\begin{tabular}{cc|cc|cc}
\hline
hyperparameters  & value           & hyperparameters   & value            & hyperparameters & value            \\ \hline
parallel envs    & 256             & linear lr decay   & False            & network         & MLP              \\
hidden sizes     & [256, 256, 256] & actor lr          & $5 \mathrm{e}-4$ & critic lr       & $5 \mathrm{e}-4$ \\
ppo epoch        & 5               & critic epoch      & 5                & clip param      & 0.2              \\
actor mini batch & 1               & critic mini batch & 1                & entropy coef    & 0.01             \\
gamma            & 0.95            &                   &                  &                 &                  \\ \hline
\end{tabular}
\end{table}

\vskip 0.2in
\bibliography{main}

\begin{thebibliography}{68}
\providecommand{\natexlab}[1]{#1}
\providecommand{\url}[1]{\texttt{#1}}
\expandafter\ifx\csname urlstyle\endcsname\relax
  \providecommand{\doi}[1]{doi: #1}\else
  \providecommand{\doi}{doi: \begingroup \urlstyle{rm}\Url}\fi

\bibitem[Ackermann et~al.(2019)Ackermann, Gabler, Osa, and Sugiyama]{ackermann2019reducing}
Johannes Ackermann, Volker Gabler, Takayuki Osa, and Masashi Sugiyama.
\newblock Reducing overestimation bias in multi-agent domains using double centralized critics.
\newblock \emph{arXiv preprint arXiv:1910.01465}, 2019.

\bibitem[Al{\'o}s-Ferrer and Netzer(2010)]{alos2010logit}
Carlos Al{\'o}s-Ferrer and Nick Netzer.
\newblock The logit-response dynamics.
\newblock \emph{Games and Economic Behavior}, 68\penalty0 (2):\penalty0 413--427, 2010.

\bibitem[Ausubel and Deneckere(1993)]{ausubel1993generalized}
Lawrence~M Ausubel and Raymond~J Deneckere.
\newblock A generalized theorem of the maximum.
\newblock \emph{Economic Theory}, 3\penalty0 (1):\penalty0 99--107, 1993.

\bibitem[Ba{\c{s}}ar and Olsder(1998)]{bacsar1998dynamic}
Tamer Ba{\c{s}}ar and Geert~Jan Olsder.
\newblock \emph{Dynamic noncooperative game theory}.
\newblock SIAM, 1998.

\bibitem[Bernstein et~al.(2002)Bernstein, Givan, Immerman, and Zilberstein]{bernstein2002complexity}
Daniel~S Bernstein, Robert Givan, Neil Immerman, and Shlomo Zilberstein.
\newblock The complexity of decentralized control of markov decision processes.
\newblock \emph{Mathematics of operations research}, 27\penalty0 (4):\penalty0 819--840, 2002.

\bibitem[Bertsekas(2019)]{bertsekas2019multiagent}
Dimitri Bertsekas.
\newblock Multiagent rollout algorithms and reinforcement learning.
\newblock \emph{arXiv preprint arXiv:1910.00120}, 2019.

\bibitem[Calvo and Dusparic(2018)]{calvo2018heterogeneous}
Jeancarlo~Arguello Calvo and Ivana Dusparic.
\newblock Heterogeneous multi-agent deep reinforcement learning for traffic lights control.
\newblock In \emph{AICS}, pages 2--13, 2018.

\bibitem[Cao et~al.(2012)Cao, Yu, Ren, and Chen]{cao2012overview}
Yongcan Cao, Wenwu Yu, Wei Ren, and Guanrong Chen.
\newblock An overview of recent progress in the study of distributed multi-agent coordination.
\newblock \emph{IEEE Transactions on Industrial informatics}, 9\penalty0 (1):\penalty0 427--438, 2012.

\bibitem[Chen et~al.(2022)Chen, Yang, Wu, Wang, Feng, Jiang, Lu, McAleer, Dong, and Zhu]{chen2022towards}
Yuanpei Chen, Yaodong Yang, Tianhao Wu, Shengjie Wang, Xidong Feng, Jiechuan Jiang, Zongqing Lu, Stephen~Marcus McAleer, Hao Dong, and Song-Chun Zhu.
\newblock Towards human-level bimanual dexterous manipulation with reinforcement learning.
\newblock In \emph{Thirty-sixth Conference on Neural Information Processing Systems Datasets and Benchmarks Track}, 2022.
\newblock URL \url{https://openreview.net/forum?id=D29JbExncTP}.

\bibitem[Christianos et~al.(2021)Christianos, Papoudakis, Rahman, and Albrecht]{christianos2021scaling}
Filippos Christianos, Georgios Papoudakis, Muhammad~A Rahman, and Stefano~V Albrecht.
\newblock Scaling multi-agent reinforcement learning with selective parameter sharing.
\newblock In \emph{International Conference on Machine Learning}, pages 1989--1998. PMLR, 2021.

\bibitem[Claus and Boutilier(1998)]{claus1998dynamics}
Caroline Claus and Craig Boutilier.
\newblock The dynamics of reinforcement learning in cooperative multiagent systems.
\newblock \emph{AAAI/IAAI}, 1998\penalty0 (746-752):\penalty0 2, 1998.

\bibitem[de~Witt et~al.(2020)de~Witt, Gupta, Makoviichuk, Makoviychuk, Torr, Sun, and Whiteson]{de2020independent}
Christian~Schroeder de~Witt, Tarun Gupta, Denys Makoviichuk, Viktor Makoviychuk, Philip~HS Torr, Mingfei Sun, and Shimon Whiteson.
\newblock Is independent learning all you need in the starcraft multi-agent challenge?
\newblock \emph{arXiv preprint arXiv:2011.09533}, 2020.

\bibitem[Ellis et~al.(2022)Ellis, Moalla, Samvelyan, Sun, Mahajan, Foerster, and Whiteson]{ellis2022smacv2}
Benjamin Ellis, Skander Moalla, Mikayel Samvelyan, Mingfei Sun, Anuj Mahajan, Jakob~N Foerster, and Shimon Whiteson.
\newblock Smacv2: An improved benchmark for cooperative multi-agent reinforcement learning.
\newblock \emph{arXiv preprint arXiv:2212.07489}, 2022.

\bibitem[Filar and Vrieze(2012)]{filar2012competitive}
Jerzy Filar and Koos Vrieze.
\newblock \emph{Competitive Markov decision processes}.
\newblock Springer Science \& Business Media, 2012.

\bibitem[Foerster et~al.(2018)Foerster, Farquhar, Afouras, Nardelli, and Whiteson]{foerster2018counterfactual}
Jakob Foerster, Gregory Farquhar, Triantafyllos Afouras, Nantas Nardelli, and Shimon Whiteson.
\newblock Counterfactual multi-agent policy gradients.
\newblock In \emph{Proceedings of the AAAI Conference on Artificial Intelligence}, volume~32, 2018.

\bibitem[Fujimoto et~al.(2018)Fujimoto, Hoof, and Meger]{fujimoto2018addressing}
Scott Fujimoto, Herke Hoof, and David Meger.
\newblock Addressing function approximation error in actor-critic methods.
\newblock In \emph{International conference on machine learning}, pages 1587--1596. PMLR, 2018.

\bibitem[Gemp et~al.(2021)Gemp, McWilliams, Vernade, and Graepel]{gemp2021eigengame}
Ian Gemp, Brian McWilliams, Claire Vernade, and Thore Graepel.
\newblock Eigengame: {\{}PCA{\}} as a nash equilibrium.
\newblock In \emph{International Conference on Learning Representations}, 2021.
\newblock URL \url{https://openreview.net/forum?id=NzTU59SYbNq}.

\bibitem[Hu et~al.(2022{\natexlab{a}})Hu, Xie, Liang, and Chang]{hu2022policy}
Siyi Hu, Chuanlong Xie, Xiaodan Liang, and Xiaojun Chang.
\newblock Policy diagnosis via measuring role diversity in cooperative multi-agent rl.
\newblock In \emph{International Conference on Machine Learning}, pages 9041--9071. PMLR, 2022{\natexlab{a}}.

\bibitem[Hu et~al.(2022{\natexlab{b}})Hu, Zhong, Gao, Wang, Dong, Li, Liang, Chang, and Yang]{hu2022marllib}
Siyi Hu, Yifan Zhong, Minquan Gao, Weixun Wang, Hao Dong, Zhihui Li, Xiaodan Liang, Xiaojun Chang, and Yaodong Yang.
\newblock Marllib: Extending rllib for multi-agent reinforcement learning.
\newblock \emph{arXiv preprint arXiv:2210.13708}, 2022{\natexlab{b}}.

\bibitem[H{\"u}ttenrauch et~al.(2017)H{\"u}ttenrauch, {\v{S}}o{\v{s}}i{\'c}, and Neumann]{huttenrauch2017guided}
Maximilian H{\"u}ttenrauch, Adrian {\v{S}}o{\v{s}}i{\'c}, and Gerhard Neumann.
\newblock Guided deep reinforcement learning for swarm systems.
\newblock \emph{arXiv preprint arXiv:1709.06011}, 2017.

\bibitem[H{\"u}ttenrauch et~al.(2019)H{\"u}ttenrauch, Adrian, Neumann, et~al.]{huttenrauch2019deep}
Maximilian H{\"u}ttenrauch, Sosic Adrian, Gerhard Neumann, et~al.
\newblock Deep reinforcement learning for swarm systems.
\newblock \emph{Journal of Machine Learning Research}, 20\penalty0 (54):\penalty0 1--31, 2019.

\bibitem[Kakade and Langford(2002)]{kakade2002approximately}
Sham Kakade and John Langford.
\newblock Approximately optimal approximate reinforcement learning.
\newblock In \emph{In Proc. 19th International Conference on Machine Learning}. Citeseer, 2002.

\bibitem[Kingma and Ba(2015)]{kingma2014adam}
Diederik~P Kingma and Jimmy Ba.
\newblock Adam: A method for stochastic optimization.
\newblock In \emph{International Conference on Learning Representations}, 2015.

\bibitem[Kuba et~al.(2021)Kuba, Wen, Meng, Zhang, Mguni, Wang, Yang, et~al.]{kuba2021settling}
Jakub~Grudzien Kuba, Muning Wen, Linghui Meng, Haifeng Zhang, David Mguni, Jun Wang, Yaodong Yang, et~al.
\newblock Settling the variance of multi-agent policy gradients.
\newblock \emph{Advances in Neural Information Processing Systems}, 34:\penalty0 13458--13470, 2021.

\bibitem[Kuba et~al.(2022{\natexlab{a}})Kuba, Chen, Wen, Wen, Sun, Wang, and Yang]{kuba2022trust}
Jakub~Grudzien Kuba, Ruiqing Chen, Muning Wen, Ying Wen, Fanglei Sun, Jun Wang, and Yaodong Yang.
\newblock Trust region policy optimisation in multi-agent reinforcement learning.
\newblock In \emph{International Conference on Learning Representations}, 2022{\natexlab{a}}.
\newblock URL \url{https://openreview.net/forum?id=EcGGFkNTxdJ}.

\bibitem[Kuba et~al.(2022{\natexlab{b}})Kuba, de~Witt, and Foerster]{kuba2022mirror}
Jakub~Grudzien Kuba, Christian~Schroeder de~Witt, and Jakob Foerster.
\newblock Mirror learning: A unifying framework of policy optimisation.
\newblock \emph{ICML}, 2022{\natexlab{b}}.

\bibitem[Kurach et~al.(2020)Kurach, Raichuk, Sta{\'n}czyk, Zaj{\k{a}}c, Bachem, Espeholt, Riquelme, Vincent, Michalski, Bousquet, et~al.]{kurach2020google}
Karol Kurach, Anton Raichuk, Piotr Sta{\'n}czyk, Micha{\l} Zaj{\k{a}}c, Olivier Bachem, Lasse Espeholt, Carlos Riquelme, Damien Vincent, Marcin Michalski, Olivier Bousquet, et~al.
\newblock Google research football: A novel reinforcement learning environment.
\newblock In \emph{Proceedings of the AAAI Conference on Artificial Intelligence}, volume~34, pages 4501--4510, 2020.

\bibitem[Li and He(2023)]{li2023multiagent}
Hepeng Li and Haibo He.
\newblock Multiagent trust region policy optimization.
\newblock \emph{IEEE Transactions on Neural Networks and Learning Systems}, 2023.

\bibitem[Lillicrap et~al.(2016)Lillicrap, Hunt, Pritzel, Heess, Erez, Tassa, Silver, and Wierstra]{lillicrap2015continuous}
Timothy~P Lillicrap, Jonathan~J Hunt, Alexander Pritzel, Nicolas Heess, Tom Erez, Yuval Tassa, David Silver, and Daan Wierstra.
\newblock Continuous control with deep reinforcement learning.
\newblock In \emph{International Conference on Learning Representations}, 2016.

\bibitem[Littman(1994)]{littman1994markov}
Michael~L Littman.
\newblock Markov games as a framework for multi-agent reinforcement learning.
\newblock In \emph{Machine learning proceedings 1994}, pages 157--163. Elsevier, 1994.

\bibitem[Lowe et~al.(2017)Lowe, Wu, Tamar, Harb, Abbeel, and Mordatch]{maddpg}
Ryan Lowe, Yi~Wu, Aviv Tamar, Jean Harb, Pieter Abbeel, and Igor Mordatch.
\newblock Multi-agent actor-critic for mixed cooperative-competitive environments.
\newblock In \emph{Proceedings of the 31st International Conference on Neural Information Processing Systems}, pages 6382--6393, 2017.

\bibitem[Mguni et~al.(2021)Mguni, Wu, Du, Yang, Wang, Li, Wen, Jennings, and Wang]{spg-david}
David~H Mguni, Yutong Wu, Yali Du, Yaodong Yang, Ziyi Wang, Minne Li, Ying Wen, Joel Jennings, and Jun Wang.
\newblock Learning in nonzero-sum stochastic games with potentials.
\newblock In Marina Meila and Tong Zhang, editors, \emph{Proceedings of the 38th International Conference on Machine Learning}, volume 139 of \emph{Proceedings of Machine Learning Research}, pages 7688--7699. PMLR, 18--24 Jul 2021.

\bibitem[Mnih et~al.(2016)Mnih, Badia, Mirza, Graves, Lillicrap, Harley, Silver, and Kavukcuoglu]{mnih2016asynchronous}
Volodymyr Mnih, Adria~Puigdomenech Badia, Mehdi Mirza, Alex Graves, Timothy Lillicrap, Tim Harley, David Silver, and Koray Kavukcuoglu.
\newblock Asynchronous methods for deep reinforcement learning.
\newblock In \emph{International conference on machine learning}, pages 1928--1937. PMLR, 2016.

\bibitem[Mordatch and Abbeel(2018)]{mordatch2018emergence}
Igor Mordatch and Pieter Abbeel.
\newblock Emergence of grounded compositional language in multi-agent populations.
\newblock In \emph{Proceedings of the AAAI conference on artificial intelligence}, volume~32, 2018.

\bibitem[Nash(1951)]{nash1951non}
John Nash.
\newblock Non-cooperative games.
\newblock \emph{Annals of mathematics}, pages 286--295, 1951.

\bibitem[Oliehoek and Amato(2016)]{oliehoek2016concise}
Frans~A Oliehoek and Christopher Amato.
\newblock \emph{A concise introduction to decentralized POMDPs}.
\newblock Springer, 2016.

\bibitem[Papoudakis et~al.(2021)Papoudakis, Christianos, Schäfer, and Albrecht]{papoudakis2021benchmarking}
Georgios Papoudakis, Filippos Christianos, Lukas Schäfer, and Stefano~V. Albrecht.
\newblock Benchmarking multi-agent deep reinforcement learning algorithms in cooperative tasks.
\newblock In \emph{Proceedings of the Neural Information Processing Systems Track on Datasets and Benchmarks (NeurIPS)}, 2021.
\newblock URL \url{http://arxiv.org/abs/2006.07869}.

\bibitem[Peng et~al.(2021)Peng, Rashid, Schroeder~de Witt, Kamienny, Torr, B{\"o}hmer, and Whiteson]{peng2021facmac}
Bei Peng, Tabish Rashid, Christian Schroeder~de Witt, Pierre-Alexandre Kamienny, Philip Torr, Wendelin B{\"o}hmer, and Shimon Whiteson.
\newblock Facmac: Factored multi-agent centralised policy gradients.
\newblock \emph{Advances in Neural Information Processing Systems}, 34:\penalty0 12208--12221, 2021.

\bibitem[Peng et~al.(2017)Peng, Yuan, Wen, Yang, Tang, Long, and Wang]{peng1703multiagent}
P~Peng, Q~Yuan, Y~Wen, Y~Yang, Z~Tang, H~Long, and J~Wang.
\newblock Multiagent bidirectionally-coordinated nets for learning to play starcraft combat games. arxiv 2017.
\newblock \emph{arXiv preprint arXiv:1703.10069}, 2017.

\bibitem[Rashid et~al.(2018)Rashid, Samvelyan, Schroeder, Farquhar, Foerster, and Whiteson]{rashid2018qmix}
Tabish Rashid, Mikayel Samvelyan, Christian Schroeder, Gregory Farquhar, Jakob Foerster, and Shimon Whiteson.
\newblock Qmix: Monotonic value function factorisation for deep multi-agent reinforcement learning.
\newblock In \emph{International Conference on Machine Learning}, pages 4295--4304. PMLR, 2018.

\bibitem[Ray-Team(accessed on 2023-03-14)]{ray}
Ray-Team.
\newblock Ray rllib documentation: Multi-agent deep deterministic policy gradient (maddpg).
\newblock \url{https://docs.ray.io/en/latest/rllib/rllib-algorithms.html#multi-agent-deep-deterministic-policy-gradient-maddpg}, accessed on 2023-03-14.

\bibitem[Samvelyan et~al.(2019)Samvelyan, Rashid, de~Witt, Farquhar, Nardelli, Rudner, Hung, Torr, Foerster, and Whiteson]{samvelyanstarcraft}
Mikayel Samvelyan, Tabish Rashid, Christian~Schroeder de~Witt, Gregory Farquhar, Nantas Nardelli, Tim G.~J. Rudner, Chia-Man Hung, Philiph H.~S. Torr, Jakob Foerster, and Shimon Whiteson.
\newblock {The} {StarCraft} {Multi}-{Agent} {Challenge}.
\newblock \emph{CoRR}, abs/1902.04043, 2019.

\bibitem[Schulman et~al.(2015)Schulman, Levine, Abbeel, Jordan, and Moritz]{trpo}
John Schulman, Sergey Levine, Pieter Abbeel, Michael Jordan, and Philipp Moritz.
\newblock Trust region policy optimization.
\newblock In \emph{International conference on machine learning}, pages 1889--1897. PMLR, 2015.

\bibitem[Schulman et~al.(2016)Schulman, Moritz, Levine, Jordan, and Abbeel]{schulman2015high}
John Schulman, Philipp Moritz, Sergey Levine, Michael Jordan, and Pieter Abbeel.
\newblock High-dimensional continuous control using generalized advantage estimation.
\newblock In \emph{International Conference on Learning Representations}, 2016.

\bibitem[Schulman et~al.(2017)Schulman, Wolski, Dhariwal, Radford, and Klimov]{ppo}
John Schulman, F.~Wolski, Prafulla Dhariwal, Alec Radford, and Oleg Klimov.
\newblock Proximal policy optimization algorithms.
\newblock \emph{ArXiv}, abs/1707.06347, 2017.

\bibitem[Shapley(1953)]{shapley1953stochastic}
Lloyd~S Shapley.
\newblock Stochastic games.
\newblock \emph{Proceedings of the national academy of sciences}, 39\penalty0 (10):\penalty0 1095--1100, 1953.

\bibitem[Silver et~al.(2014)Silver, Lever, Heess, Degris, Wierstra, and Riedmiller]{silver2014deterministic}
David Silver, Guy Lever, Nicolas Heess, Thomas Degris, Daan Wierstra, and Martin Riedmiller.
\newblock Deterministic policy gradient algorithms.
\newblock In \emph{International conference on machine learning}, pages 387--395. PMLR, 2014.

\bibitem[Sunehag et~al.(2018)Sunehag, Lever, Gruslys, Czarnecki, Zambaldi, Jaderberg, Lanctot, Sonnerat, Leibo, Tuyls, et~al.]{sunehag2018value}
Peter Sunehag, Guy Lever, Audrunas Gruslys, Wojciech~Marian Czarnecki, Vinicius Zambaldi, Max Jaderberg, Marc Lanctot, Nicolas Sonnerat, Joel~Z Leibo, Karl Tuyls, et~al.
\newblock Value-decomposition networks for cooperative multi-agent learning based on team reward.
\newblock In \emph{Proceedings of the 17th International Conference on Autonomous Agents and MultiAgent Systems}, pages 2085--2087, 2018.

\bibitem[Sutton et~al.(2000)Sutton, Mcallester, Singh, and Mansour]{sutton:nips12}
R.~S. Sutton, D.~Mcallester, S.~Singh, and Y.~Mansour.
\newblock Policy gradient methods for reinforcement learning with function approximation.
\newblock In \emph{Advances in Neural Information Processing Systems 12}, volume~12, pages 1057--1063. MIT Press, 2000.

\bibitem[Sutton and Barto(2018)]{sutton2018reinforcement}
Richard~S Sutton and Andrew~G Barto.
\newblock \emph{Reinforcement learning: An introduction}.
\newblock MIT press, 2018.

\bibitem[Tan(1993)]{tan1993multi}
Ming Tan.
\newblock Multi-agent reinforcement learning: Independent vs. cooperative agents.
\newblock In \emph{Proceedings of the tenth international conference on machine learning}, pages 330--337, 1993.

\bibitem[Terry et~al.(2021)Terry, Black, Grammel, Jayakumar, Hari, Sullivan, Santos, Dieffendahl, Horsch, Perez-Vicente, et~al.]{terry2021pettingzoo}
J~Terry, Benjamin Black, Nathaniel Grammel, Mario Jayakumar, Ananth Hari, Ryan Sullivan, Luis~S Santos, Clemens Dieffendahl, Caroline Horsch, Rodrigo Perez-Vicente, et~al.
\newblock Pettingzoo: Gym for multi-agent reinforcement learning.
\newblock \emph{Advances in Neural Information Processing Systems}, 34:\penalty0 15032--15043, 2021.

\bibitem[Terry et~al.(2020)Terry, Grammel, Son, and Black]{terry2020parameter}
Justin~K Terry, Nathaniel Grammel, Sanghyun Son, and Benjamin Black.
\newblock Parameter sharing for heterogeneous agents in multi-agent reinforcement learning.
\newblock \emph{arXiv preprint arXiv:2005.13625}, 2020.

\bibitem[Van~Hasselt et~al.(2016)Van~Hasselt, Guez, and Silver]{van2016deep}
Hado Van~Hasselt, Arthur Guez, and David Silver.
\newblock Deep reinforcement learning with double q-learning.
\newblock In \emph{Proceedings of the AAAI conference on artificial intelligence}, volume~30, 2016.

\bibitem[Wang et~al.(2023)Wang, Ye, and Lu]{wang2023more}
Jiangxing Wang, Deheng Ye, and Zongqing Lu.
\newblock More centralized training, still decentralized execution: Multi-agent conditional policy factorization.
\newblock In \emph{The Eleventh International Conference on Learning Representations}, 2023.
\newblock URL \url{https://openreview.net/forum?id=znLlSgN-4S0}.

\bibitem[Wang et~al.(2021)Wang, Gupta, Mahajan, Peng, Whiteson, and Zhang]{wang2020rode}
Tonghan Wang, Tarun Gupta, Anuj Mahajan, Bei Peng, Shimon Whiteson, and Chongjie Zhang.
\newblock Rode: Learning roles to decompose multi-agent tasks.
\newblock \emph{International Conference on Learning Representations}, 2021.

\bibitem[Wang et~al.(2016)Wang, Schaul, Hessel, Hasselt, Lanctot, and Freitas]{wang2016dueling}
Ziyu Wang, Tom Schaul, Matteo Hessel, Hado Hasselt, Marc Lanctot, and Nando Freitas.
\newblock Dueling network architectures for deep reinforcement learning.
\newblock In \emph{International conference on machine learning}, pages 1995--2003. PMLR, 2016.

\bibitem[Wen et~al.(2018)Wen, Yang, Luo, Wang, and Pan]{wen2018probabilistic}
Ying Wen, Yaodong Yang, Rui Luo, Jun Wang, and Wei Pan.
\newblock Probabilistic recursive reasoning for multi-agent reinforcement learning.
\newblock In \emph{International Conference on Learning Representations}, 2018.

\bibitem[Wen et~al.(2020)Wen, Yang, and Wang]{gr2}
Ying Wen, Yaodong Yang, and Jun Wang.
\newblock Modelling bounded rationality in multi-agent interactions by generalized recursive reasoning.
\newblock In Christian Bessiere, editor, \emph{Proceedings of the Twenty-Ninth International Joint Conference on Artificial Intelligence, {IJCAI-20}}, pages 414--421. International Joint Conferences on Artificial Intelligence Organization, 7 2020.
\newblock Main track.

\bibitem[Wen et~al.(2022)Wen, Chen, Yang, Li, Tian, Chen, and Wang]{wen2022game}
Ying Wen, Hui Chen, Yaodong Yang, Minne Li, Zheng Tian, Xu~Chen, and Jun Wang.
\newblock A game-theoretic approach to multi-agent trust region optimization.
\newblock In \emph{International Conference on Distributed Artificial Intelligence}, pages 74--87. Springer, 2022.

\bibitem[Wu et~al.(2021)Wu, Yu, Ye, Zhang, Zhuo, et~al.]{wu2021coordinated}
Zifan Wu, Chao Yu, Deheng Ye, Junge Zhang, Hankz~Hankui Zhuo, et~al.
\newblock Coordinated proximal policy optimization.
\newblock \emph{Advances in Neural Information Processing Systems}, 34:\penalty0 26437--26448, 2021.

\bibitem[Yang and Wang(2020)]{yang2020overview}
Yaodong Yang and Jun Wang.
\newblock An overview of multi-agent reinforcement learning from game theoretical perspective.
\newblock \emph{arXiv preprint arXiv:2011.00583}, 2020.

\bibitem[Yang et~al.(2018)Yang, Luo, Li, Zhou, Zhang, and Wang]{yang2018mean}
Yaodong Yang, Rui Luo, Minne Li, Ming Zhou, Weinan Zhang, and Jun Wang.
\newblock Mean field multi-agent reinforcement learning.
\newblock In \emph{International Conference on Machine Learning}, pages 5571--5580. PMLR, 2018.

\bibitem[Yang et~al.(2020)Yang, Wen, Wang, Chen, Shao, Mguni, and Zhang]{yang2020multi}
Yaodong Yang, Ying Wen, Jun Wang, Liheng Chen, Kun Shao, David Mguni, and Weinan Zhang.
\newblock Multi-agent determinantal q-learning.
\newblock In \emph{International Conference on Machine Learning}, pages 10757--10766. PMLR, 2020.

\bibitem[Yu et~al.(2022)Yu, Velu, Vinitsky, Gao, Wang, Bayen, and Wu]{mappo}
Chao Yu, Akash Velu, Eugene Vinitsky, Jiaxuan Gao, Yu~Wang, Alexandre Bayen, and Yi~Wu.
\newblock The surprising effectiveness of {PPO} in cooperative multi-agent games.
\newblock In \emph{Thirty-sixth Conference on Neural Information Processing Systems Datasets and Benchmarks Track}, 2022.

\bibitem[Zhang et~al.(2020)Zhang, Chen, Huang, Li, Yang, Zhang, and Wang]{zhang2020bi}
Haifeng Zhang, Weizhe Chen, Zeren Huang, Minne Li, Yaodong Yang, Weinan Zhang, and Jun Wang.
\newblock Bi-level actor-critic for multi-agent coordination.
\newblock In \emph{Proceedings of the AAAI Conference on Artificial Intelligence}, volume~34, pages 7325--7332, 2020.

\bibitem[Zhang et~al.(2021)Zhang, Yang, and Ba{\c{s}}ar]{zhang2021multi}
Kaiqing Zhang, Zhuoran Yang, and Tamer Ba{\c{s}}ar.
\newblock Multi-agent reinforcement learning: A selective overview of theories and algorithms.
\newblock \emph{Handbook of reinforcement learning and control}, pages 321--384, 2021.

\bibitem[Zhou et~al.(2023)Zhou, Wan, Wang, Wen, Wu, Wen, Yang, Yu, Wang, and Zhang]{JMLR:v24:22-0169}
Ming Zhou, Ziyu Wan, Hanjing Wang, Muning Wen, Runzhe Wu, Ying Wen, Yaodong Yang, Yong Yu, Jun Wang, and Weinan Zhang.
\newblock Malib: A parallel framework for population-based multi-agent reinforcement learning.
\newblock \emph{Journal of Machine Learning Research}, 24\penalty0 (150):\penalty0 1--12, 2023.
\newblock URL \url{http://jmlr.org/papers/v24/22-0169.html}.

\end{thebibliography}

\end{document}